\documentclass[]{fairmeta}

\title{Towards joint scaling laws with optimal batch size schedules}

\author[1]{Jiaxiang Li}
\author[1]{Zhiqi Bu}
\author[2]{Shiyun Xu}

\affiliation[1]{Meta}
\affiliation[2]{Independent researcher}

\abstract{
Modern deep learning typically keeps the batch size static throughout training, thus overlooking the joint effect of learning rate and batch size on the training dynamics. In this paper, we study the deep learning dynamics through the lens of convex optimization and derive a joint characterization of loss in terms of both schedules, applicable to general optimizers and model architectures. This characterization yields a closed-form optimal batch size schedule for any prescribed learning rate schedule, and further leads to joint scaling laws that consistently outperform static batch size baselines, highlighting the significance of dynamic batch size schedule in large language model training.
}

\usepackage{verbatim}
\usepackage{url}
\usepackage{algorithm}
\usepackage{algorithmicx}
\usepackage[noend]{algpseudocode}
\usepackage{amsmath, amsfonts, amssymb, amsthm, amsbsy, amscd, bm, bbm,mathrsfs}
\usepackage{graphicx}
\usepackage{setspace}
\usepackage{hyperref}
\usepackage{cancel}
\usepackage{cleveref}
\usepackage[hyperpageref]{backref}
\usepackage{appendix}
\usepackage{xspace}
\usepackage{epigraph}
\usepackage[thinc]{esdiff}
\usepackage{multirow}
\usepackage{booktabs,tabularx,array}
\usepackage[labelfont=bf]{caption}
\usepackage[most]{tcolorbox}

\makeatletter

\newcommand{\E}{\mathbb{E}}

\allowdisplaybreaks \numberwithin{equation}{section}
\makeatother

\newcommand{\w}{\mathbf{w}}
\newcommand{\g}{\mathbf{g}}
\newcommand{\R}{\mathbb{R}}

\usepackage{aliascnt}

\newtheorem{theorem}{Theorem}

\newaliascnt{lemma}{theorem}
\newtheorem{lemma}[lemma]{Lemma}
\aliascntresetthe{lemma}

\newaliascnt{definition}{theorem}

\aliascntresetthe{definition}

\newaliascnt{corollary}{othertheorem}
\newtheorem{corollary}[corollary]{Corollary}
\aliascntresetthe{corollary}

\newaliascnt{proposition}{othertheorem}

\aliascntresetthe{proposition}

\newaliascnt{remark}{othertheorem}
\newtheorem{remark}[remark]{Remark}
\aliascntresetthe{remark}

\begin{document}

\maketitle

\section{Introduction}
Large-scale model training has been central to recent progress in artificial intelligence  \citep{kaplan2020scaling,hoffmann2022training}, but its effectiveness depends critically on robust choices of hyperparameters, particularly the learning rate and batch size (BS). 
The learning rate controls the magnitude of parameter updates: overly small values slow optimization, whereas overly large values can destabilize training and cause loss divergence \citep{yang2021tuning,li2025predictable}. The batch size controls a trade-off between model quality and systems efficiency: larger batches can improve throughput by exploiting data parallelism, amortizing communication, and increasing hardware utilization \citep{shallue2019measuring,goyal2017accurate,rajbhandari2020zero}; however, beyond a critical regime, increasing the batch size can degrade model performance under fixed data or compute budgets, because it reduces the number of optimization steps and changes the gradient-noise profile \citep{mccandlish2018empirical,golmant2018computational,keskar2017large}.

Consequently, hyperparameter selection is essential for efficient model training, yet exhaustive tuning is often prohibitively computationally expensive at frontier scale. This motivates principled methods for transferring or scaling hyperparameters across model sizes and training horizons, so that small-scale experiments can reliably inform large-scale ones \citep{kaplan2020scaling,li2025predictable,hu2024minicpm,zhangdoes,bergsma2025powerlines,shuai2024scalinglawbs,bi2024deepseek}.

For the learning rate, the mainstream practice is to use a schedule $\{\eta_t\}_{t=1}^T$ rather than a fixed constant. Common schedules typically combine warmup, which mitigates early-training instability, with decay, which improves late-stage optimization \citep{vaswani2017attention,tuningplaybookgithub,defazio2023optimal,schaipp2025surprising}. Under maximal update parametrization ($\mu$P), optimal learning rates can transfer across model sizes when the training horizon is fixed \citep{yang2021tuning,yang2024tensor,dey2026don}. Furthermore, when the training horizon $T$ changes, the optimal peak learning rate must be rescaled, often following a power law of the form $\eta^\star(T)\propto T^{-\alpha}$ \citep{bi2024deepseek,bjorck2025scaling,bu2026convex}.

In contrast, batch size has been less explored systematically as a scheduled hyperparameter. Existing approaches often use fixed batch sizes or hand-designed batch-size schedules coupled to the learning-rate schedule \citep{smith2018dontdecay,balles2017coupling,meterez2025seesaw}. See \Cref{appendix:existing_works} for a summarization of batch size schedules for recent LLMs. This is potentially suboptimal: a fixed batch size is only a constrained special case of a broader class of batch-size schedules $\{B_t\}_{t=1}^T$.
This gap motivates a principled study of batch-size scheduling and its transfer across model sizes and training horizons.

In this work, we study the fundamental question:

\begin{center}
    \noindent\textbf{What is the optimal batch size schedule and which factors determine it?}
\end{center}

Our approach builds on the stochastic convex optimization theory, capturing the training dynamics of deep learning \citep{defazio2023optimal} which is non-convex and non-smooth. To be specific, we develop a principled and generalizable mapping in \eqref{eq:loss mapping last} that characterizes \textit{the loss sequence} from both the \textit{learning rate sequence} and the \textit{batch size sequence}. 
This sequence-to-sequence characterization is empirically accurate and analytically tractable to formalize a constrained optimization problem: the objective is the last iteration of the loss sequence, the domain is batch size sequences, and the constraint is data or compute budgets. We solve this problem for the optimal batch size schedule in \Cref{thm:optimal B}. 

In summary, our answer is that the optimal batch size schedule is a function of total computational budget and learning rate schedule shape (cosine, linear, WSD), but independent of peak learning rate, weight decay, model size, model architecture, or optimizer. As a result, the batch size becomes a principled scaling dimension rather than a separately tuned hyperparameter.

\paragraph{\textbf{Contributions:}}
\begin{enumerate}
  \item \textbf{Loss prediction by learning rate and batch size schedules.}
  We derive an any-iteration loss characterization in \eqref{eq:loss mapping equality} from any learning rate and batch size schedules. This characterization is empirically precise for general optimizers and model architectures, providing evidence for the convex dominance in deep learning.

  \item \textbf{Closed-form optimal batch size schedule.}
  We derive the closed-form optimal batch size for any learning rate schedule in \Cref{thm:optimal B}, given the data, compute or general budgets. Notably, the optimal batch size schedule is independent of the peak learning rate, thus decoupled from the learning rate tuning.

  \item \textbf{Joint scaling laws and universal dynamics.}
  We build joint scaling laws that employ $1/\sqrt{T}$-scaling of learning rate and weight decay, as well as our scale-free batch size schedule, with an improvement on compute efficiency by $6\sim 15\%$ and zero computational overhead. Crucially, we show that the optimal batch size schedule preserves universal training dynamics, also known as supercollapse,   rendering a robust and high-performing recipe for scaling.
  
  \item \textbf{Empirical ablations in various regimes.}
  We validate the advantage of our batch size schedules, in pre-training and fine-tuning regimes, on Llama3 (dense), Qwen3 (MoE) and vision-language models up to 7B parameters, with AdamW and Muon optimizers, across different scales and configurations in \Cref{app:ablation}.
\end{enumerate}

\begin{figure}[!htb]
    \centering
\includegraphics[width=\linewidth]{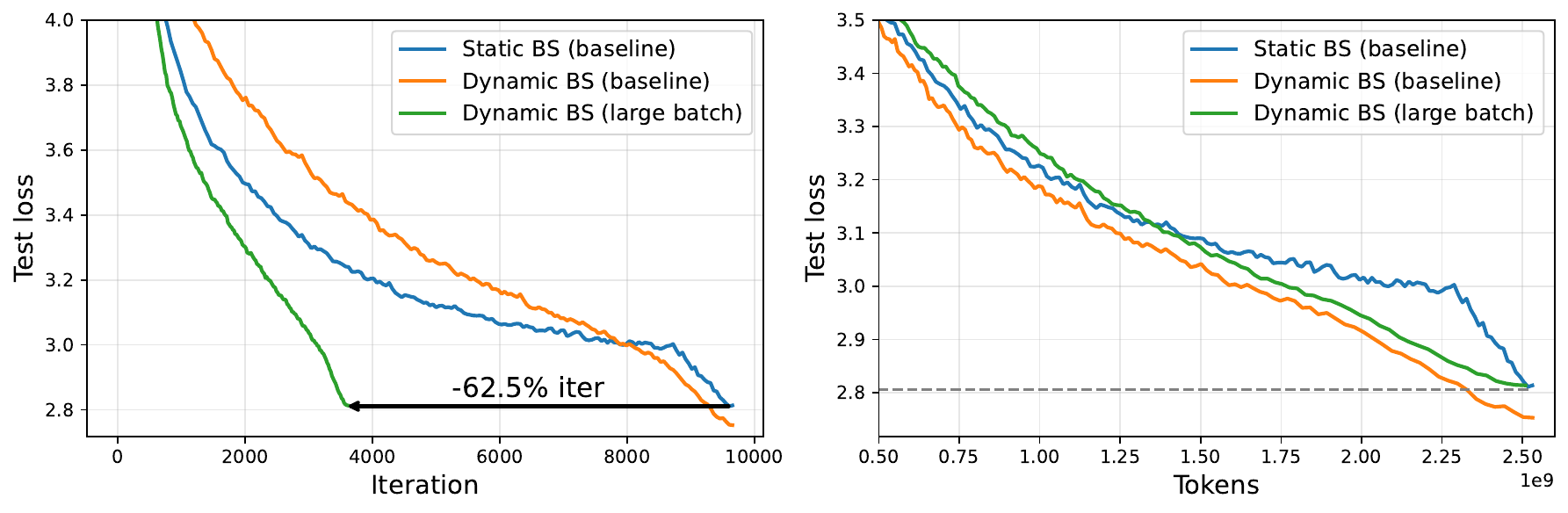}
\\
\includegraphics[width=\linewidth]{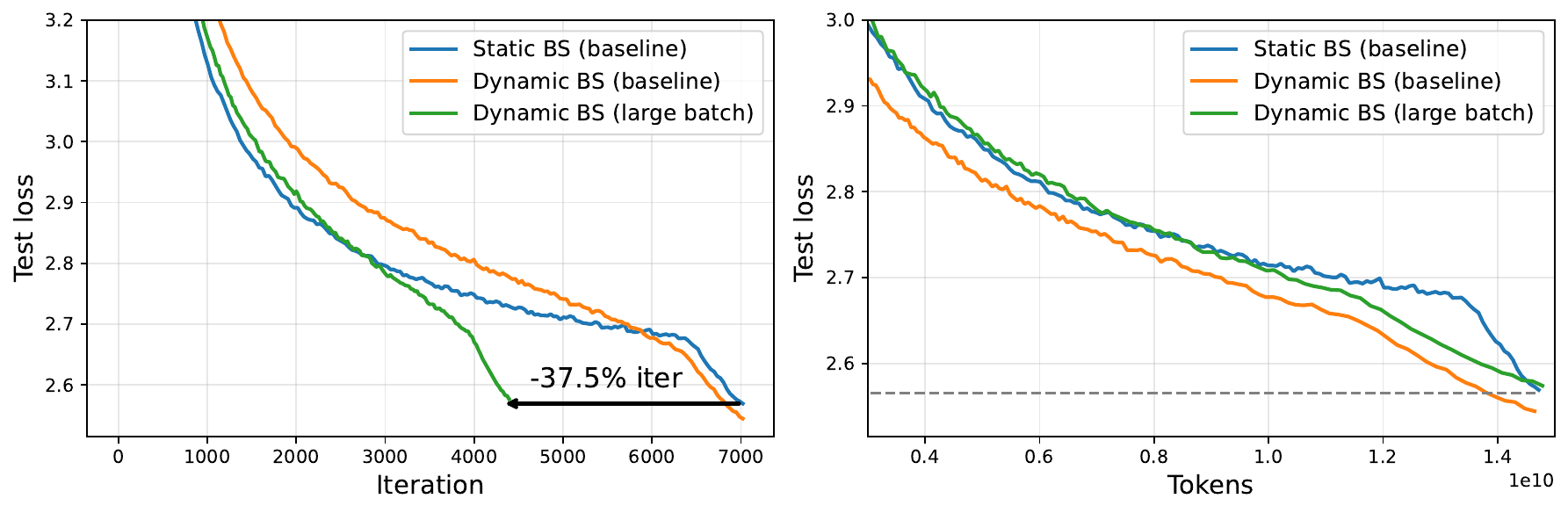}
    \vspace{-0.2cm}
    \caption{Dynamic batch size outperforms static batch size under fixed compute of 1e20 FLOPs on Fineweb-edu dataset: (I) for fixed training horizon, \textcolor{orange}{dynamic batch size schedule} (same BS) achieves better loss than \textcolor{blue}{static baselines}; (II) for the same loss, \textcolor{green}{dynamic batch size schedule} (larger BS) uses significantly fewer iterations than \textcolor{blue}{static baselines}, with the same learning rate. Upper row: Llama3-7B (dense). Lower row: Qwen3-A1B (MoE). See the evaluation in \Cref{tab:longest} and \Cref{fig:sub eval}.}
    \label{fig:longest}
\end{figure}

\begin{table}[!htb]
\centering
\caption{Dynamic batch size generally outperforms static batch size (marked in \textcolor{green}{green}) by lm-evaluation-harness v0.4 on the last-iterate. See \Cref{fig:sub eval} for the evaluation on intermediate iterates.}
\resizebox{\linewidth}{!}{
\begin{tabular}{c|c||c|c|c|c|c|c|c|c}
\hline
 Model & Batch Size & LAMBADA & HellaSwag & ARC-Easy & ARC-Challenge & WinoGrande & PIQA & OpenBookQA & BoolQ \\
 \hline
 \multirow{2}{*}{\textbf{Llama3 (dense)}} & static& 32.23 & 42.93 & 56.22 & 29.26 & 53.03 & 67.89 & 32.40 & 61.89 \\
 & dynamic& \textcolor{green}{35.36} & \textcolor{green}{44.33} & \textcolor{green}{56.36} & \textcolor{green}{30.46} & \textcolor{green}{53.99} & \textcolor{green}{68.34} & \textcolor{green}{33.20} & \textcolor{green}{62.26} \\
 \hline
 \multirow{2}{*}{\textbf{Qwen3 (MoE)}} & static& 36.06 & 45.61 & 57.83 & 32.59 & 52.49 & 68.77 & 34.00 & 61.47 \\
 & dynamic & \textcolor{green}{37.71} & \textcolor{green}{46.80} & \textcolor{green}{59.09} & \textcolor{red}{32.25} & \textcolor{green}{55.33} & \textcolor{green}{69.97} & \textcolor{green}{36.20} & \textcolor{green}{61.59} \\
 \hline
\end{tabular}
}
\label{tab:longest}
\end{table}

\section{Discrete loss prediction by sequences of batch size and learning rate}
We start with the stochastic gradient descent (SGD) as
$\w_{t+1}=\w_t-\eta_{t} \g_t$,
where $\w_t$ is the parameters, $\eta_t$ is the learning rate, $\g_t=\g(\w_t)$ is the mini-batch gradient with batch size $B_t$,
and $1\leq t\leq T$ with $T$ total iterations. In what follows, we denote $\w_*$ as the minimizer of loss $L$, and $D=\|\w_t-\w_*\|$.

\subsection{From learning rate sequence to loss sequence}
The prediction from learning rate to loss value has been investigated under convexity or strong convexity loss, Lipschitz continuity or smoothness, and averaged or current iterate. We follow the work of \cite{defazio2023optimal,bu2026convex, schaipp2025surprising} and study the sequence-to-sequence prediction under convex loss and SGD in \Cref{thm: simple} (see proof in \Cref{app:thm1}). 
\begin{theorem}
\label{thm: simple}
For convex loss and any learning rate sequence $\eta_{t}$, under SGD optimizer, we have
\begin{align}
\E L(\w_\tau; \eta_t)\leq L_*+\frac{D^2}{2\sum_{t=1}^\tau \eta_t}+
\frac{1}{2}\sum_{t=1}^{\tau-1}\frac{\eta_t^2\E\|\g_t\|^2}{\sum_{k=t+1}^\tau\eta_k}+\frac{1}{2}\eta_\tau \E\|\g_\tau\|^2
\label{eq:origin3}
\end{align}
\end{theorem}
We emphasize that although the theoretical basis for \Cref{thm: simple} is limited to convex loss and SGD optimizer, the empirical applicability generalizes to non-convex neural networks (e.g. LLM and VLM) and adaptive optimizers (e.g. AdamW and Muon) as extensively investigated in \cite{bu2026convex}.

\subsection{From batch size sequence to loss sequence}
To understand the effect of batch size on the loss, we simply assume
\begin{align}
\E\|\g_t\|^2 \leq G^2+{X}/{B_t}
\label{eq:grad condition}
\end{align}
where $G$ upper bounds the norm of expected gradient $\|\E\g_t\|$, and $X$ upper bounds the trace of per-example gradient covariance matrix. We derive \eqref{eq:grad condition} and extend it to adaptive optimizers in \Cref{app:grad condition}.

\begin{theorem}
\label{thm:seq2seq}
Assuming \eqref{eq:grad condition}, for convex loss and any learning rate sequence $\eta_t$, we have
\begin{align}
\begin{split}
\E L(\w_\tau; \eta_t, B_t)\leq L_*+\frac{D^2}{2\sum_{t=1}^{\tau} \eta_t}+
\frac{1}{2}\sum_{t=1}^{\tau-1}\frac{\eta_t^2  (G^2+X/B_t)}{\sum_{k=t+1}^{\tau} \eta_k}+\frac{\eta_\tau}{2} (G^2+X/B_\tau)
\end{split}
\label{eq:loss mapping}
\end{align}
\end{theorem}

\begin{proof}[Proof of \Cref{thm:seq2seq}]
Substituting \eqref{eq:grad condition} to \Cref{thm: simple}.
\end{proof}

\subsection{Precise characterization under linear regression}
We validate \Cref{thm:seq2seq} and thus justify the assumptions \eqref{eq:origin3} and \eqref{eq:grad condition} on the various sequences of learning rate and batch size in \Cref{fig:nanogpt-muon}, where we construct three variables
$$x_{1,\tau}=\frac{1}{2\sum_{t=1}^\tau \eta_t},
x_{2,\tau}=\frac{1}{2}\sum_{t=1}^{\tau-1}\frac{\eta_t^2}{\sum_{k=t+1}^\tau \eta_k}+\frac{\eta_\tau}{2}, 
x_{3,\tau}=\frac{1}{2}\sum_{t=1}^{\tau-1}\frac{\eta_t^2 /B_t}{\sum_{k=t+1}^\tau \eta_k}+\frac{\eta_\tau/B_\tau}{2}$$
then solve a linear regression\footnote{If $B_t$ is constant, then the features $(\bm{x}_2,\bm{x}_3)$ are collinear and perfectly correlated, rendering the coefficients $(G,X)$ indistinguishable, although the linear regression is solvable.} with $\bm{L}=[L_1,...,L_T]$:
\begin{align}
\min_{L_*,D,G,X} \left\|L_*+D^2 \bm{x}_{1}+G^2 \bm{x}_{2}+X \bm{x}_{3}-\bm{L}\right\|_2^2
\label{eq:linear fit}
\end{align}

\begin{figure}[!htb]
\centering

\begin{tcolorbox}[
    colback=white,
    colframe=yellow!60!orange,
    boxrule=1.2pt,
    arc=3mm,
    left=0mm,
    right=0mm,
    top=1mm,
    bottom=0mm,
    width=\linewidth,
    title={\textbf{constant batch size, dynamic learning rate}},
    coltitle=black,
    colbacktitle=yellow!25,
    fonttitle=\normalsize,
    enhanced,
    attach boxed title to top center={
        yshift=-1.5mm
    },
    boxed title style={
        colback=yellow!25,
        colframe=yellow!60!orange,
        arc=2mm,
        boxrule=1pt
    }
]

\centering

\includegraphics[width=0.24\linewidth]{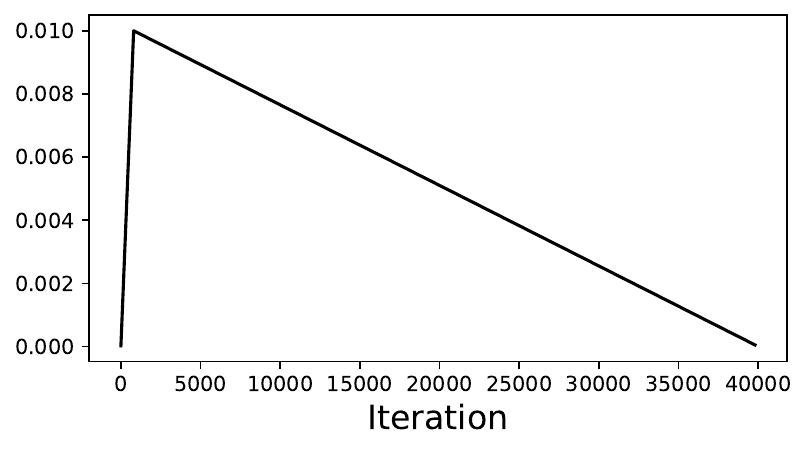}
\includegraphics[width=0.24\linewidth]{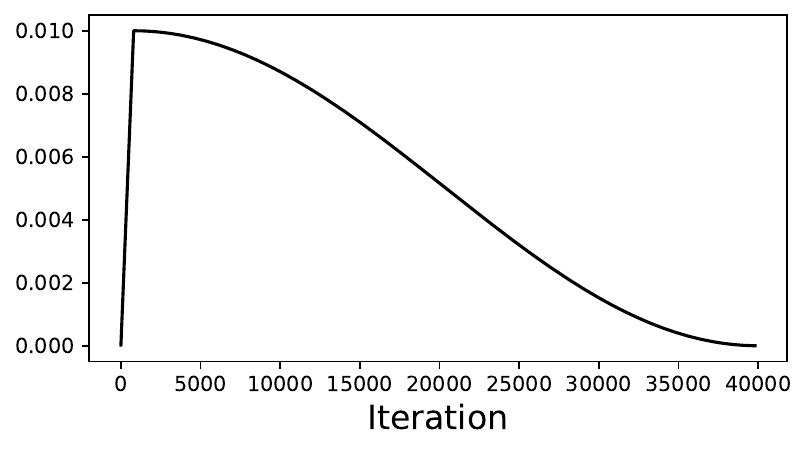}
\includegraphics[width=0.24\linewidth]{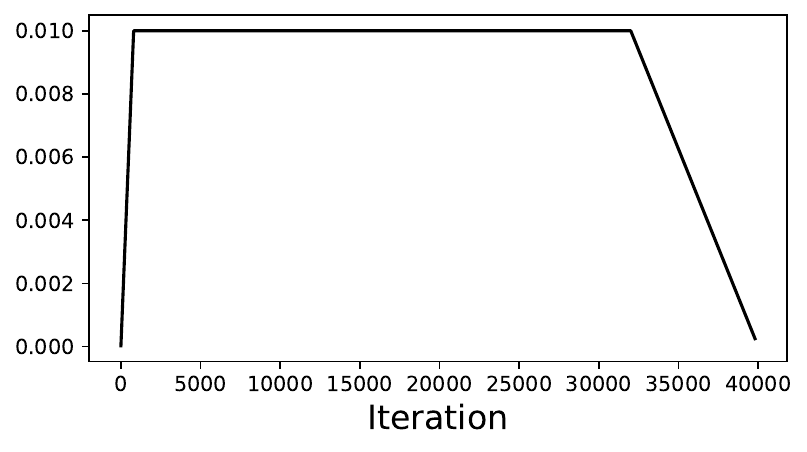}
\includegraphics[width=0.24\linewidth]{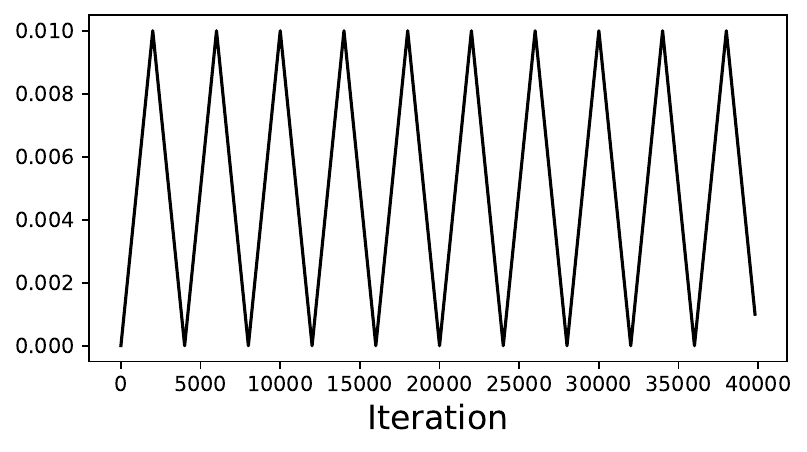}

\includegraphics[width=0.24\linewidth]{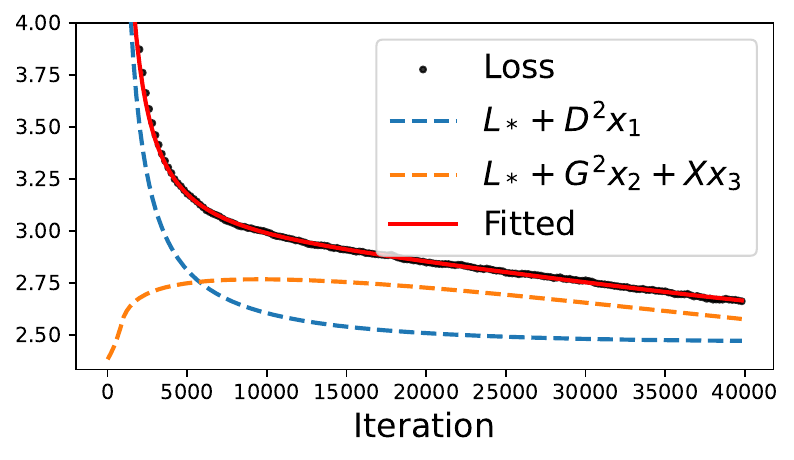}
\includegraphics[width=0.24\linewidth]{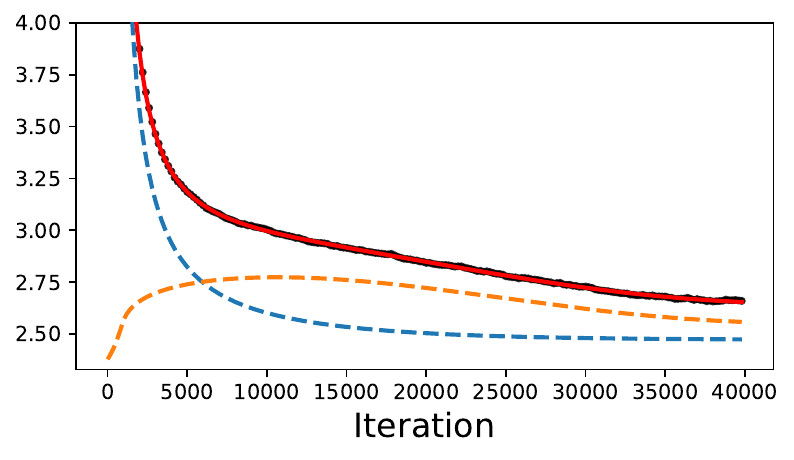}
\includegraphics[width=0.24\linewidth]{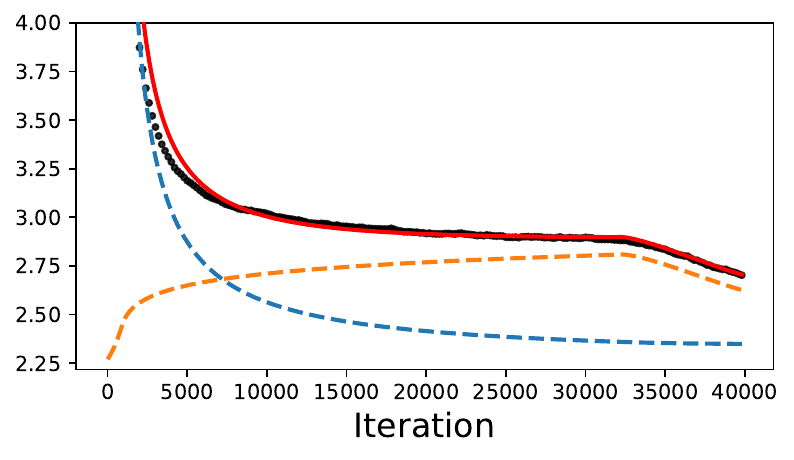}
\includegraphics[width=0.24\linewidth]{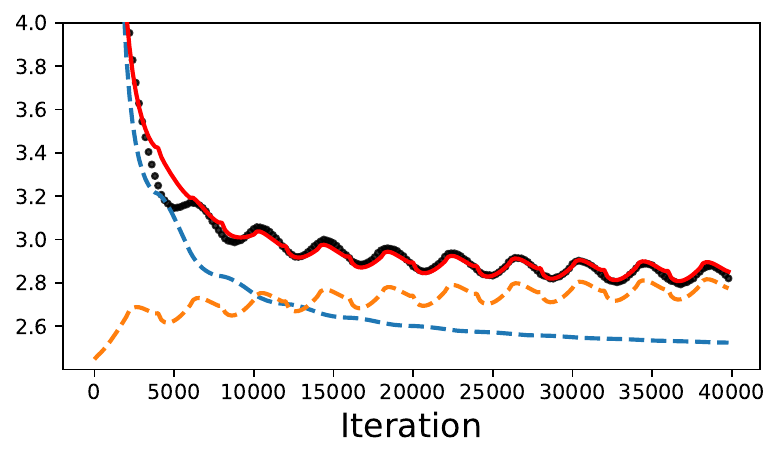}

\end{tcolorbox}

\begin{tcolorbox}[
    colback=white,
    colframe=blue!60,
    boxrule=1.2pt,
    arc=3mm,
    left=0mm,
    right=0mm,
    top=1mm,
    bottom=0mm,
    width=\linewidth,
    title={\textbf{constant learning rate, dynamic batch size}},
    coltitle=black,
    colbacktitle=blue!15,
    fonttitle=\normalsize,
    enhanced,
    attach boxed title to top center={
        yshift=-1.5mm
    },
    boxed title style={
        colback=blue!15,
        colframe=blue!60,
        arc=2mm,
        boxrule=1pt
    }
]

\centering

\includegraphics[width=0.24\linewidth]{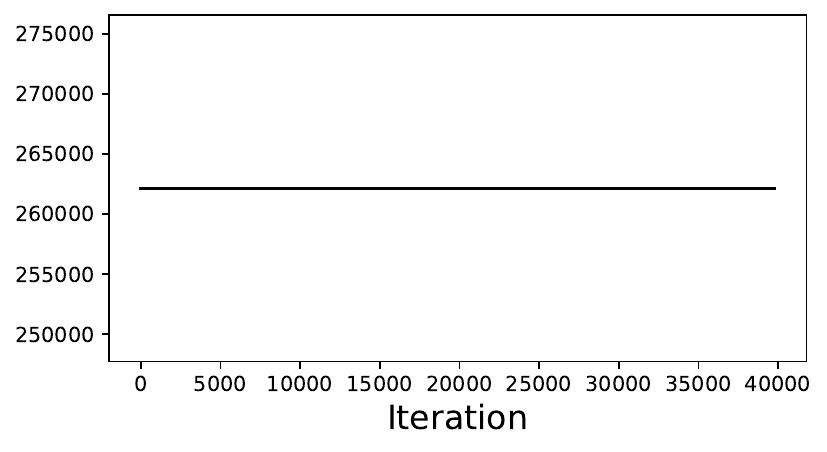}
\includegraphics[width=0.24\linewidth]{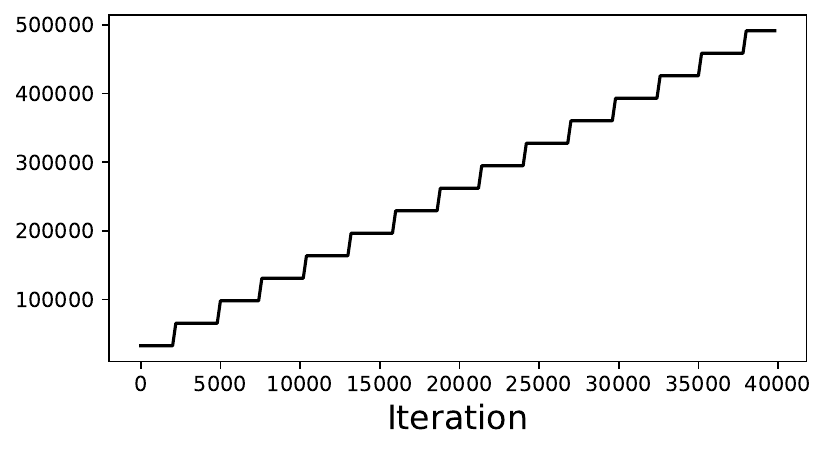}
\includegraphics[width=0.24\linewidth]{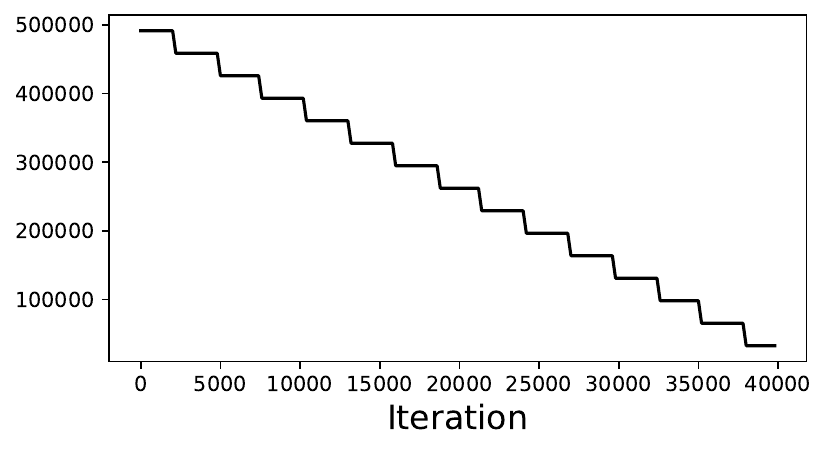}
\includegraphics[width=0.24\linewidth]{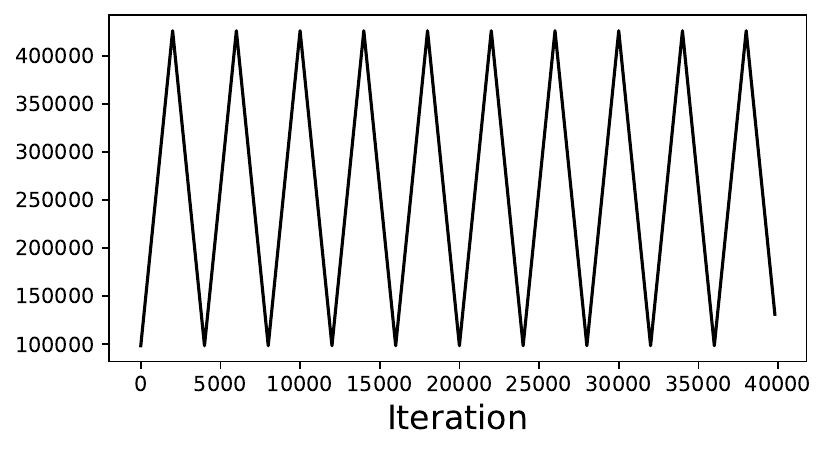}

\includegraphics[width=0.24\linewidth]{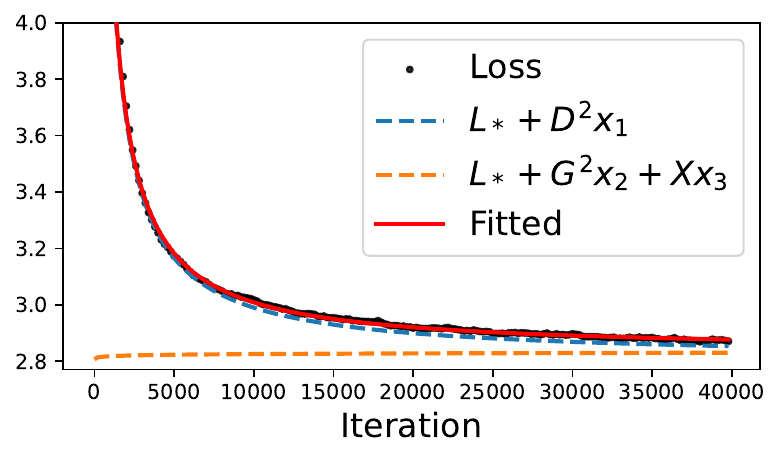}
\includegraphics[width=0.24\linewidth]{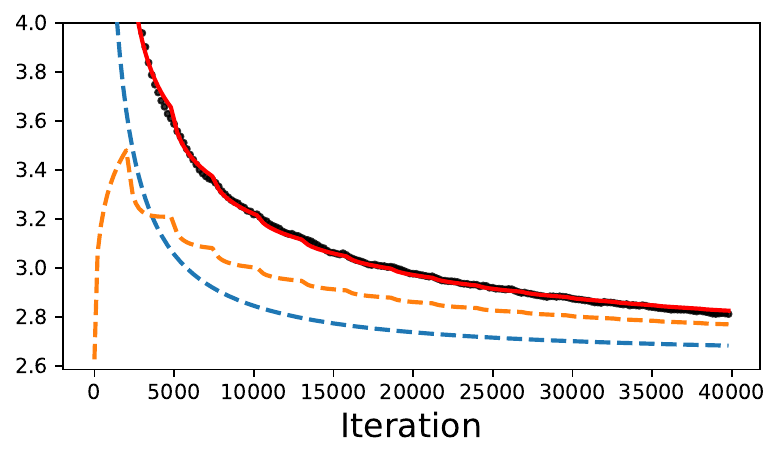}
\includegraphics[width=0.24\linewidth]{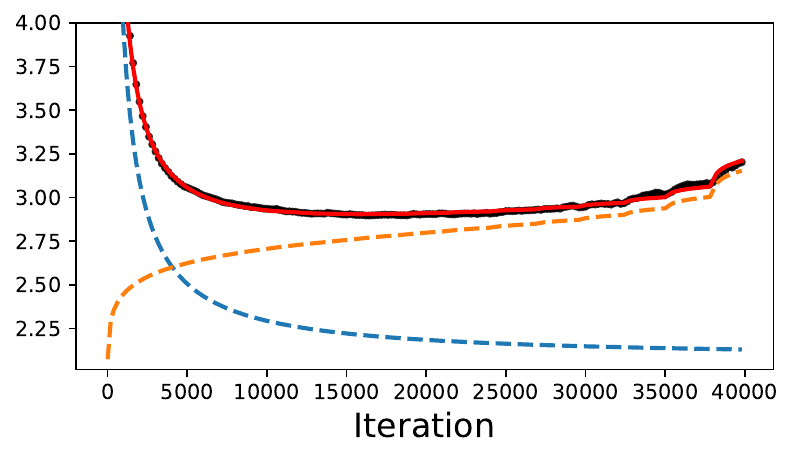}
\includegraphics[width=0.24\linewidth]{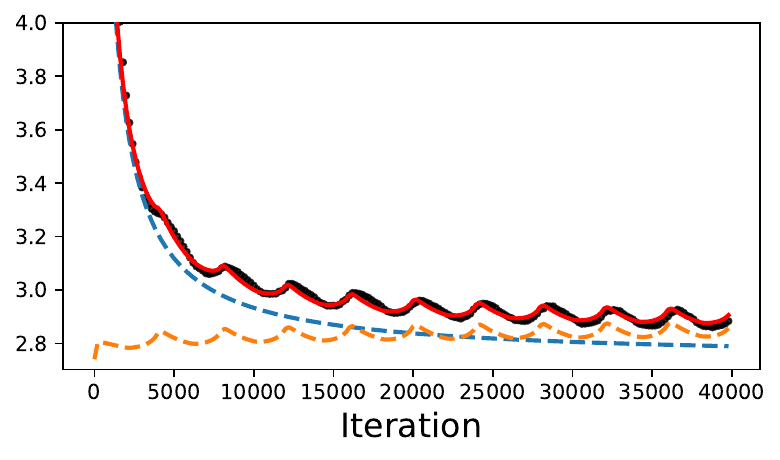}

\end{tcolorbox}
\caption{Sequence-to-sequence prediction by \Cref{thm:seq2seq} for Llama3-1B model on Fineweb-edu data, trained by Muon-NSGD optimizer. Top two rows: $B_t$ is fixed ($B=256*1024$) and $\eta_t$ is dynamic. Bottom two rows: $\eta_t$ is fixed ($\eta=0.01$) and $B_t$ is dynamic. Similar patterns for AdamW optimizer can be observed in \Cref{fig:nanogpt-adam}.
}
\label{fig:nanogpt-muon}
\end{figure}

As shown in \Cref{fig:nanogpt-muon}, for all cases, a precise characterization emerges after an initial transient phase. This indicates that the training dynamics become dominated by convexity, and that the inequality in \Cref{thm:seq2seq} may be replaced by equality in practice\footnote{$X/G^2$ is essentially the critical batch size $B_\text{simple}$ defined in \cite{mccandlish2018empirical}, which we treat as time-independent in line with \cite{merrillcritical,golmant2018computational}.}:
\begin{align}
\begin{split}
\E L(\w_\tau)\approx L_*+\frac{D^2}{2\sum_{t=1}^{\tau} \eta_t}+
\frac{1}{2}\sum_{t=1}^{\tau-1}\frac{\eta_t^2  (G^2+X/B_t)}{\sum_{k=t+1}^{\tau} \eta_k}+\frac{\eta_\tau}{2} (G^2+X/B_\tau)
\end{split}
\label{eq:loss mapping equality}
\end{align}

Specifically, if the learning rate decays to zero ($\eta_T=0$), then for the last iterate,
\begin{align}
\begin{split}
\E L(\w_T)\approx L_*+\frac{D^2}{2\sum_{t=1}^{T} \eta_t}+
\frac{1}{2}\sum_{t=1}^{T-1}\frac{\eta_t^2  (G^2+X/B_t)}{\sum_{k=t+1}^{T} \eta_k}
\end{split}
\label{eq:loss mapping last}
\end{align}

\section{Continuous loss prediction and optimal batch size schedule}
Now we optimize the batch size with respect to the last-iterate loss in \eqref{eq:loss mapping last}, subject to some constraints on the data size, compute budget, wall-clock training time, etc. To derive closed-form solutions, we will work in the continuous regime and replace the summations in \eqref{eq:loss mapping last} by integrals:
\begin{align}
L_*+\frac{D^2}{2\int_0^T \eta_t dt}+
\frac{G^2}{2}\int_0^T \left(\frac{\eta_t^2}{ \int_{t}^T\eta_k dk}\right)dt+\frac{X}{2}\int_0^T \left(\frac{\eta_t^2/B_t}{ \int_{t}^T\eta_k dk}\right) dt
\label{eq:loss mapping completely continuous}
\end{align}

For static batch size $B_t=B_\textup{static}$, we get
\begin{align}
L_\text{static}(T):=L_*+\frac{D^2}{2\int_0^T \eta_t dt}+
\frac{G^2}{2}\int_0^T \left(\frac{\eta_t^2}{ \int_{t}^T\eta_k dk}\right)dt+\frac{X}{2B_\textup{static}}\int_0^T \left(\frac{\eta_t^2}{ \int_{t}^T\eta_k dk}\right) dt
\label{eq:static loss}
\end{align}

\subsection{Data-optimal and compute-optimal batch size schedule}
We optimize the batch size given a fixed amount of data budget $K=B_\textup{static}T$, while the learning rate schedule and the number of total iterations $T$ are fixed. This is equivalent to optimizing the batch size given a fixed compute budget $C=6NK$ (FLOPs), where $N$ is model size.

It suffices to only minimize the last term of \eqref{eq:loss mapping completely continuous} with respect to $B_t$:
\begin{align}
\min_{B_t} \int_0^T \left( \frac{1}{B_t} \frac{\eta_t^2}{\int_t^T \eta_k dk}\right)dt,\quad \text{s.t.} \int_{0}^T B_t dt = K
\label{eq: constrained optim}
\end{align}
We solve this problem in \Cref{thm:optimal B} with a proof in \Cref{app:thm3}.
\begin{theorem}
\label{thm:optimal B}
The closed-form solution to \eqref{eq: constrained optim} is
\begin{align}
B_t^\textup{optim}=K\cdot\frac{b_t}{\int_0^T b_t dt}=\frac{K}{2}\frac{\eta_t}{\sqrt{\int_{0}^T\eta_k dk\int_{t}^T\eta_k dk}} \textup{ where }b_t:=\frac{\eta_t}{\sqrt{\int_{t}^T\eta_k dk}}
\label{eq:data optimal BS}
\end{align}
Importantly, the optimal batch size schedule only depends on the shape of the learning rate schedule, not the peak learning rate. \textbf{Thus, $B_t^\textup{optim}$ is fully decoupled from learning rate tuning and scale-free in terms of model size} (see \Cref{tab:our first} for the independence on $\eta_t$, and \Cref{fig:vary peak lr} for ablations on peak learning rate).
\end{theorem}

To give some concrete examples, we summarize the optimal batch size schedules corresponding to different learning rate schedules, which are visualized in \Cref{fig:visual} \footnote{We observe no empirical difference whether the learning rate warmup is calculated in $B_t^\text{optim}$ or not.}. The endpoint behavior of $B_t^{\mathrm{optim}}$ is governed by the \emph{order} at which the learning rate approaches zero: if $\eta_t \sim (T-t)^p$ as $t\to T$, then $b_t = \eta_t/\sqrt{\int_t^T \eta_k\,dk} \sim (T-t)^{(p-1)/2}$, so schedules with a nonzero terminal slope ($p=1$, e.g.\ linear and the WSD cooldown) yield a constant terminal batch size, whereas schedules that flatten to zero ($p=2$, e.g.\ cosine with $\eta_t \propto (T-t)^2$) drive $B_t^{\mathrm{optim}} \to 0$.

\begin{remark}
     Our characterization 
departs from the common simplification that the training dynamics are governed by the ratio $\eta_t/B_t$ alone \citep{smith2018dontdecay, goyal2017accurate, mccandlish2018empirical}, under which learning rate decay could be freely traded for batch size increase. We test this directly in \Cref{sec:eta_B_ratio}: for configurations with matched $\eta_t/B_t$ trajectories and identical token budgets, realizing the schedule through the learning rate consistently outperforms realizing it through the batch size (by $+0.017$ to $+0.743$ validation loss; see Table~\ref{tab:eta_B_ratio}). Hence $\eta_t/B_t$ is not a sufficient statistic for the training dynamics, and we justify that optimizing the batch size schedule is an independent dimension rather than as a re-parameterization of the learning rate.
\end{remark}

\begin{table}[!htb]
\caption{Optimal batch size schedules for  different learning rate schedules.}
\centering
\resizebox{\linewidth}{!}{
    \begin{tabular}{c|c|c|c}
    \hline
        learning rate & $\eta_t$ formula &$b_t$ formula& $B_t^\textup{optim}$ formula\\\hline\hline
      constant  & $\eta$& $\sqrt{\frac{\eta}{T-t}}$&$\frac{B_\text{static}}{2}\sqrt{\frac{T}{T-t}}$\\\hline
      cosine  & $\frac{\eta}{2}(\cos(\frac{\pi t}{T})+1)$&$\sqrt{\frac{\eta}{2}}\frac{1+\cos(\pi t/T)}{\sqrt{T-t-T\sin(\pi t/T)/\pi}}$&$
\frac{
B_\text{static}\,\cos^2\!\left(\frac{\pi t}{2T}\right)
}{\,\sqrt{
1-\frac{t}{T}-\frac{1}{\pi}\sin\frac{\pi t}{T}
}
}$\\\hline
      linear & $\eta(1-t/{T})$& $\sqrt{\frac{2\eta}{T}}$ &$B_\text{static}$\\\hline
      WSD & $\begin{cases}
  \eta&\textit{ if } t\leq cT\\
  \eta\frac{T-t}{T-cT}&\textit{ if } t> cT         
      \end{cases}$ & $\begin{cases}
  \sqrt{\frac{2\eta}{T+cT-2t}}  &\text{ if }t\leq cT\\
  \sqrt{\frac{2\eta}{T-cT}}  &\text{ if }t> cT
\end{cases}$&$\begin{cases}
  \frac{B_\text{static}T}{\sqrt{(T+cT-2t)(T+cT)}}  &\text{ if }t\leq cT\\ \frac{B_\text{static}}{\sqrt{1-c^2}}  &\text{ if }t> cT
\end{cases}$ \\
\hline
\end{tabular}
}
    \label{tab:our first}
\end{table}

\begin{figure}[!htb]
    \centering
    \includegraphics[width=0.324\linewidth]{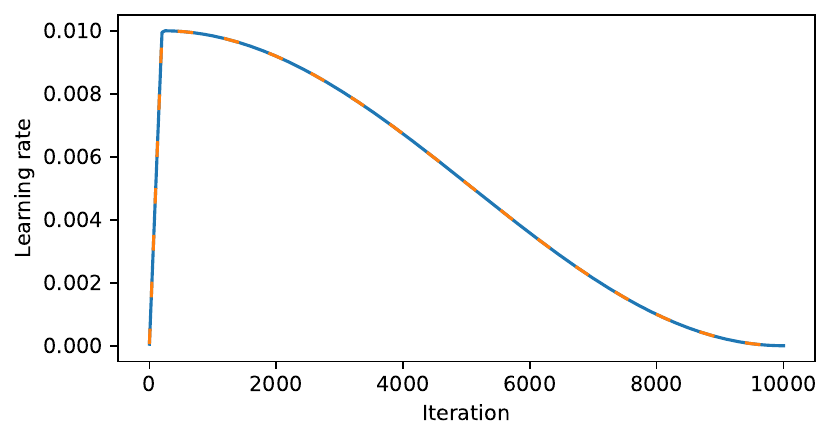}
    \includegraphics[width=0.324\linewidth]{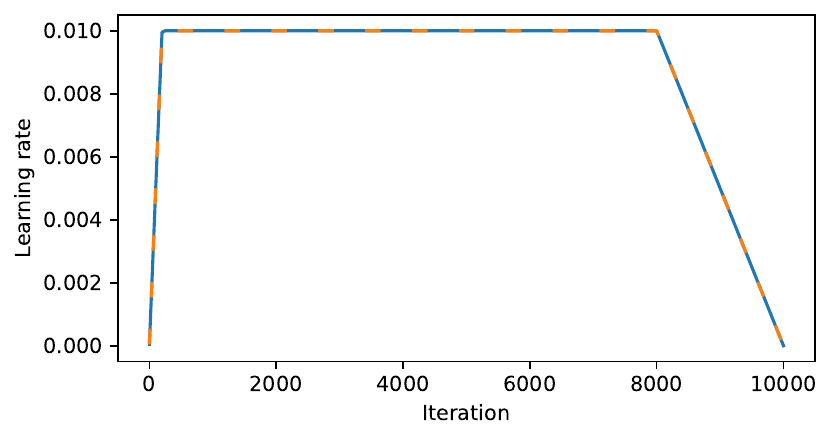}
    \includegraphics[width=0.33\linewidth]{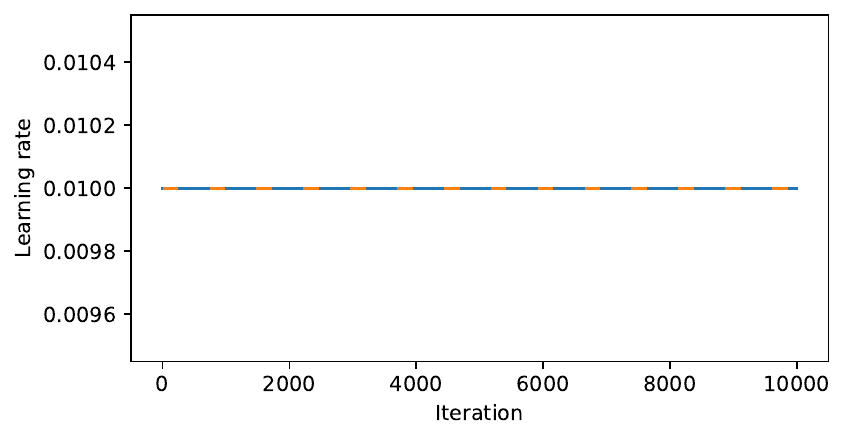}
    \\
    \includegraphics[width=0.324\linewidth]{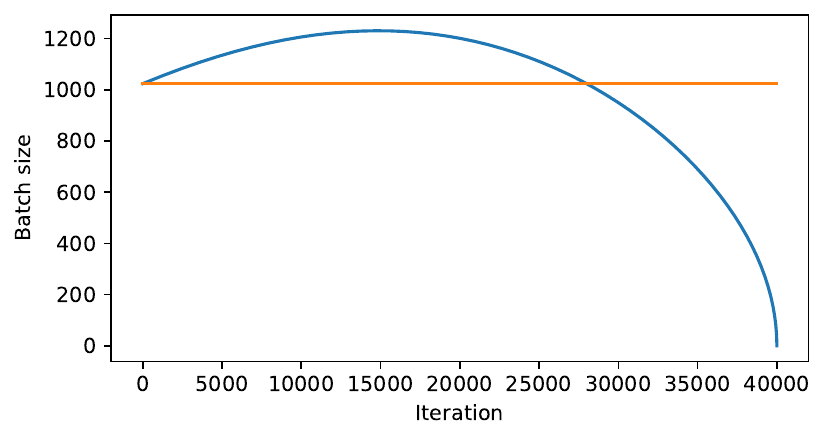}
    \includegraphics[width=0.324\linewidth]{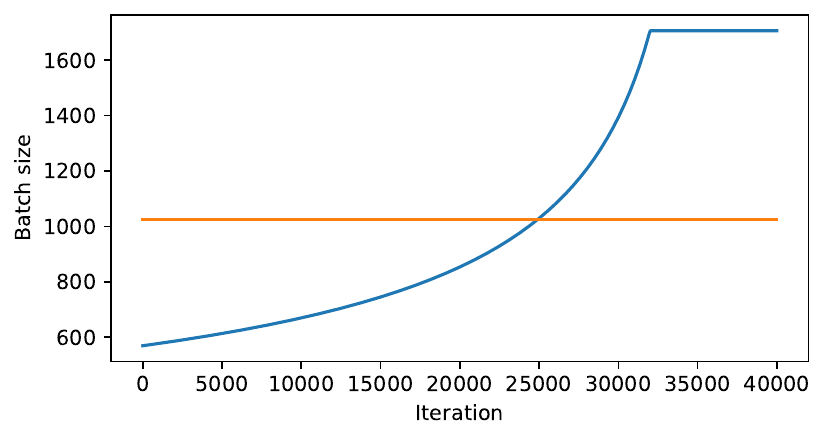}
\includegraphics[width=0.324\linewidth]{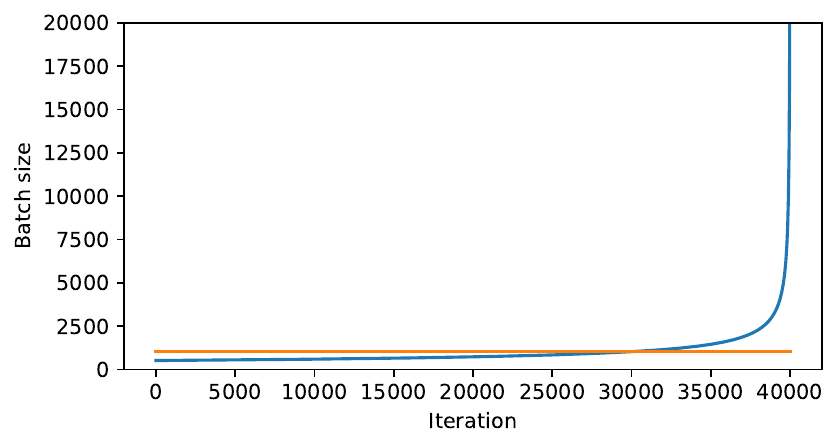}
\vspace{-0.2cm}
    \caption{Optimal batch size schedules (\textcolor{blue}{blue}) by \Cref{thm:optimal B} and static schedules (\textcolor{orange}{orange}) for cosine, WSD, and constant learning rate schedules.
 }
    \label{fig:visual}
\end{figure}

\subsection{Optimal batch size leverages variance reduction}
To demonstrate the theoretical advantage of the optimal batch size schedule, we firstly derive the last-iterate loss.
\begin{corollary}
\label{thm:dynamic loss}
Substituting \eqref{eq:data optimal BS} back to \eqref{eq:loss mapping completely continuous}, we have
$$L_\textup{dynamic}(T):=L_*+\frac{D^2}{2\int_0^T \eta_t dt}+
\frac{G^2}{2}\int_0^T \left(\frac{\eta_t^2}{ \int_{t}^T\eta_k dk}\right)dt+\frac{X}{2B_\textup{static}T}\left(\int_0^T \frac{\eta_t}{\sqrt{\int_{t}^T\eta_k dk}} dt \right)^2$$
\end{corollary}
Secondly, we compare $L_\textup{dynamic}$ in \Cref{thm:dynamic loss} to $L_\text{static}$ in \eqref{eq:static loss}, where the only difference is in the last terms. This difference explains the advantage of optimal batch size through the lens of variance reduction, as we show in \Cref{cor:variance reduction}.
\begin{corollary}
\label{cor:variance reduction}
Denoting $Z_t:=\frac{\eta_t}{\sqrt{\int_{t}^T\eta_k dk}}$, we show that the difference is always non-positive, 
$$L_\textup{dynamic}(T)-L_\textup{static}(T)=\frac{XT}{2B_\textup{static}}\left[\left(\frac{\int_0^T Z_t dt}{T}\right)^2- \frac{\int_0^T Z_t^2dt}{T}\right]\leq 0$$
\end{corollary}
\begin{proof}[Proof of \Cref{cor:variance reduction}]
Jensen's inequality, or Cauchy-Schwarz inequality applied to $Z_t$ and $1$.
\end{proof}

Additionally, we derive three facts from \Cref{tab:theorem1}.
\begin{itemize}
    \item \textbf{Variance reduction:} Our batch size schedule reduces the gradient covariance term with $\frac{X}{B_\text{static}}$. For example, dynamic batch size improves $1.061\frac{X}{B_\text{static}}\to \frac{X}{B_\text{static}}$ for cosine learning rate, and $2.47\frac{X}{B_\text{static}}\to 1.9\frac{X}{B_\text{static}}$ for WSD learning rate with 10\% cooldown.
    \item \textbf{Scaled learning rate:} The optimal peak learning rate for our batch size schedule is $1/\sqrt{T}$-scaled for qualified learning rate, by minimizing $L(T)-L_*$ over $\eta$ in the third column.
    \item \textbf{Optimal loss convergence rate:} Our batch size schedule consequently achieves $O(1/\sqrt{T})$ loss convergence in the last column.
\end{itemize}
\begin{table}[!htb]
\centering
\caption{Loss prediction by dynamic and static batch size schedules for different learning rate schedules.}
\resizebox{\linewidth}{!}{
    \begin{tabular}{c|c|c|c}
    \hline
        learning rate & batch size& $L(T)-L_*$ &optimal $L(T)-L_*$
        \\\hline
         \multirow{2}{*}{constant}
         &dynamic& $\frac{D^2}{2T\eta} + \frac{\eta}{2}G^2\ln T + \frac{2\eta X}{B_\textup{static}}$&$D\sqrt{G^2\frac{\ln T}{T}+4X/{B_\textup{static}}\frac{1}{T}}$
         \\
         &static& $ \frac{D^2}{2T \eta} + \frac{\eta}{2}G^2 \ln T+\frac{\eta X}{2B_\textup{static}} \ln T$& $ D \sqrt{G^2\frac{\ln T}{T}+X/B_\textup{static}\frac{\ln T}{T}}$ \\\hline     
         \multirow{2}{*}{cosine}
         &dynamic& $\frac{D^2}{T\eta}+{1.061\eta}G^{2}+\frac{X\eta}{B_\text{static}}$
&$2 D \sqrt{1.061G^2+X/B_\textup{static}} \sqrt{\frac{1}{T}}$
\\
         &static&$\frac{D^2}{T \eta}+1.061\eta G^2+1.061\frac{X\eta}{B_\textup{static}}$ & $2 D \sqrt{1.061G^2+1.061X/B_\textup{static}} \sqrt{\frac{1}{T}}$ \\\hline
       \multirow{2}{*}{WSD}
      &dynamic&$\frac{D^2}{(1+c) T\eta}+\eta G^2\left[1+\frac{1}{2}\ln\frac{1+c}{1-c}\right]+\eta X/B_\textup{static}(1+c)$ &$\frac{2D}{\sqrt{(1+c)T}}\sqrt{\left[1+\frac{1}{2}\ln\frac{1+c}{1-c}\right]G^2+(1+c)\frac{X}{B_\textup{static}}}$ 
\\
         &static&$\frac{D^2}{(1+c)T\eta}+ \eta (G^2+X/B_\textup{static})
\left[1+\frac{1}{2}\ln\left(\frac{1+c}{1-c}\right)\right]$& $  \frac{2D}{\sqrt{(1+c)T}}\sqrt{G^2+\frac{X}{B_\textup{static}}} \sqrt{1+\frac{1}{2}\ln\left(\frac{1+c}{1-c}\right)}$ \\
\hline
          \end{tabular}
}
\label{tab:theorem1}
\end{table}

\begin{figure}[!htb]
    \centering
\includegraphics[width=0.32\linewidth]{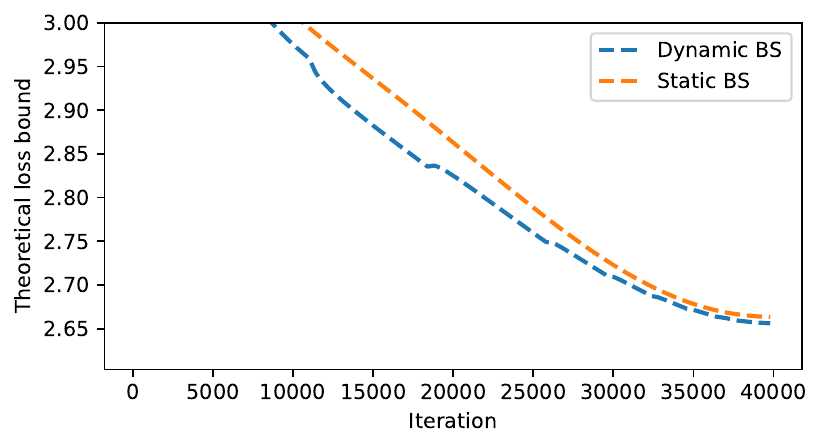}
    \includegraphics[width=0.32\linewidth]{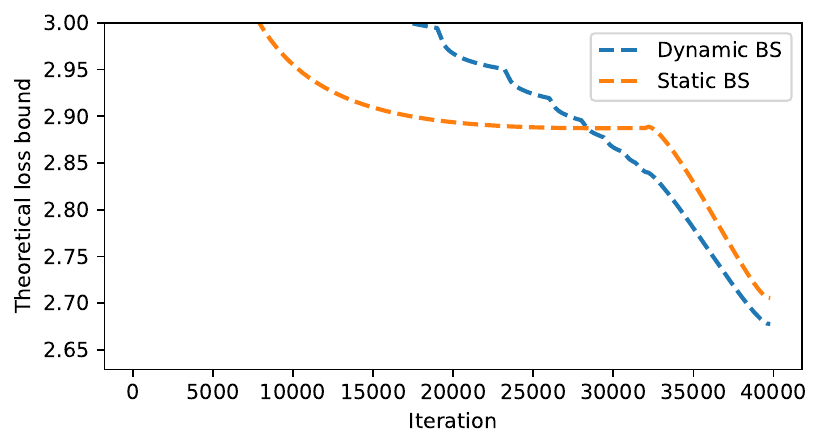}
\includegraphics[width=0.32\linewidth]{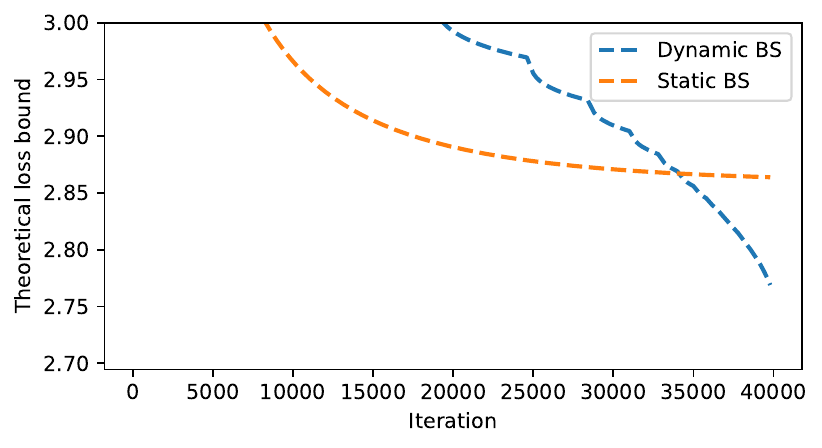}
  \\
    \includegraphics[width=0.32\linewidth]{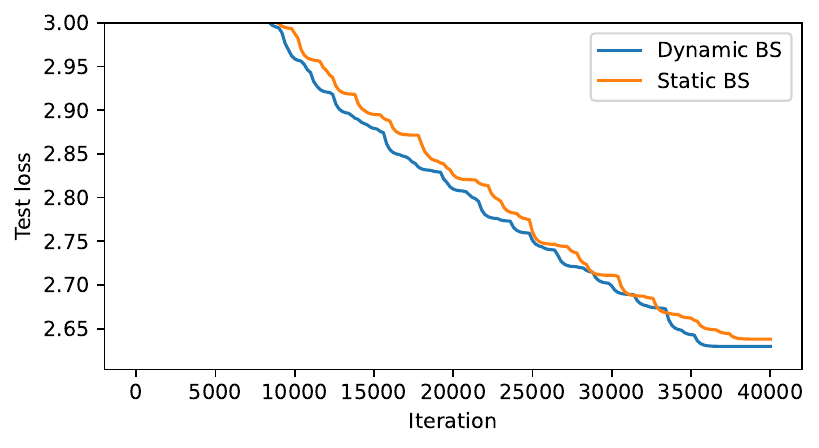}
    \includegraphics[width=0.32\linewidth]{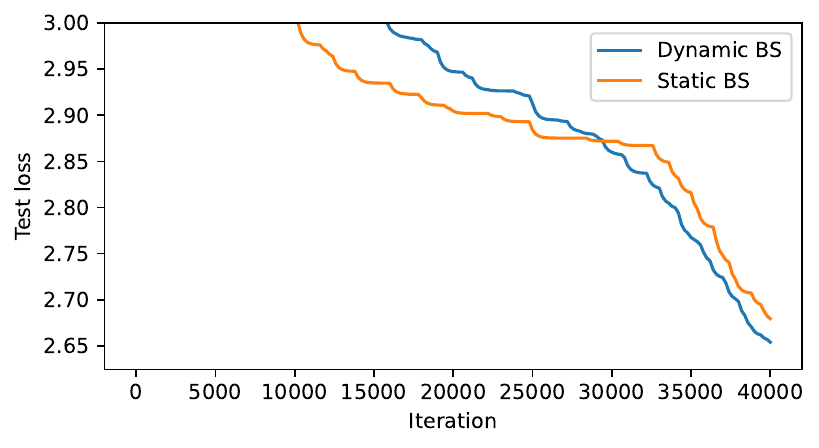}
    \includegraphics[width=0.32\linewidth]{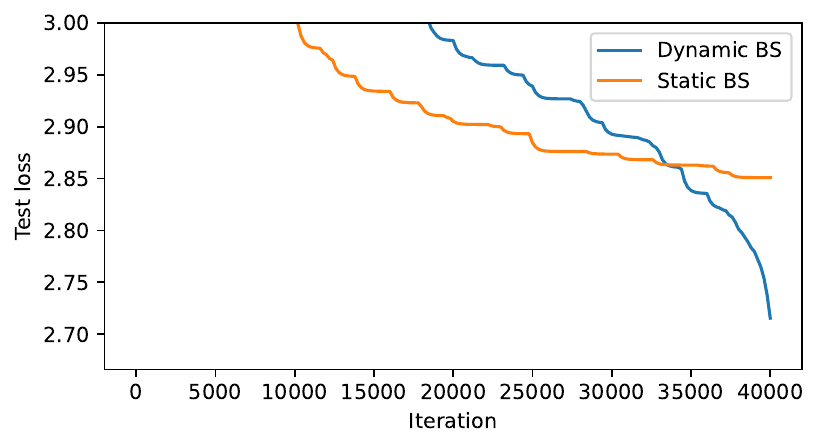}
  \\
    \includegraphics[width=0.32\linewidth]{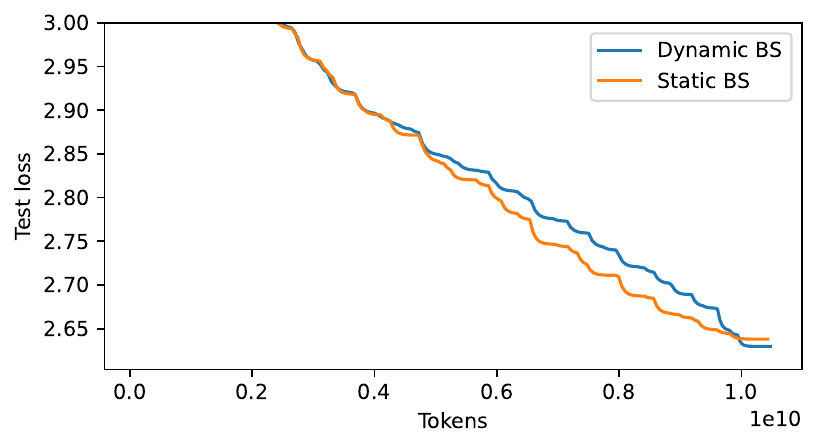}
    \includegraphics[width=0.32\linewidth]{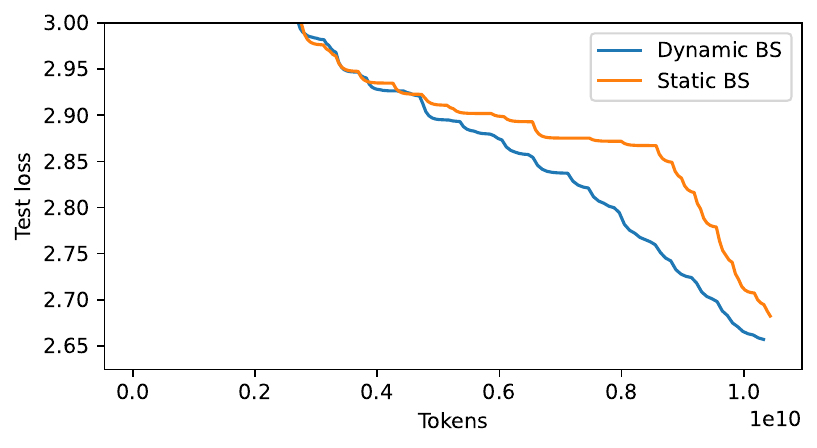}
    \includegraphics[width=0.32\linewidth]{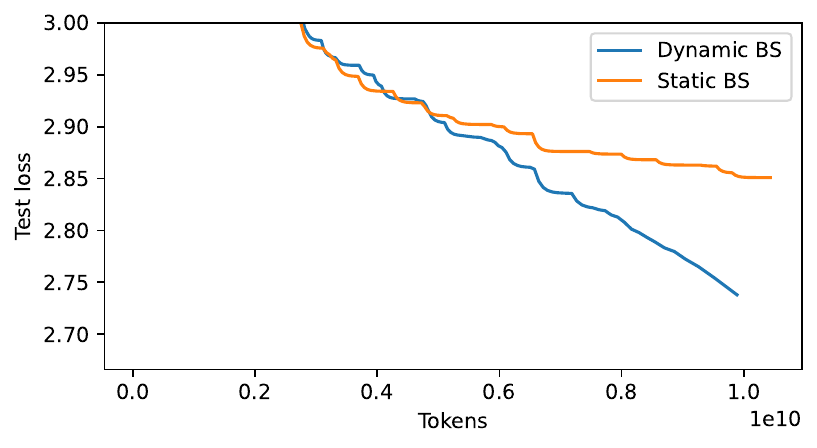}
\vspace{-0.2cm}
    \caption{Dynamic batch size consistently outperforms static batch size for Llama3-1B model on Fineweb-edu data, trained by Muon-NSGD optimizer. Left to right: cosine, WSD, and constant learning rate. Relative improvement of best perplexity is 0.8\%, 2.5\%, and 14.6\%, respectively.
}
    \label{fig:dynamicBS}
\end{figure}

\subsection{General computational cost-optimal batch size schedule}
Theoretically, we can derive the optimal batch size for any general computational cost function $f$, as an extension of \eqref{eq: constrained optim}, e.g. the wall-clock training time. 
\begin{align}
\min_B \int_0^T \left( \frac{\eta_t^2}{B_t} \frac{1}{\int_t^T \eta_s ds}\right) dt,\quad  s.t. \int_0^T f(B_t) dt= K
\label{eq: constrained optim2}
\end{align}
Notice that when $f(B)=B$, the following reduces to the data-optimal case in \Cref{thm:optimal B}. We obtain the closed form solution to \eqref{eq: constrained optim2} as $B_t^*(\lambda)$ in \Cref{lemma:general cost} (see proof in \Cref{app:thm4}).
\begin{lemma}
\label{lemma:general cost}
Denote $g^{-1}(x)=x^2 f'(x)$, then we can determine $\lambda$ and $B_t^*$ through
$$\int_0^T f\left(g\left(\frac{\eta_t^2}{\int_t^T \eta_s ds}/\lambda\right)\right) dt = K,
\quad B_t^*=g\left(\frac{\eta_t^2}{\int_t^T \eta_s ds}/\lambda\right)$$
\end{lemma}

\section{Joint scaling laws and universal dynamics}
\label{sec:scaling}

\subsection{Hyperparameter scaling}
With our batch size schedule in \Cref{thm:optimal B}, we present the joint scaling laws of multiple hyperparameters, which extends the single scaling law of peak learning rate in \citet{bu2026convex}: with $\eta^*(N,T)$ and $\kappa^*(N,T)$ being the optimal peak learning rate and weight decay for model size $N$ and horizon $T$, 
\begin{align}
\begin{split}
\text{peak learning rate: }&\eta^*(N,T)= \eta^*(N_\text{small},T_\text{small})/\sqrt{T/T_\text{small}}
\\
\text{optimal batch size: }&B_t^{\text{optim}}(T)=\frac{\text{total data}}{2}\frac{\eta^*_t}{\sqrt{\int_{0}^T\eta^*_k dk\int_{t}^T\eta^*_k dk}}
\\
\text{weight decay: }&\kappa(N,T)=\kappa^*(N_\text{small},T_\text{small})/\sqrt{T/T_\text{small}}
\end{split}
\label{eq:loss scaling law}
\end{align}

We theoretically validate \eqref{eq:loss scaling law} in \Cref{thm:joint wd}, showing that $1/\sqrt{T}$ scaling of peak learning rate and weight decay leads to asymptotically optimal loss convergence.

\begin{theorem}
\label{thm:joint wd}
Assuming \eqref{eq:grad condition}, convex loss and bounded parameters, for SGD with weight decay $\w_{t+1}=(1-\eta_t\kappa)\w_t-\eta_t\g_t$ and for qualified schedule with $\eta_\text{ref}/\sqrt{T}$ peak learning rate, we have 
$$\E L(\w_T)- L_*=O(1/\sqrt{T})$$
under static batch size and optimal batch size schedule if $\kappa=O(1/\sqrt{T})$.
\end{theorem}

We demonstrate on the Llama3 and Qwen3 MoE families that the optimal batch size schedule is highly beneficial, on a wide range of model sizes (from 0.1B to 7B) and compute budgets (up to 1e21 FLOPs). We visualize two experiments in \Cref{fig:scaling}, one with fixed compute and the other with fixed TPP. For the static batch size, we determine the optimal learning rate by following \cite{bu2026convex}: the peak learning rate is scaled as $\eta_\text{ref}/\sqrt{T}$, independent of model size, where $\eta_\text{ref}$ is tuned on small-scale runs. Weight decay is determined in the same manner. We emphasize that our batch size schedule directly uses the same learning rate and weight decay tuned for the static batch size, rendering the comparison fair but conservative. Nevertheless, in both experiments, the optimal batch size significantly outperforms the static batch size, with 15\% improvement of compute efficiency on Llama3 as well as 6\% improvement on Qwen3 MoE (see the middle column). Especially, the performance is consistent for most iterations during training.

\begin{figure}[!htb]
    \centering
\includegraphics[width=0.312\linewidth]{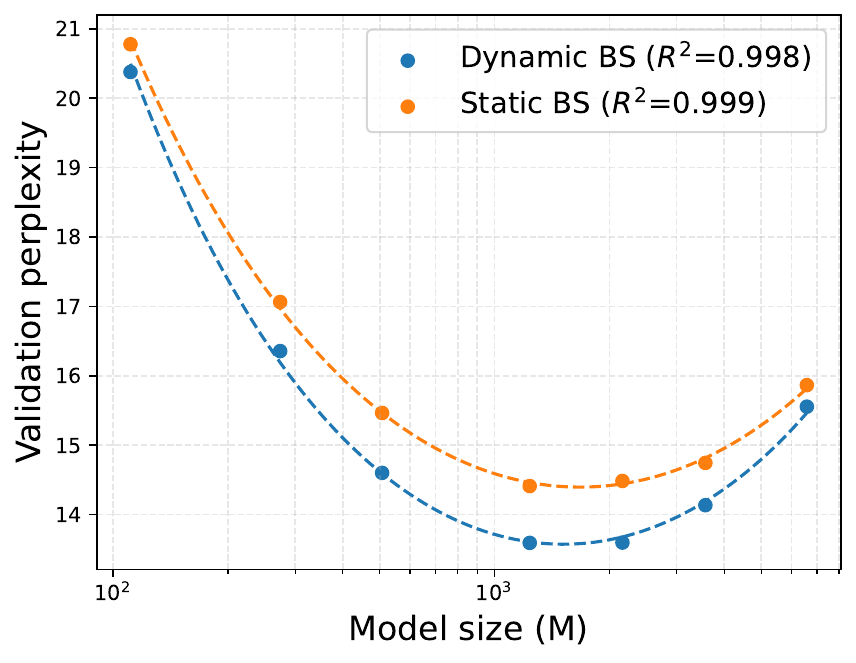}
\includegraphics[width=0.312\linewidth]{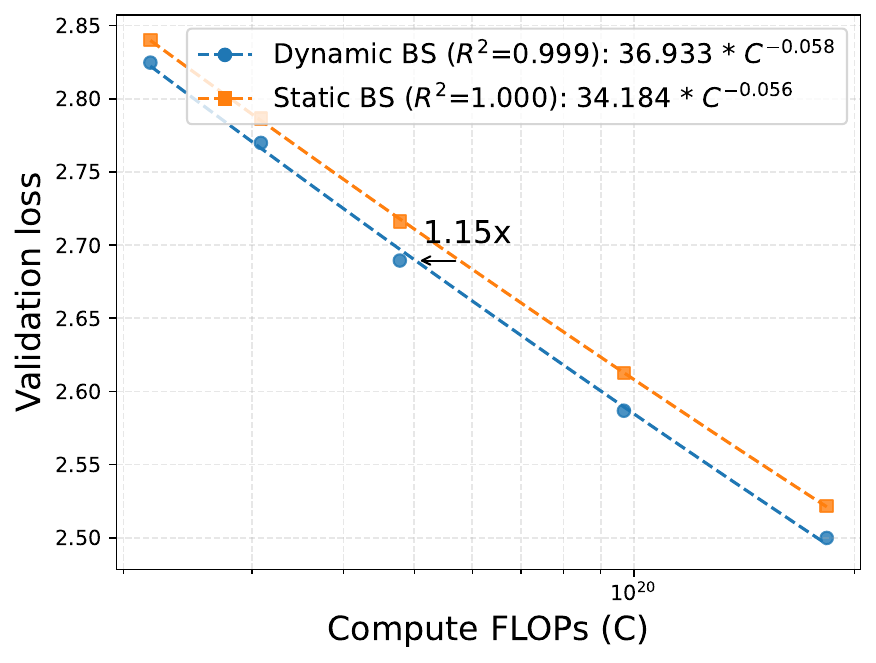}
\includegraphics[width=0.32\linewidth]{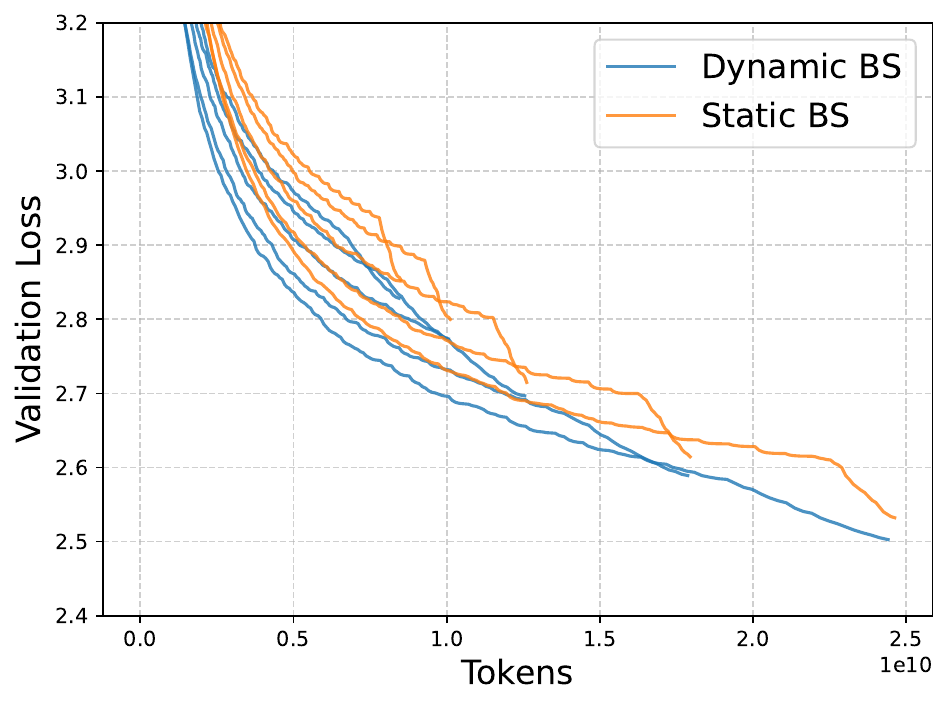}\\
\includegraphics[width=0.312\linewidth]{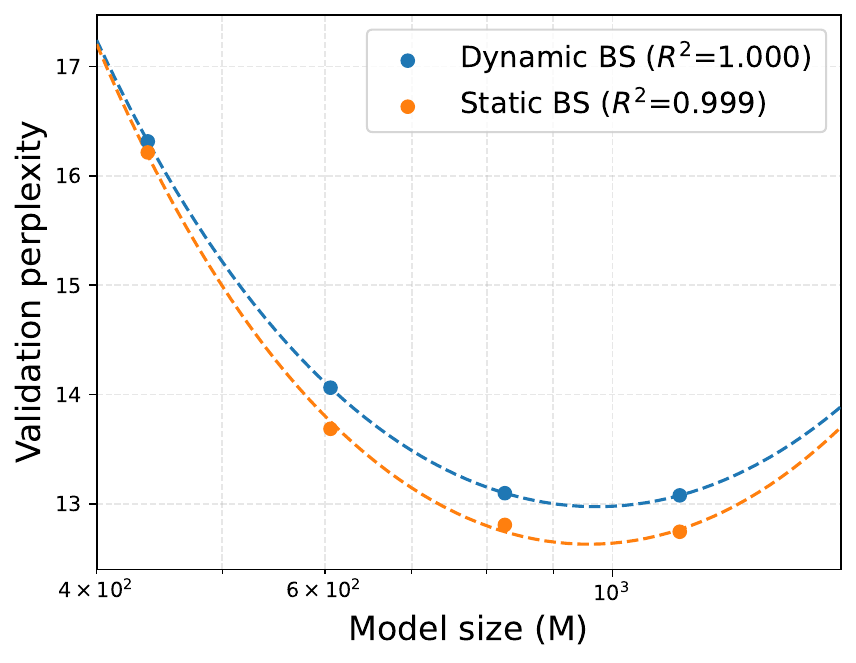}
\includegraphics[width=0.312\linewidth]{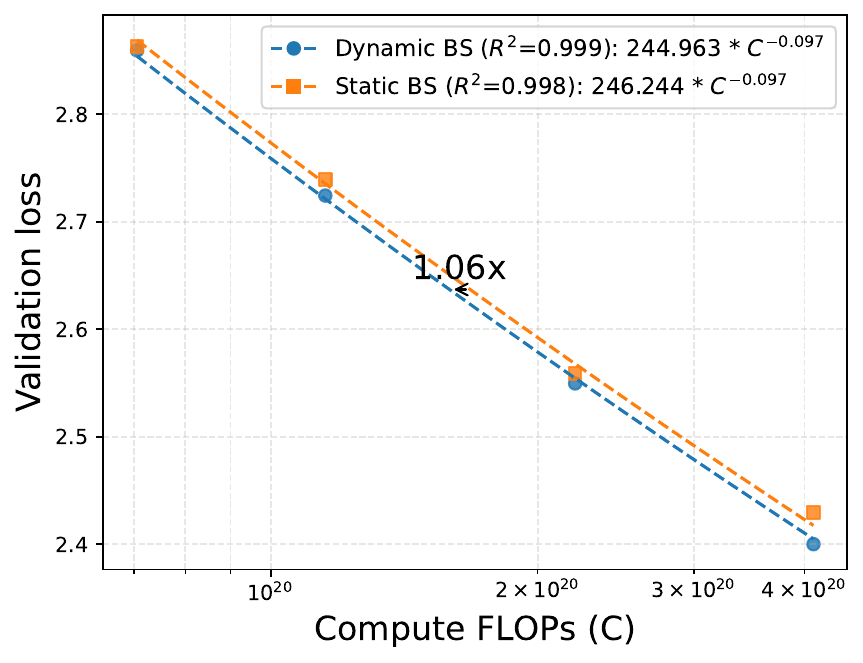}
\includegraphics[width=0.32\linewidth]{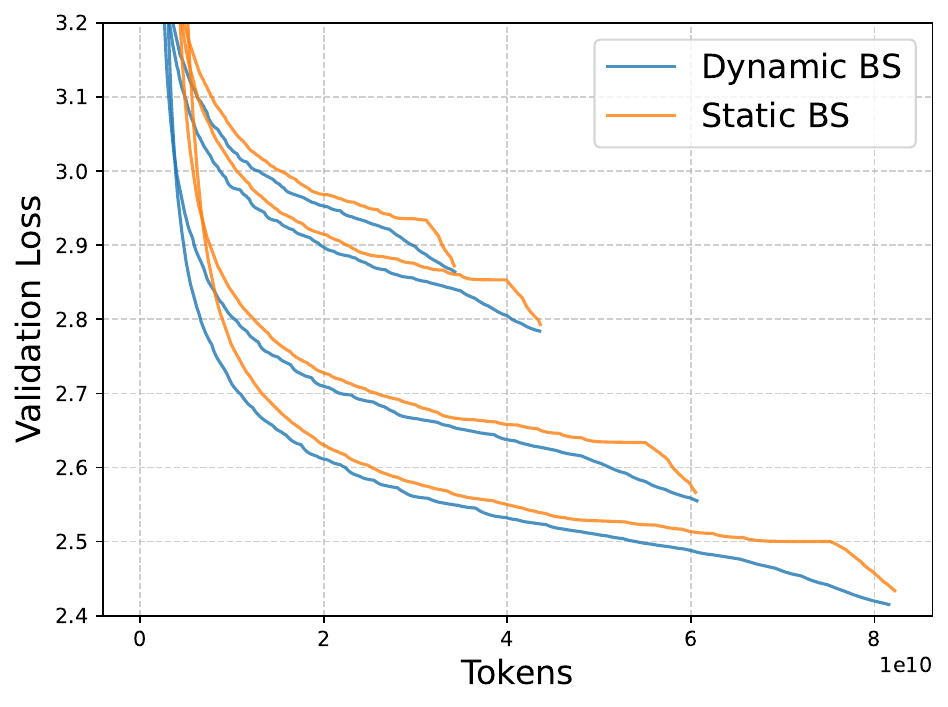}
\caption{Dynamic batch size significantly outperforms static batch size for Llama3 (top row) and Qwen3 MoE (bottom row). Left: perplexity under fixed compute of 1e20 FLOPs. Quadratic functions are fitted between validation loss and log(model size) following \cite{hoffmann2022chinchilla}. Middle \& Right: TPP=50 for Llama3 and TPP=100 for Qwen3 MoE. Power laws are fitted between validation loss and compute FLOPs.}
\label{fig:scaling}
\end{figure}

\subsection{Universal dynamics}
Crucially, our joint scaling laws lead to the universal dynamics \citep{qiu2025scaling,bergsma2025scaling}, where different models share the same normalized loss curve when 
\begin{enumerate}
    \item [(I)] TPP is fixed (governing depth, width, and training horizon).
    \item [(II)] Hyperparameters are scaled (governing learning rate and weight decay).
\end{enumerate}
In \Cref{fig:collapse}, each loss curve is normalized on the training horizon $x=t/T$ and on the validation loss $F(x)=\frac{L(\w_{xT})-L_*}{L(\w_{T})-L_*}$. In addition, we observe another universal dynamics when the loss is normalized with respect to tokens $x=\frac{\int_0^t B_k dk}{\int_0^T B_k dk}$, besides to iterations. Notice that these two types of normalization are equivalent for static batch size but different for dynamic batch size. As a result, our joint scaling laws enable the robust prediction of future scaling and demonstrate the advantage of optimal batch size.
\begin{figure}[!htb]
    \centering
    \includegraphics[width=0.3045\linewidth]{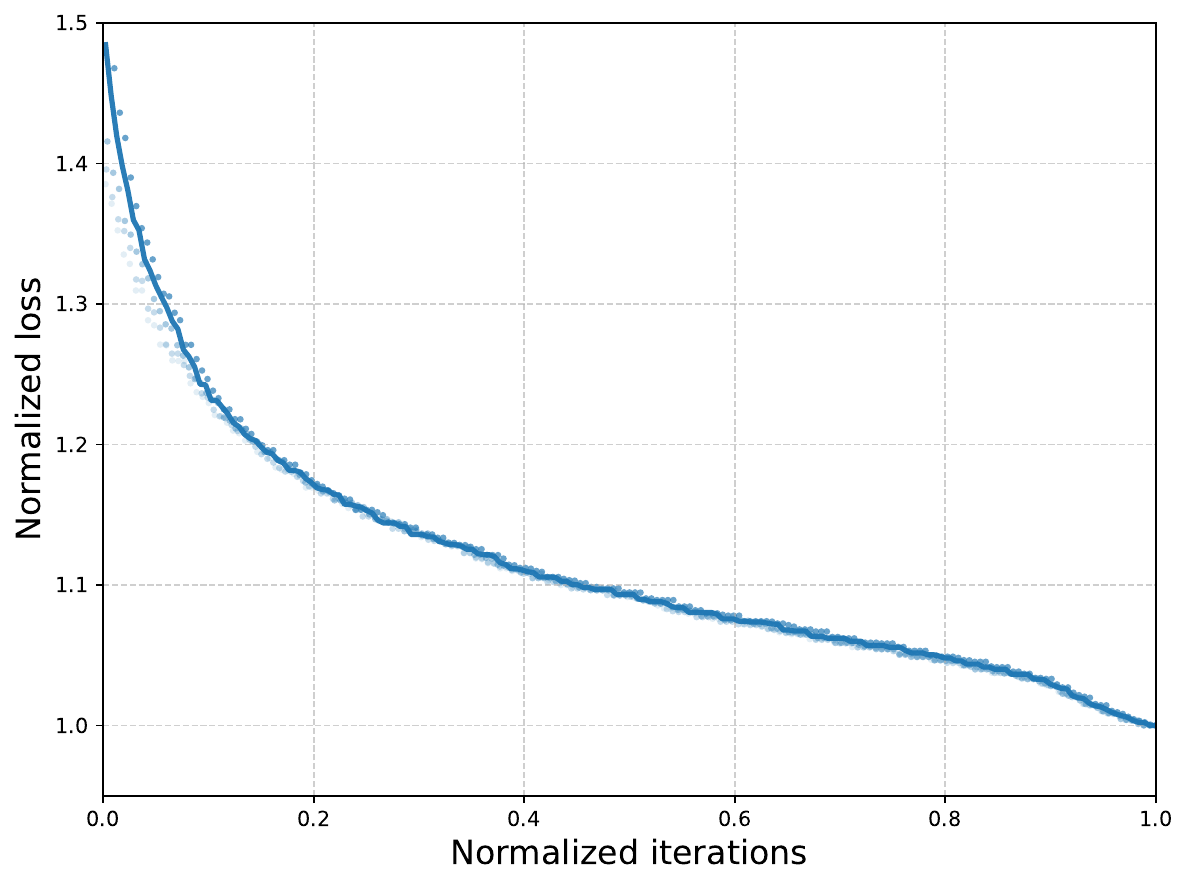}
    \includegraphics[width=0.3045\linewidth]{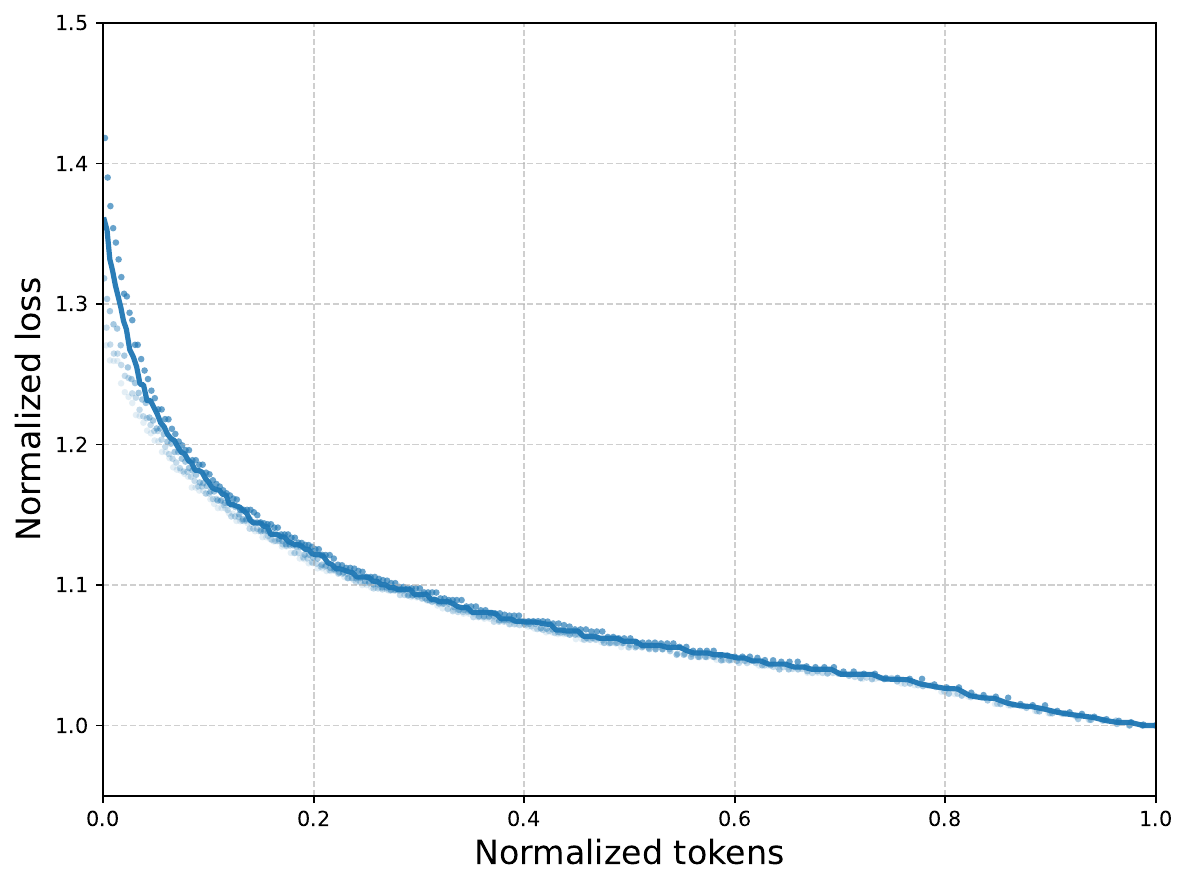}
    \includegraphics[width=0.3045\linewidth]{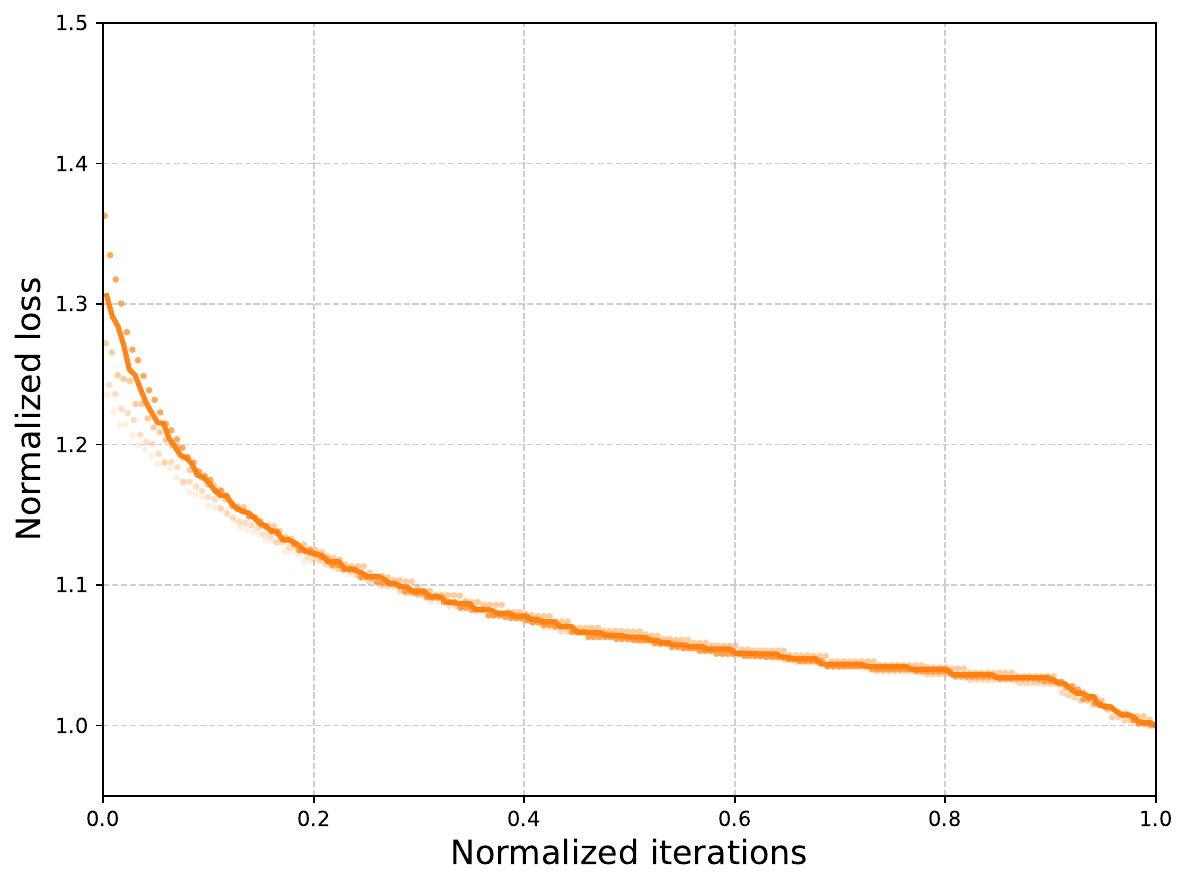}
\caption{Dynamic batch size has universal dynamics with respect to normalized iterations and normalize tokens, similar to the static batch size for which two types of normalization are equivalent. Each plot shows 4 runs of Qwen3 MoE with TPP=100 from 0.4 to 2B parameters, with lighter color being smaller model size .}
    \label{fig:collapse}
\end{figure}

\section{Post-training experiments}
\label{sec:posttrain}
We further test our batch size schedule on two post-training experiments---vision-language model (VLM) and math-domain fine-tuning. See additional experiment details in \Cref{app:posttrain}.

\subsection{Vision-language model}
\label{sec:vlm}
We train vision-language models (VLM) with SmoLM2-360M-Instruct as language backbone and Siglip2-large as vision backbone, following nanoVLM codebase \citep{wiedmann2025nanovlm} on the Cauldron dataset \citep{laurencon2024matters}. We train with Muon-NSGD optimizer. The static batch size is 0.26M tokens for text and 512*512 pixels for image.

\begin{figure}[!htb]
    \centering
    \includegraphics[width=0.425\linewidth]{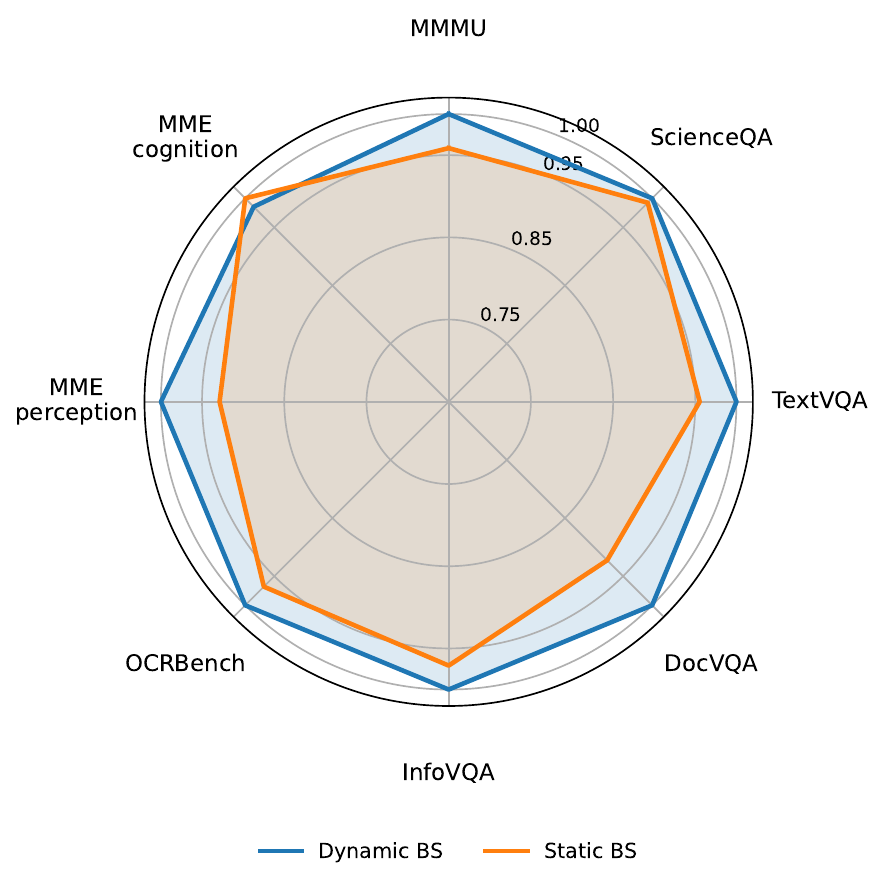}    \includegraphics[width=0.385\linewidth]{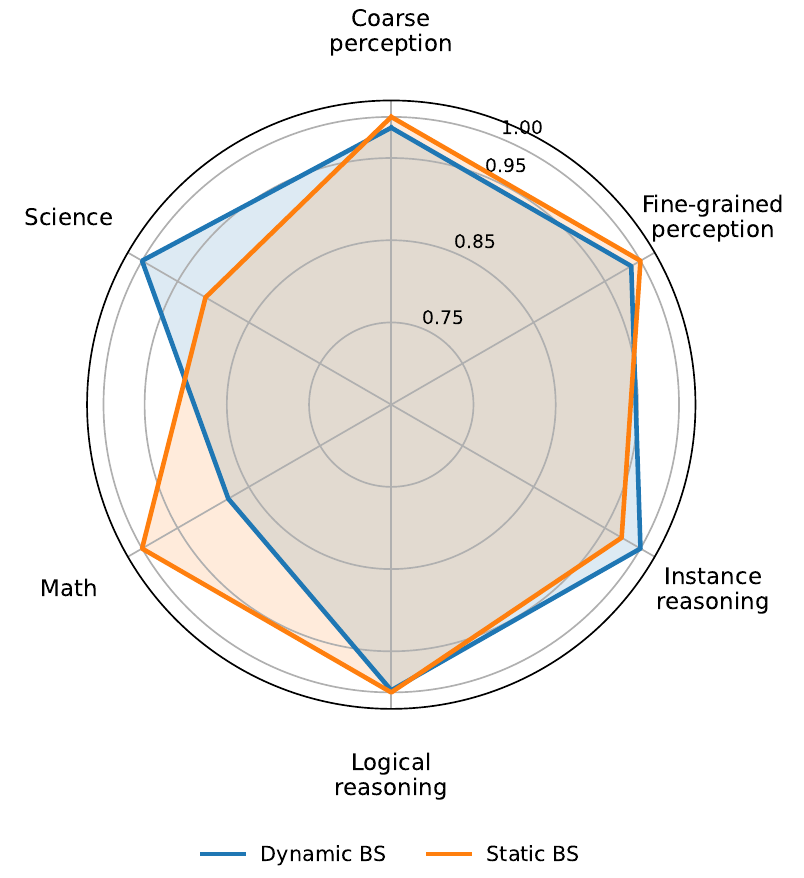}
    \caption{Normalized radar comparison of dynamic vs.\ static batch size across VLM benchmarks. Left: MMStar subtasks. Right: other benchmarks (MMMU, ScienceQA, TextVQA, DocVQA, InfoVQA, OCRBench, and MME). Each axis is normalized by $\max(\text{Dynamic BS},\text{Static BS})$ so that mixed-scale metrics are visually comparable. See \Cref{tab:vlm_results} for absolute scores.}
    \label{fig:vlm_radar}
\end{figure}

\Cref{fig:vlm_radar} and \Cref{tab:vlm_results} compare the dynamic and static batch sizes across various VLM evaluation tasks, where the dynamic batch size outperforms the static baseline on 7 of 9 tasks, while the static baseline can be slightly better on some MMStar subtasks. Overall, we see that dynamic batch size is competitive or superior for VLM tasks, with no significant regression on any benchmark.

\subsection{Math-domain fine-tuning} 
We fine-tune Qwen3-0.6B \citep{qwen3technicalreport} on \texttt{orca\_math} dataset \citep{mitra2024orcamath}  with AdamW optimizer and WSD learning rate schedule. The peak learning rate is selected via a sweep with a constant batch size. In \Cref{fig:math500_placeholder}, dynamic batch size shows better overall trajectories on both validation loss and GSM8K accuracy than static batch size.

\begin{figure}[!htb]
    \centering
    \includegraphics[width=0.24\linewidth]{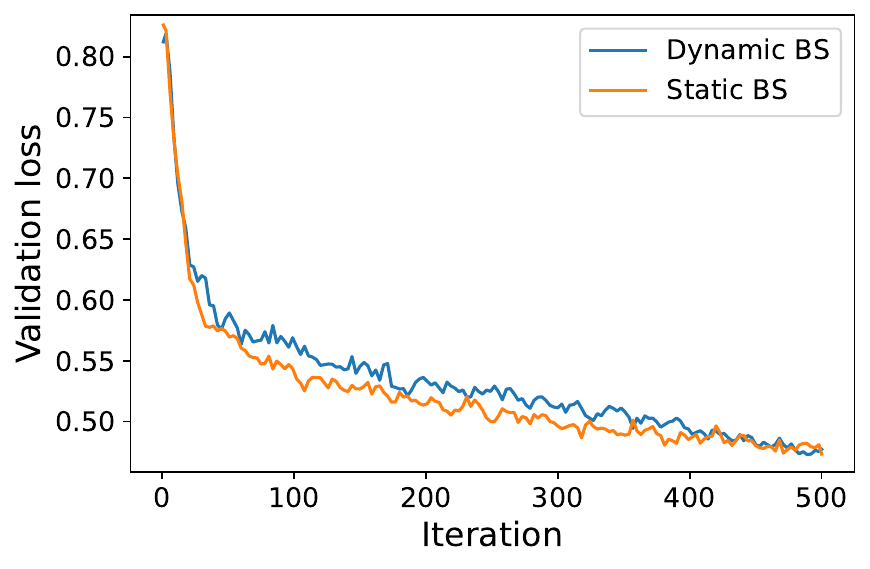}
    \includegraphics[width=0.24\linewidth]{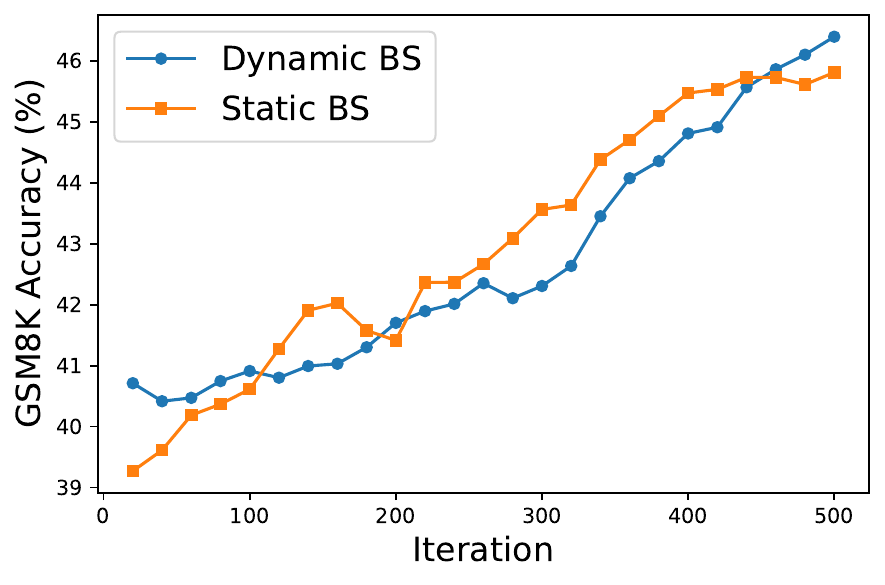}
    \includegraphics[width=0.24\linewidth]{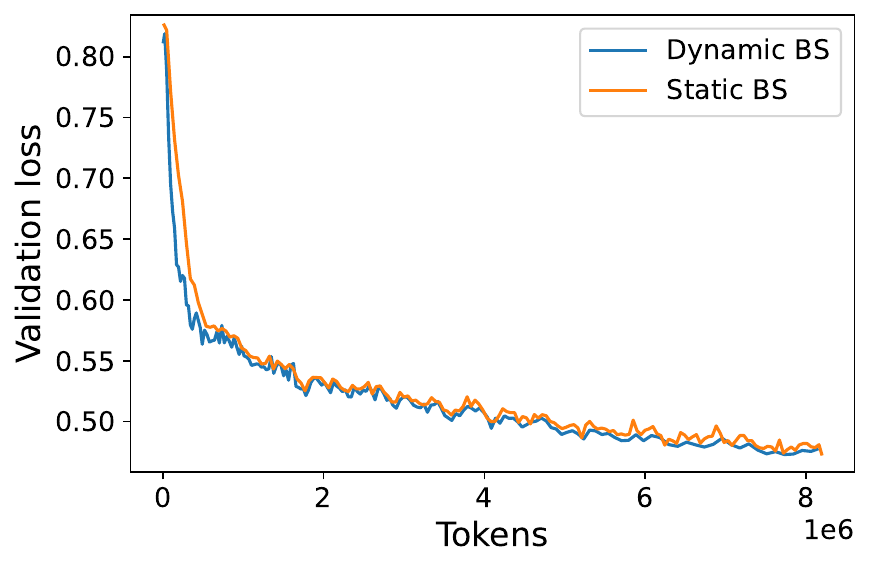}
    \includegraphics[width=0.24\linewidth]{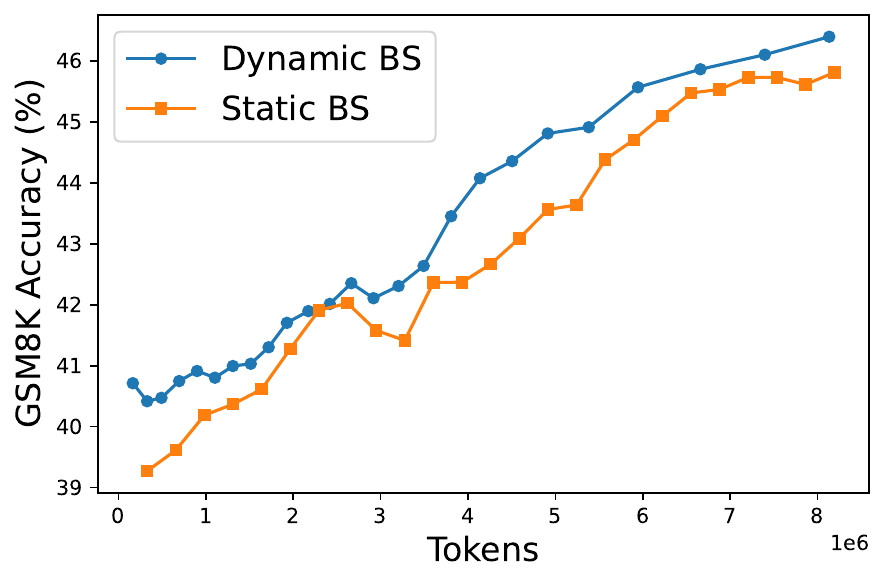}
    \caption{Dynamic batch size consistently outperforms static batch size on math-domain fine-tuning. Left: validation loss and 5-shot GSM8K accuracy over iterations. Right: the same metrics plotted against total tokens consumed.}
    \label{fig:math500_placeholder}
\end{figure}

\section{Conclusion}
In this work, we validated that the convex optimization theory can characterize the training dynamics of deep learning in \eqref{eq:loss mapping equality}, from the perspectives of learning rate and batch size schedules. This allows us to reverse-engineer and derive an optimal batch size schedule, which is decoupled from the choice of learning rate (no hyperparameter tuning) and consistently outperforms static batch size. As a result, we build joint scaling laws that inform the scaling of hyperparameters, and lead to better scaling of model performance.

Throughout the paper, we ablate the most representative configurations to test the optimal batch size schedule: optimizers (Muon-NSGD and AdamW), model families (Llama3 and Qwen3 MoE), model sizes and training horizons in \Cref{fig:scaling}, and learning rate schedules in \Cref{fig:nanogpt-muon} and \Cref{fig:dynamicBS}. We provide additional ablations in the \Cref{app:ablation}: (I) We observe that our batch size schedule always outperforms the static batch size for any learning rate, even the non-optimal ones. (II) While our batch size always outperforms the static baseline, the advantage vanishes in theory as $B\to \infty$ yet is practically significant even when the base batch size is increased $8\times$. (III) However, the advantage may vanish under purely low-precision training.

\clearpage
\newpage
\bibliographystyle{assets/plainnat}
\bibliography{references}

\clearpage
\newpage
\beginappendix

\section{Theoretical results}

\subsection{\Cref{thm: simple}}
\label{app:thm1}
\begin{proof}[Proof of \Cref{thm: simple}]
Theorem 10 in \cite{defazio2023optimal} gives an upper bound of the loss
\begin{align}
\E L(\w_{\tau})\leq L_*+\frac{D^2}{2\sum_{t=1}^T \eta_t}+
\frac{1}{2}\left(\frac{\sum_{t=1}^{\tau}\eta_t^2\E\|\g_t\|^2}{\sum_{t=1}^{\tau} \eta_t}+\sum_{k=1}^{\tau-1}\frac{\eta_k}{\sum_{t=k+1}^{\tau}\eta_t}\frac{\sum_{t=k}^{\tau} \eta_t^2\E\|\g_t\|^2}{\sum_{t=k}^{\tau} \eta_t}\right)
\label{eq:origin}
\end{align}

To simplify the notation, we write $$x_t = \eta_t^2\E\|\g_t\|^2, A_k=\frac{\eta_{k}}{\sum_{t=k+1}^{\tau} \eta_t \sum_{t=k}^{\tau}  \eta_t}=\frac{1}{\sum_{t=k+1}^{\tau} \eta_t}-\frac{1}{ \sum_{t=k}^{\tau}  \eta_t}$$
so that \eqref{eq:origin} becomes
\begin{align*}
&L_*+\frac{D^2}{2\sum_{t=1}^{\tau} \eta_t}+\frac{1}{2}\frac{1}{\sum_{t=1}^{\tau} \eta_t}\sum_{t=1}^{\tau}  x_t+\frac{1}{2}\sum_{k=1}^{\tau -1} \left(A_k\sum_{t=k}^{\tau}  x_t\right)
\\
=&L_*+\frac{D^2}{2\sum_{t=1}^{\tau} \eta_t}+\frac{1}{2}\frac{1}{\sum_{t=1}^{\tau}  \eta_t}\sum_{t=1}^{\tau-1}  x_t +\frac{1}{2}\frac{x_{\tau}}{\sum_{t=1}^{\tau} \eta_t} +\frac{1}{2}\sum_{k=1}^{\tau-1} \left(A_k\sum_{t=k}^{\tau}  x_t\right)
\\
=&L_*+\frac{D^2}{2\sum_{t=1}^{\tau} \eta_t}+\frac{1}{2}\frac{1}{\sum_{t=1}^{\tau}  \eta_t}\sum_{t=1}^{{\tau}-1}  x_t +\frac{1}{2}x_{\tau}\frac{1}{\sum_{t=1}^{\tau} \eta_t} +\frac{1}{2}x_{\tau}\sum_{k=1}^{{\tau}-1}A_k +\frac{1}{2}\sum_{k=1}^{{\tau}-1} \left(A_k\sum_{t=k}^{{\tau}-1}  x_t\right)
\end{align*}
We exchange the double sum to yield
$\sum_{k=1}^{{\tau}-1} \left(A_k\sum_{t=k}^{{\tau}-1}  x_t\right)=\sum_{t=1}^{{\tau}-1} \left( x_t\sum_{k=1}^t A_k\right)$, and obtain
\begin{align*}
&L_*+\frac{D^2}{2\sum_{t=1}^{\tau} \eta_t}+\frac{1}{2}x_{\tau}\left(\frac{1}{\sum_{s=1}^{\tau}  \eta_s}+\sum_{k=1}^{{\tau}-1} A_k\right)+\frac{1}{2}\sum_{t=1}^{{\tau}-1} x_t\left(\frac{1}{\sum_{s=1}^{\tau}  \eta_s}+\sum_{k=1}^t A_k\right)
\end{align*}
which can be further simplified to
\begin{align*}
&L_*+\frac{D^2}{2\sum_{t=1}^{\tau} \eta_t}+ \frac{1}{2}\frac{x_{\tau}}{\eta_{\tau}}+ \frac{1}{2}\sum_{t=1}^{{\tau}-1}  x_t\frac{1}{\sum_{s=t+1}^{\tau} \eta_s}
\end{align*}
because the telescoping sum gives $\sum_{k=1}^{t} A_k = \frac{1}{\sum_{s=t+1}^{\tau}\eta_s}-\frac{1}{\sum_{s=1}^{\tau}\eta_s}$. 
\end{proof}

\subsection{\Cref{eq:grad condition}}
\label{app:grad condition}

    According to the mean-covariance decomposition, 
    \begin{align}
        \E\|\g_t\|^2 = \|\E \g_t\|^2+\text{tr}(\text{Cov}(\g_t)).
    \end{align}

   Let $G$ be the upper bound of $\|\E\g_t\|$ and $X$ be the upper bound of trace of $\Sigma_t:= Cov_{\xi}(g(\w_t;\xi))$ which is the per-example stochastic gradient at iteration $t$, where 
$\xi$ is a data sample drawn from the data distribution.
Since
$$Cov(\g_t)=Cov\left(\frac{1}{B_t}\sum_{i=1}^{B_t}g(w_t; \xi_{t,i})\right)=\frac{1}{B_t}\Sigma_t$$, it gives
   $$tr(Cov(\g_t))=tr(\frac{1}{B_t}\Sigma_t)\leq\frac{X}{B_t}.$$
Therefore, 
\begin{align}
        \E\|\g_t\|^2 = \|\E \g_t\|^2+tr(Cov(\g_t))\leq G^2+\frac{X}{B_t}.
\end{align}

\begin{remark}
    For adaptive gradient descent $\w_{t+1}=\w_t-\eta_t \phi(\g_t)$, suppose some upper bounds $G,X$, then we obtain
$$\E\|\phi(\g_t)\|^2\leq G^2+X/B_t+O_p(1/B_t^2)$$
\end{remark}

\subsection{\Cref{thm:optimal B}}
\label{app:thm3}
\begin{proof}[Proof of \Cref{thm:optimal B}]
Using the notation of $b_t$, we rewrite the minimization problem \eqref{eq: constrained optim} as
\[
\min_{B_t} \int_0^T \frac{b_t^2}{B_t} \, dt
\text{ s.t.}
\int_0^T B_t \, dt = K.
\]

The Lagrangian functional associated with the equality constraint is, with $\lambda \in \R$,
\[
\int_0^T
    \frac{b_t^2}{B_t}dt + \lambda \left(\int_0^T B_t
dt
- K\right)
=\int_0^T \left(
    \frac{b_t^2}{B_t} + \lambda B_t
\right) dt
- \lambda K.
\]

Denote the integrand as $\ell_t(B_t) := \frac{b_t^2}{B_t} + \lambda B_t$. The first-order condition gives
\[
\frac{\partial \ell_t}{\partial B_t}
=
-\frac{b_t^2}{B_t^2} + \lambda=0
\Longrightarrow
B_t^\textup{optim} = \frac{b_t}{\sqrt{\lambda}},
\]

Next, we determine $\lambda$ from the constraint:
\[
\int_0^T B_t^\textup{optim} \, dt
=
\frac{1}{\sqrt{\lambda}} \int_0^T b_t\, dt
= K
 \Longrightarrow 
\frac{1}{\sqrt{\lambda}}
=
\frac{K}{\int_0^T b_t \, dt}
\Longrightarrow
B_t^\textup{optim}= K \cdot\frac{b_t}{\int_0^T b_t \, dt}.
\]
To make $B_t^\textup{optim}$ explicit on $\eta_t$, notice that
$$\int_0^T b_tdt=2\sqrt{\int_0^T \eta_kdk}.$$
\end{proof}

\subsection{\Cref{tab:our first} and \Cref{tab:theorem1}}
Here we derive $b_t$, $B_t^\text{optim}$ and $L(T)$. We omit the optimal $L(T)$ because it is trivial to minimize the formula $L_*+a/\eta+b\cdot\eta$.

\paragraph{constant learning rate}
We start from $\eta_t=\eta$. Firstly, we write 
$$
b_t:=\frac{\eta_t}{\sqrt{\int_{t}^T\eta_s ds}}
=\sqrt{\frac{\eta}{T-t}}
$$
Secondly, we get
$$\int_0^T b_t dt=2\sqrt{T\eta}$$
Thirdly, we get
$$B_t=BT\cdot \frac{b_t}{\int_0^T b_t dt}
= \frac{BT}{2\sqrt{T(T-t)}}
$$
Lastly, substituting $B_t$ and $\eta_t$ into the loss bound:
$$L(T)=L_*+\frac{D^2}{2T\eta} + \frac{\eta}{2}G^2\ln T + \frac{2\eta X}{B}$$

\paragraph{linear learning rate}
We start from $\eta_t=\eta(1-t/T)$. Firstly, we write 
$$
b_t:=\frac{\eta_t}{\sqrt{\int_{t}^T\eta_s ds}}
=\frac{\eta(1-\frac{t}{T})}{\sqrt{\frac{(T-t)^2}{2T}}}=\sqrt{\frac{2}{T}}
$$
Secondly, we get 
$$B_t=B$$
Lastly, substituting $B_t$ and $\eta_t$ into the loss bound:
$$L_*+\frac{D^2}{T \eta}+\eta (G^2+X/B)$$

\paragraph{WSD}

We start from $\eta_t=\begin{cases}
  \eta&\textit{ if } t\leq cT\\
  \eta\frac{T-t}{T-cT}&\textit{ if } t> cT         
      \end{cases}$.

Firstly, we write
$$
b_t:=\frac{\eta_t}{\sqrt{\int_{t}^T\eta_s ds}}
$$
where
$$\int_{t}^T\eta_s ds=
\begin{cases}
  \frac{\eta}{2}(T+cT-2t)  &\text{ if }t\leq cT\\
  \frac{\eta}{2}\frac{(T-t)^2}{(T-cT)}  &\text{ if }t> cT
\end{cases}$$
so that
$$b_t
=
\begin{cases}
  \sqrt{\frac{2\eta}{T+cT-2t}}  &\text{ if }t\leq cT\\
  \sqrt{\frac{2\eta}{T-cT}}  &\text{ if }t> cT
\end{cases}$$

Secondly, we get
$$\int_0^T b_t dt= \sqrt{2\eta(T+cT)}$$
Thirdly, we get
$$B_t=BT\cdot \frac{b_t^*}{\int_0^T b_t^* dt}= 
\begin{cases}
  \frac{BT}{\sqrt{(T+cT-2t)(T+cT)}}  &\text{ if }t\leq cT\\
  \frac{BT}{\sqrt{(T-cT)(T+cT)}} = \frac{B}{\sqrt{1-c^2}}  &\text{ if }t> cT
\end{cases}$$
Lastly, substituting $B_t$ and $\eta_t$ into the loss bound:
\begin{align*}
L_*+ \frac{D^2}{\eta(1+c) T}+\frac{\eta}{2}\left[2+\ln\frac{1+c}{1-c}\right]G^2+\eta(1+c)\frac{X}{B}.
\end{align*}

\paragraph{cosine learning rate}

We start from $\eta_t=\frac{\eta}{2}(1+\cos{(\frac{\pi t}{T})})$.

Firstly, we write 
$$b_t:=\frac{\eta_t}{\sqrt{\int_{t}^T\eta_s ds}}=
\cos^2\!\left(\frac{\pi t}{2T}\right)
\sqrt{
\frac{
\eta
}{
\frac{T}{2}
\left(
1 - \frac{t}{T} - \frac{1}{\pi}\sin\frac{\pi t}{T}
\right)
}
}.
$$
Secondly, we get 
$$B_t=BT\cdot \frac{b_t}{\int_0^T b_t dt}
=
\frac{
B\,\cos^2\!\left(\frac{\pi t}{2T}\right)
}{\,\sqrt{
1-\frac{t}{T}-\frac{1}{\pi}\sin\frac{\pi t}{T}
}
},
$$
Thirdly, substituting $B_t$ and $\eta_t$ into the loss bound:
$$L_*+\frac{D^2}{T\eta}+1.061\eta G^2+\frac{\eta X}{B}$$

\subsection{\Cref{lemma:general cost}}
\label{app:thm4}
\begin{proof}[Proof of \Cref{lemma:general cost}]
Similar to \Cref{thm:optimal B}, using the notation of $b_t$, we rewrite the minimization problem \eqref{eq: constrained optim2} as
\[
\min_{B_t} \int_0^T \frac{b_t^2}{B_t} \, dt
\text{ s.t.}
\int_0^T f(B_t) \, dt = K.
\]

The Lagrangian functional associated with the equality constraint is, with $\lambda \in \R$,
\[
\int_0^T
    \frac{b_t^2}{B_t}dt + \lambda \left(\int_0^T f(B_t)
dt
- K\right)
=\int_0^T \left(
    \frac{b_t^2}{B_t} + \lambda f(B_t)
\right) dt
- \lambda K.
\]

Denote the integrand as $\ell_t(B_t) := \frac{b_t^2}{B_t} + \lambda f(B_t)$. The first-order condition gives
\[
\frac{\partial \ell_t}{\partial B_t}
=
-\frac{b_t^2}{B_t^2} + \lambda f'(B_t)=0
\Longrightarrow
g^{-1}(B_t)=B_t^2 f'(B_t)=\frac{b_t^2}{\lambda}
\Longrightarrow
B_t^* = g\left(\frac{b_t^2}{\lambda}\right),
\]

Next, we determine $\lambda$ from the constraint:
\[
\int_0^T f(B_t^*) \, dt
= K
 \Longrightarrow 
\int_0^T f\left(g\left(\frac{\eta_t^2}{\int_t^T \eta_s ds}/\lambda\right)\right) dt = K
\Longrightarrow
B_t^*=g\left(\frac{\eta_t^2}{\int_t^T \eta_s ds}/\lambda\right)
\]
    
\end{proof}

\subsection{\Cref{thm:joint wd}}
\begin{proof}[Proof of \Cref{thm:joint wd}]
Denote $L_\kappa(\w)=L(\w)+\frac{\kappa}{2}\|\w\|^2$, then SGD with weight decay on $L$ is equivalent to SGD without weight decay on $L_\kappa$, where the bounded gradient condition becomes $\nabla L_\kappa=\nabla L+\kappa \w \Longrightarrow \|\nabla L_\kappa\|=\|\nabla L\|+\kappa \|\w\|\leq G+\kappa W$.

It is not hard to obtain  (c.f. \eqref{eq:loss mapping last})
$$\E L_\kappa(\w_T) - L_{\kappa,*}\leq \frac{D^2}{2\sum_{t=1}^{T} \eta_t}+
\frac{1}{2}\sum_{t=1}^{T-1}\frac{\eta_t^2  ((G+\kappa W)^2+X/B_t)}{\sum_{k=t+1}^{T} \eta_k}
$$
as long as the learning rate decays to zero for any fixed $T$.

Notice that weight decay does not impact the last term $X/B_t$, rendering the same minimization problem as \eqref{eq: constrained optim} and hence the same solution $B_t^\text{optim}$, i.e. our optimal batch size schedule.

For qualified learning rate schedule and scaled peak learning rate, we get $\Theta(1/\sqrt{T})$ loss by \Cref{tab:theorem1}, where the only difference is on $G$. Therefore, under either static batch size or optimal batch size, we get
$$\E L_\kappa(\w_T) - L_{\kappa,*}=O(1/\sqrt{T})$$

To translate the gap between regularized loss $L_\kappa$ to the gap between vanilla loss $L$, we need
$$\E L(\w_T) - L_*\leq \E L_\kappa(\w_T) - L_*\leq \E L_\kappa(\w_T)-L_{\kappa,*}+L_{\kappa,*}-L_*.$$
The first difference has already been derived. The second difference is obvious by contradiction:
$$L_{\kappa,*}-L_*\leq \frac{\kappa}{2}\|\w_*\|^2\leq \frac{\kappa}{2} W^2$$
All in all, we need
$$\E L(\w_T) - L_*\leq O(1/\sqrt{T})+\frac{\kappa}{2} W^2$$
which requires $\kappa=O(1/\sqrt{T})$.
\end{proof}

\section{Experiment details}
\subsection{\Cref{fig:longest}, \Cref{tab:longest}}
We train with WSD learning rate schedule and cooldown at 0.9; peak learning rate is 0.02. The optimizer is Muon-NSGD. Default initialization in Huggingface.

Llama3 model has 48 layer, hidden size 3072, intermediate size 3072*4, batch size 256, and sequence length 1024. Qwen3 MoE model has 3B total parameters, 32 layers, hidden size 2048, intermediate size 2048*3, moe intermediate size 2048/4, 16 experts, batch size 512, and sequence length 4096.

\begin{figure}[H]
    \centering
    \includegraphics[width=0.525\linewidth]{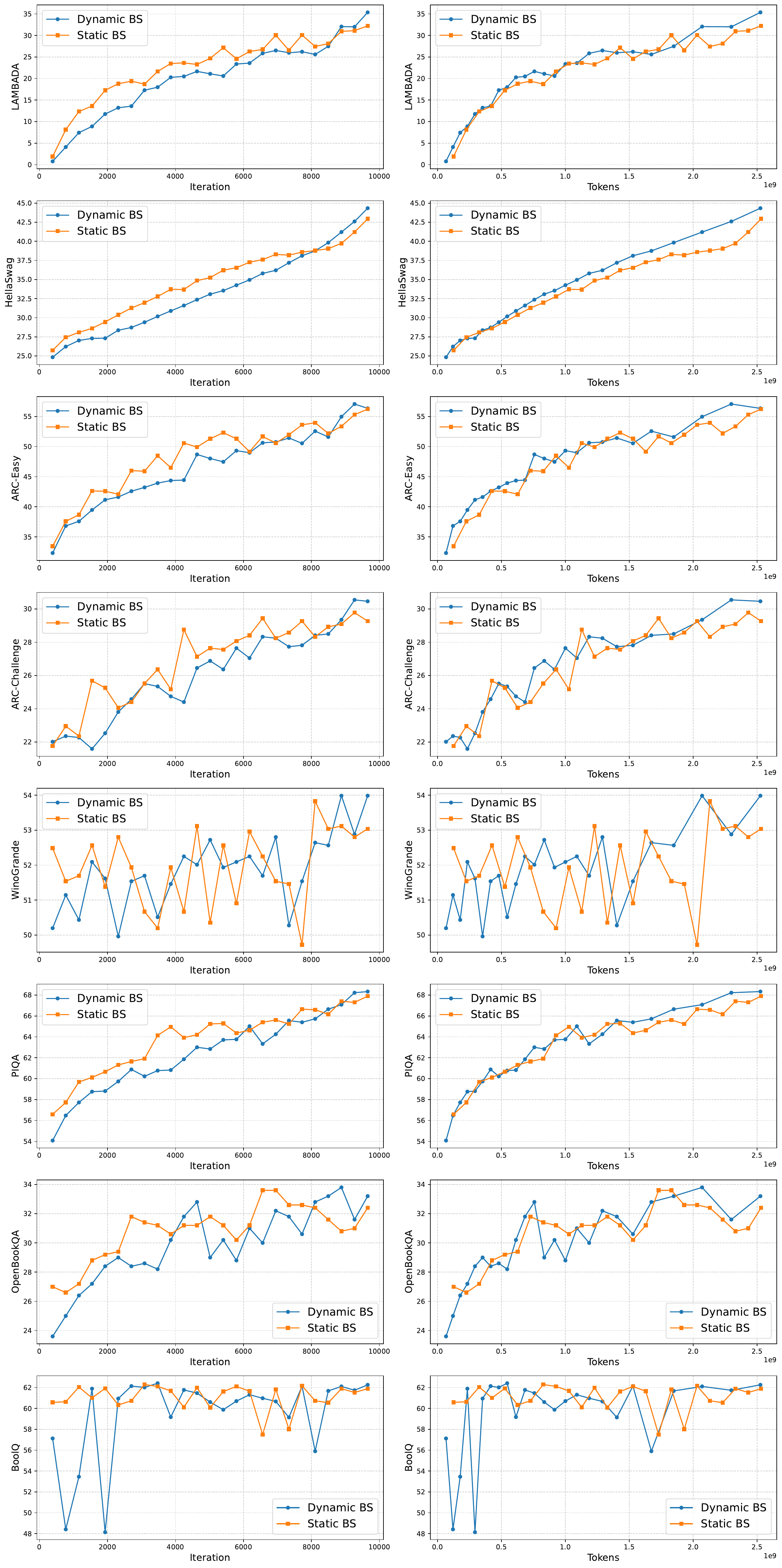}
    \caption{Evaluation of intermediate iterates for Llama3-8B with dynamic or static batch size.}
    \label{fig:sub eval}
\end{figure}

\subsection{\Cref{fig:nanogpt-muon}, \Cref{fig:dynamicBS}}
Llama3 model has 28 layer, hidden size 1024, intermediate size 3072, batch size 256, and sequence length 1024. \Cref{fig:nanogpt-adam} is analogous to \Cref{fig:nanogpt-muon} but uses AdamW optimizer. The losses are smoothed by time weighted EMA.

\begin{figure}[!htb]
\centering

\begin{tcolorbox}[
    colback=white,
    colframe=yellow!60!orange,
    boxrule=1.2pt,
    arc=3mm,
    left=0mm,
    right=0mm,
    top=1mm,
    bottom=0mm,
    width=\linewidth,
    title={\textbf{constant batch size, dynamic learning rate}},
    coltitle=black,
    colbacktitle=yellow!25,
    fonttitle=\normalsize,
    enhanced,
    attach boxed title to top center={
        yshift=-1.5mm
    },
    boxed title style={
        colback=yellow!25,
        colframe=yellow!60!orange,
        arc=2mm,
        boxrule=1pt
    }
]

\centering

\includegraphics[width=0.24\linewidth]{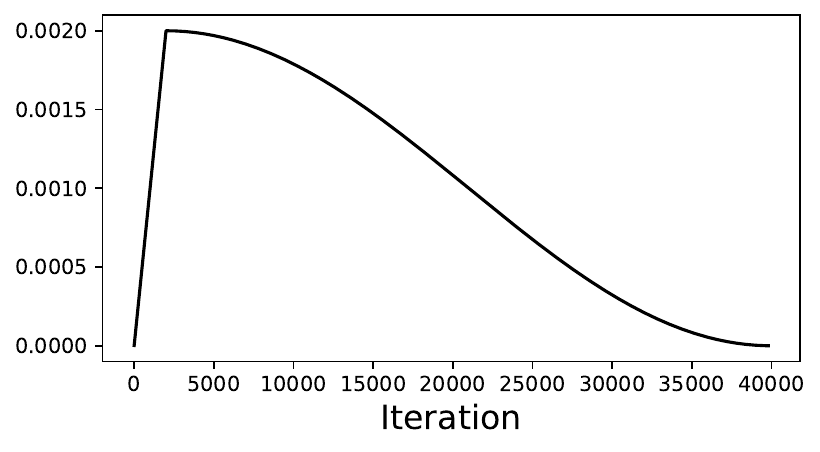}
\includegraphics[width=0.24\linewidth]{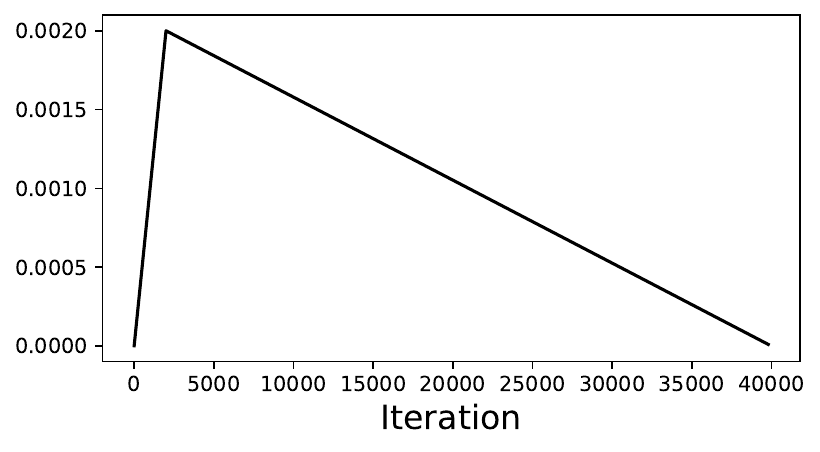}
\includegraphics[width=0.24\linewidth]{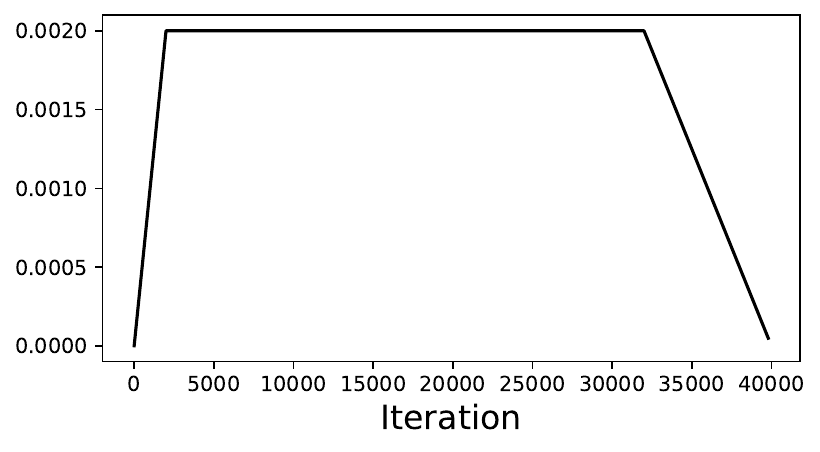}
\includegraphics[width=0.24\linewidth]{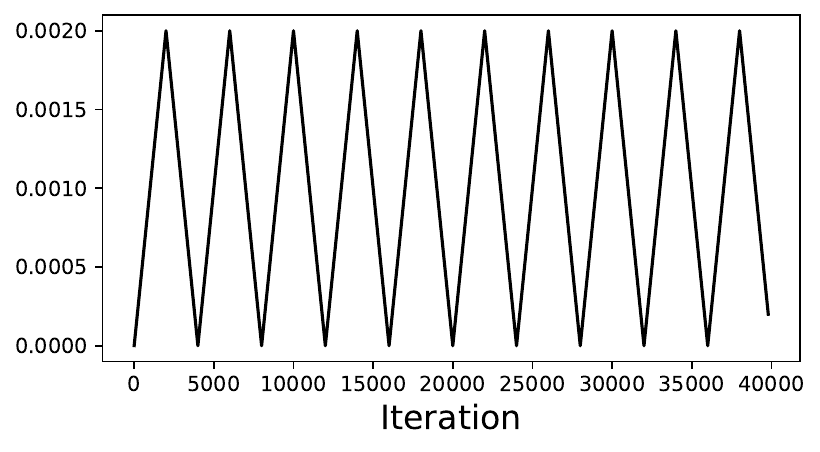}

\vspace{1mm}

\includegraphics[width=0.24\linewidth]{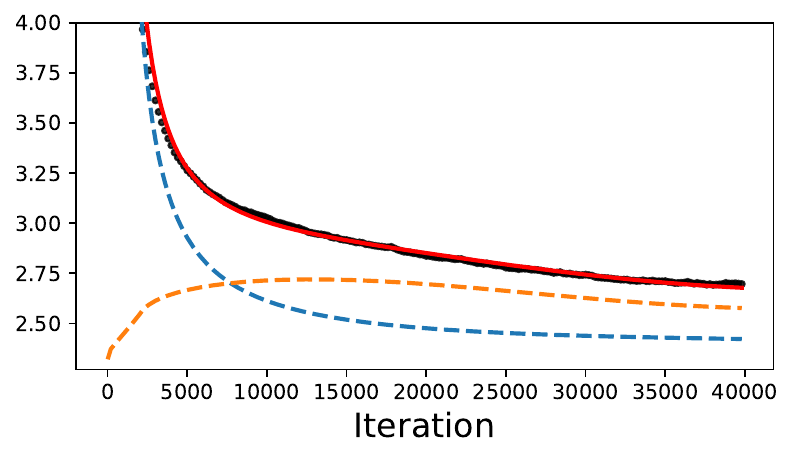}
\includegraphics[width=0.24\linewidth]{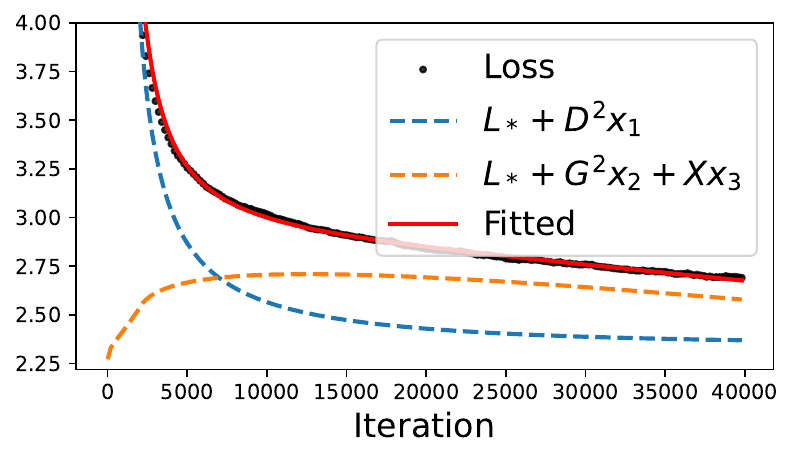}
\includegraphics[width=0.24\linewidth]{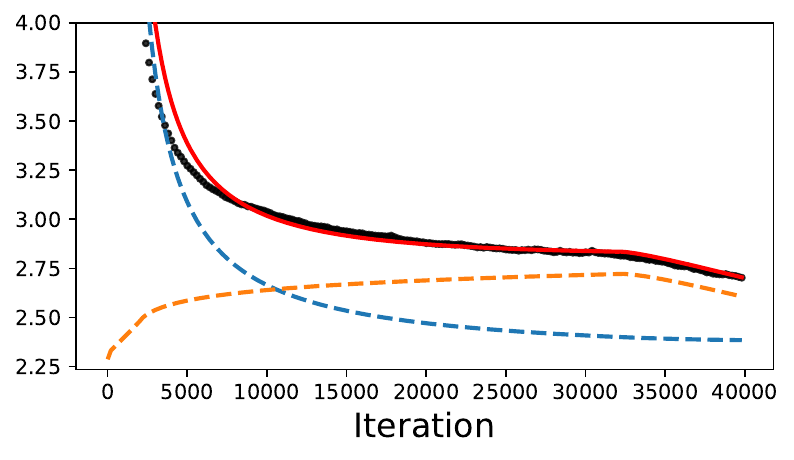}
\includegraphics[width=0.24\linewidth]{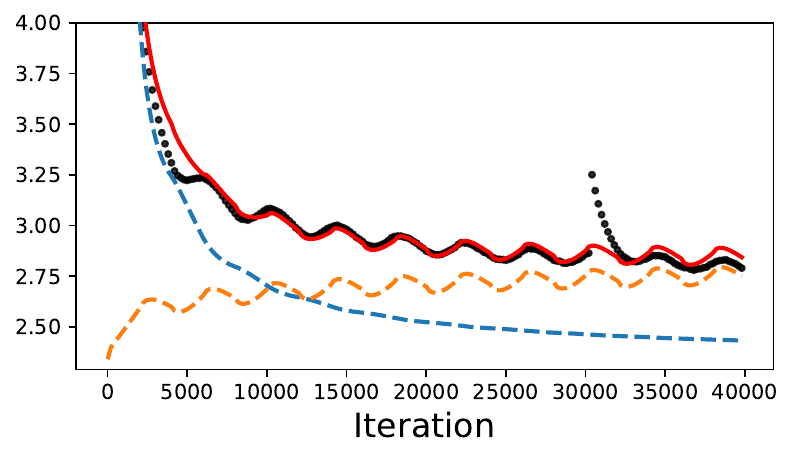}

\end{tcolorbox}

\begin{tcolorbox}[
    colback=white,
    colframe=blue!60,
    boxrule=1.2pt,
    arc=3mm,
    left=0mm,
    right=0mm,
    top=1mm,
    bottom=0mm,
    width=\linewidth,
    title={\textbf{constant learning rate, dynamic batch size}},
    coltitle=black,
    colbacktitle=blue!15,
    fonttitle=\normalsize,
    enhanced,
    attach boxed title to top center={
        yshift=-1.5mm
    },
    boxed title style={
        colback=blue!15,
        colframe=blue!60,
        arc=2mm,
        boxrule=1pt
    }
]

\centering

\includegraphics[width=0.24\linewidth]{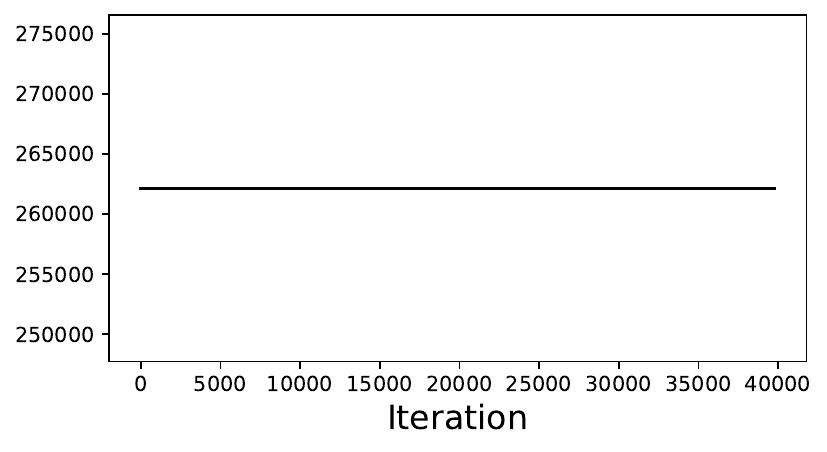}
\includegraphics[width=0.24\linewidth]{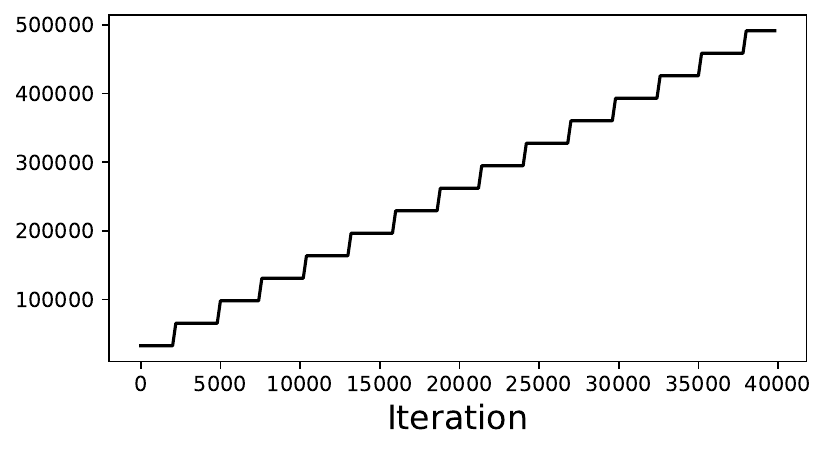}
\includegraphics[width=0.24\linewidth]{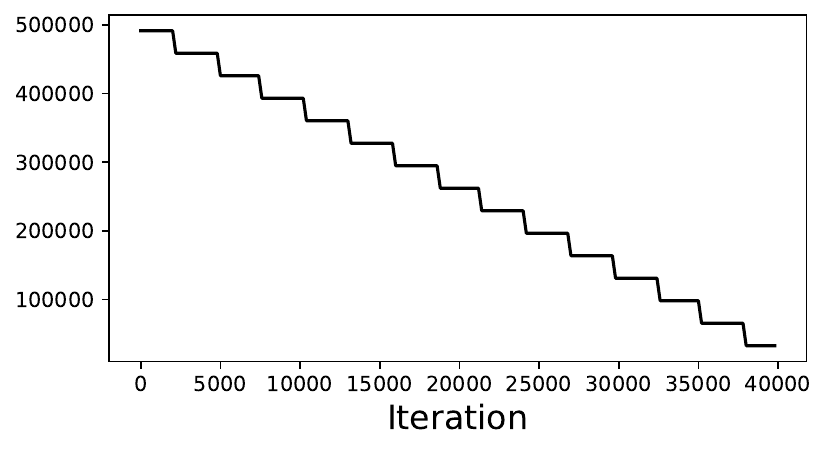}
\includegraphics[width=0.24\linewidth]{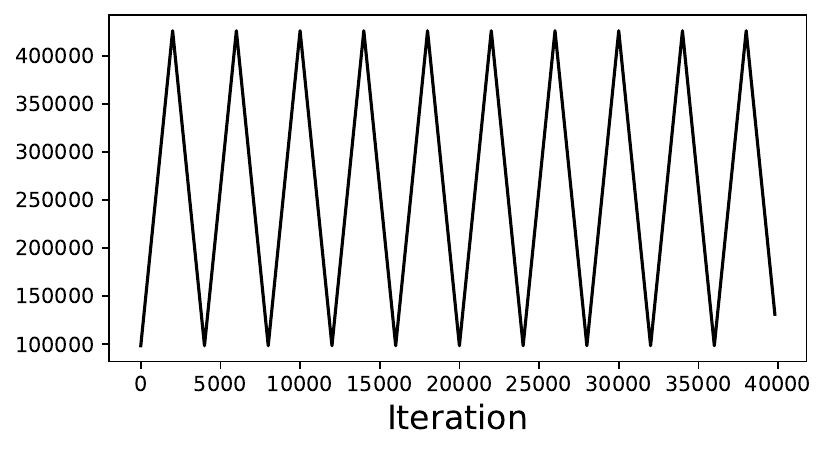}

\vspace{1mm}

\includegraphics[width=0.24\linewidth]{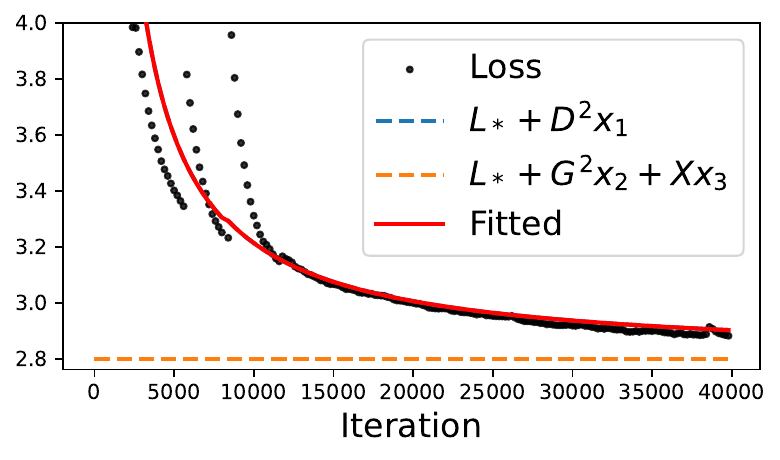}
\includegraphics[width=0.24\linewidth]{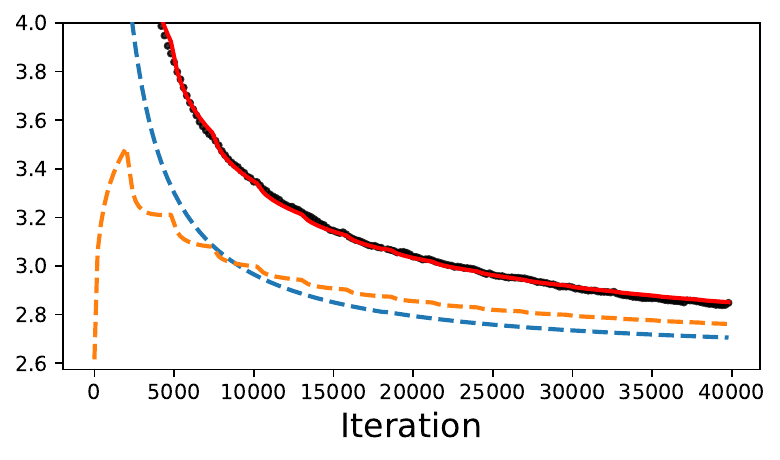}
\includegraphics[width=0.24\linewidth]{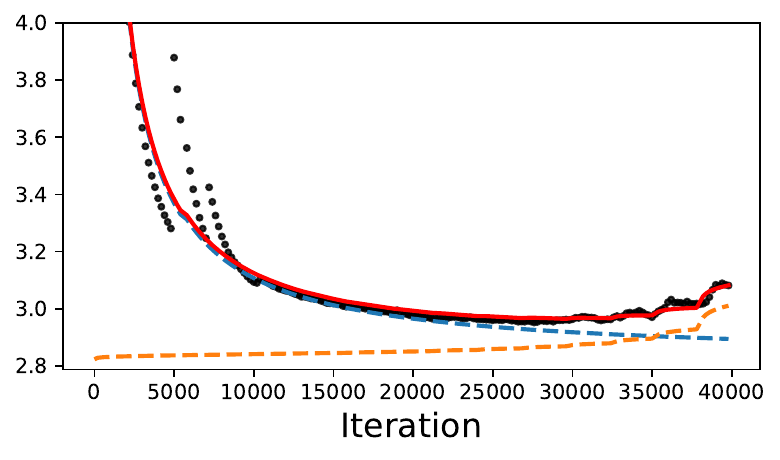}
\includegraphics[width=0.24\linewidth]{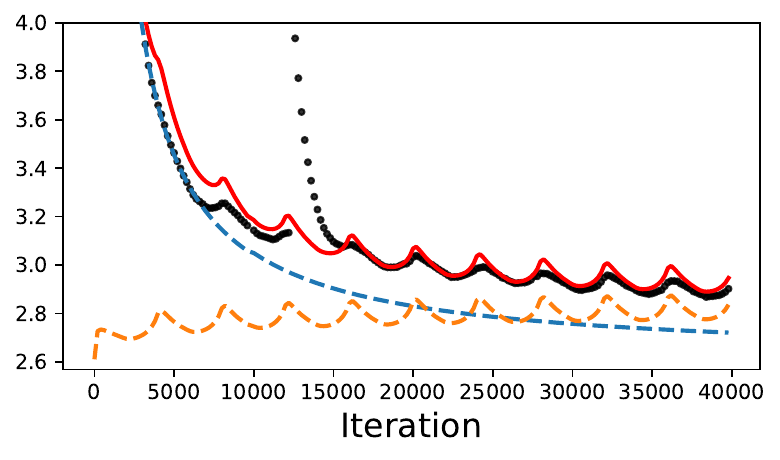}

\end{tcolorbox}
\caption{Sequence-to-sequence prediction by \Cref{thm:seq2seq} for Llama3-1B model on Fineweb-edu data, trained by AdamW optimizer. Top two rows: $B_t$ is fixed ($B=256*1024$) and $\eta_t$ is dynamic. Bottom two rows: $\eta_t$ is fixed ($\eta=0.002$) and $B_t$ is dynamic.
}
\label{fig:nanogpt-adam}
\end{figure}

\subsection{\Cref{fig:scaling}, \Cref{fig:collapse}}
We use WSD learning rate schedule with 10\% cooldown and Muon-NSGD optimizer. For \Cref{fig:scaling} (left column), the peak learning rate is 0.01; for other figures, the peak learning rate is scaled as $2/\sqrt{T}$. TPP is based on active parameters for MoE and total parameters for dense models.

For Llama3 models, we experiment with
\begin{itemize}
    \item total parameters 111M, 8 layers, 512 hidden size
    \item total parameters 274M, 8 layers, 1024 hidden size
    \item total parameters 426M, 10 layers, 1280 hidden size
    \item total parameters 507M, 10 layers, 1440 hidden size
    \item total parameters 631M, 12 layers, 1536 hidden size
    \item total parameters 898M, 14 layers, 1792 hidden size
    \item total parameters 1.236B, 16 layers, 2048 hidden size
    \item total parameters 2.163B, 20 layers, 2560 hidden size
    \item total parameters 3.569B, 28 layers, 2880 hidden size
    \item total parameters 6.585B, 48 layers, 3072 hidden size
\end{itemize}
The intermediate size is set to 4 times hidden size, batch size 256, and sequence
length 1024.

For Qwen3 MoE models, we experiment with
\begin{itemize}
\item active parameters 438M, 8 layers, 1280 hidden size
    \item active parameters 606M, 16 layers, 1536 hidden size
    \item active parameters 826M, 24 layers, 1792 hidden size
    \item active parameters 1.127B, 32 layers, 2048 hidden size
\end{itemize}
The intermediate size is set to 3 times hidden size, MoE intermediate size is 1/4 hidden size, 16 experts, batch size 512, and sequence length 4096.

\section{Ablation results}
\label{app:ablation}
\subsection{Optimizers}
Our main optimizer is Muon-NSGD \citep{boreiko2025towards}. We additionally experiment on AdamW optimizer \citep{loshchilovdecoupled} to confirm the applicability of \Cref{thm:seq2seq} in \Cref{fig:nanogpt-adam}. We also trained VLM in \Cref{sec:vlm} with AdamW.

\subsection{Peak learning rate}
We have ablated $O(1/\sqrt{T})$ learning rate in \Cref{sec:scaling}, which is empirically the optimal learning rate. Furthermore, we show in \Cref{fig:vary peak lr} that optimal batch size \textit{\textbf{always outperforms}} static batch size for any learning rate, larger or smaller than the optimal learning rate (which is 0.01 in this case). This holds because our batch size derived in \Cref{thm:optimal B} is \textit{\textbf{decoupled}} from the peak learning rate. 

\begin{figure}[!htb]
    \centering
  \begin{minipage}[c]{0.72\textwidth}
    \includegraphics[width=0.32\linewidth]{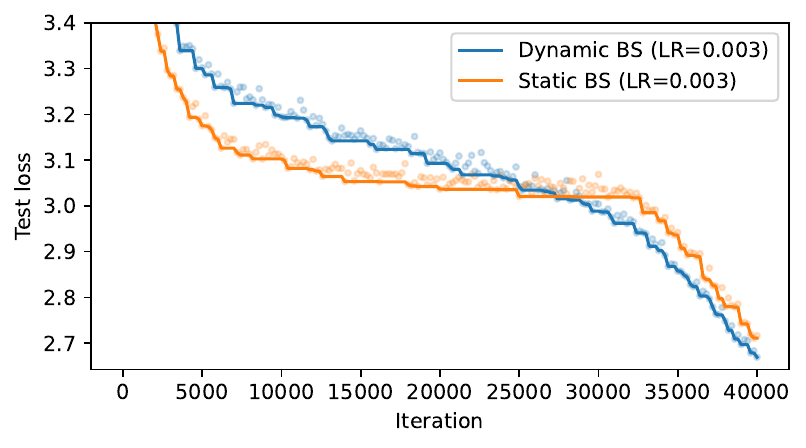}
    \includegraphics[width=0.32\linewidth]{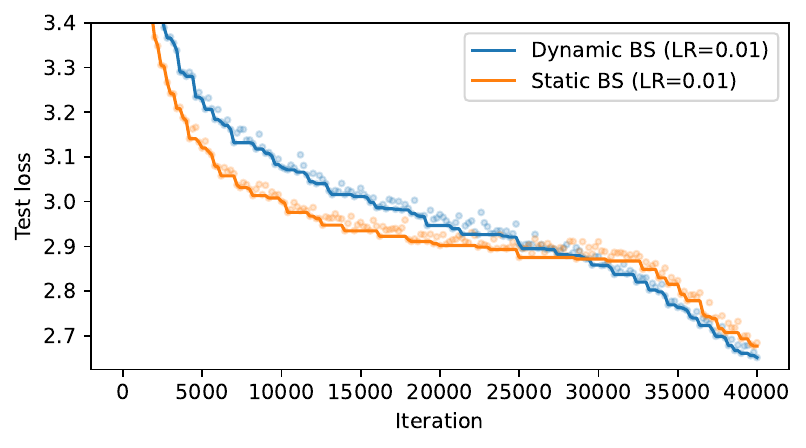}
    \includegraphics[width=0.32\linewidth]{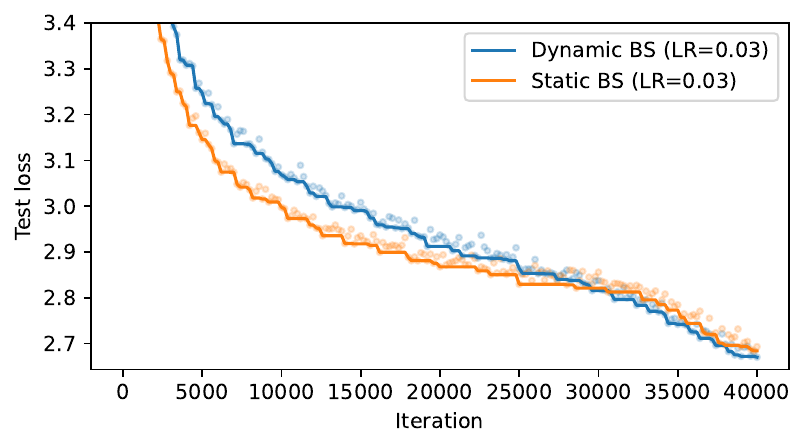}
  \\
    \includegraphics[width=0.32\linewidth]{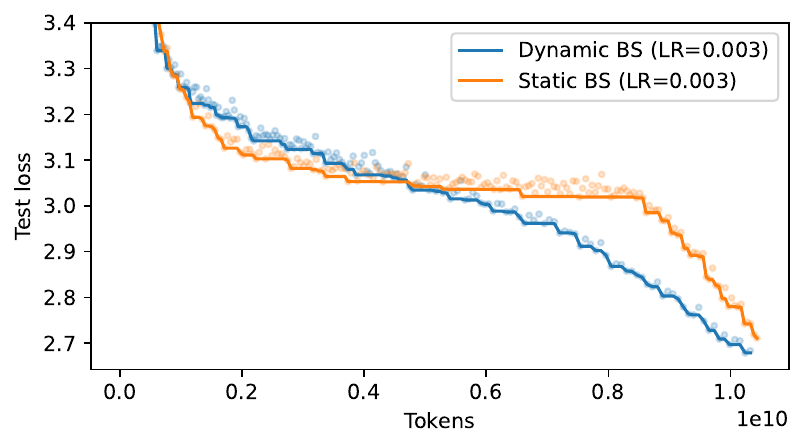}
    \includegraphics[width=0.32\linewidth]{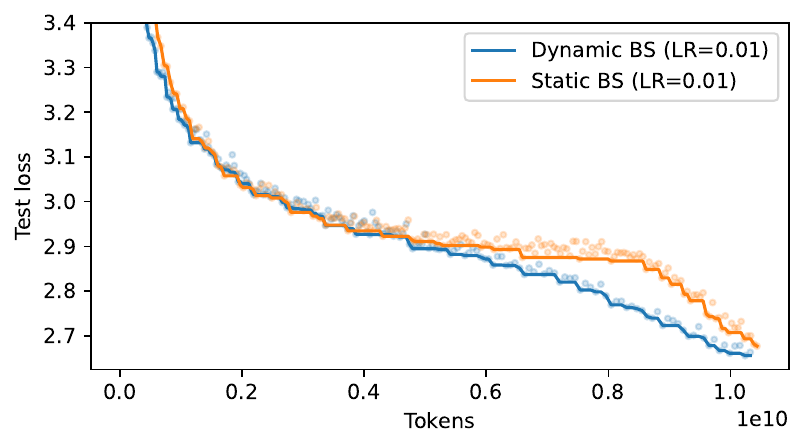}
    \includegraphics[width=0.32\linewidth]{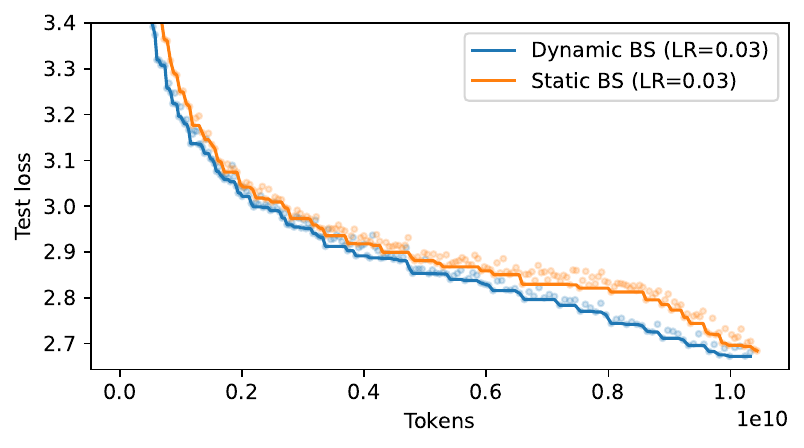}
    \end{minipage}
  \hspace{-0.4cm}
  \begin{minipage}[c]{0.28\textwidth}
    \includegraphics[width=\linewidth]{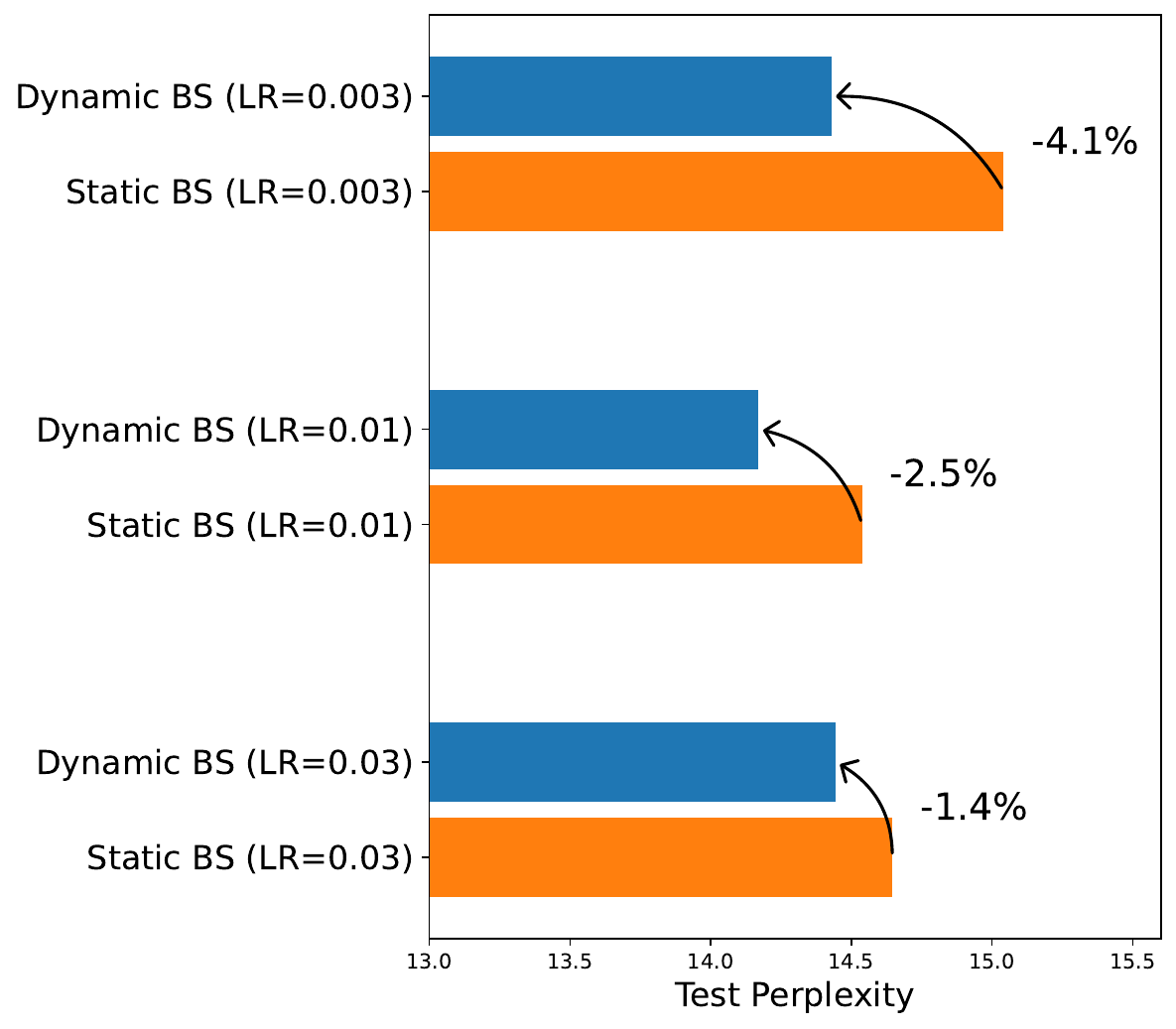}
    \end{minipage}
\vspace{-0.2cm}
    \caption{Impact of batch size schedules for Llama3 1B model on Fineweb-edu data. Left to right: peak learning rate=0.003,0.01 (optimal), 0.03. 
}
    \label{fig:vary peak lr}
\end{figure}

\subsection{Training precision}
By default, we use mixed-precision training. We also observe strong performance of optimal batch size schedule when using full-precision training. However, with half-precision training, the advantage of our batch size vanishes in comparison to the static batch size.
\begin{figure}[H]
    \centering
  \begin{minipage}[c]{0.72\textwidth}
    \includegraphics[width=0.32\linewidth]{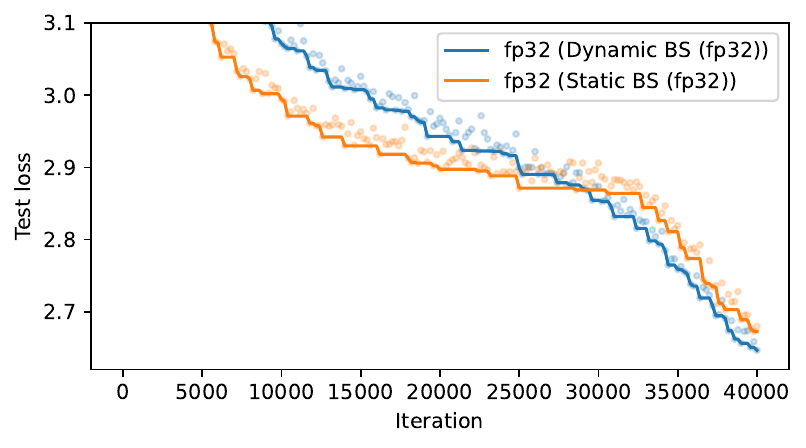}
    \includegraphics[width=0.32\linewidth]{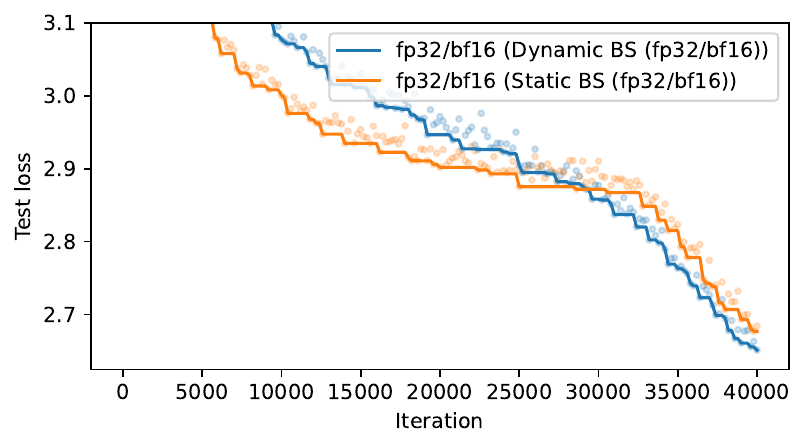}
    \includegraphics[width=0.32\linewidth]{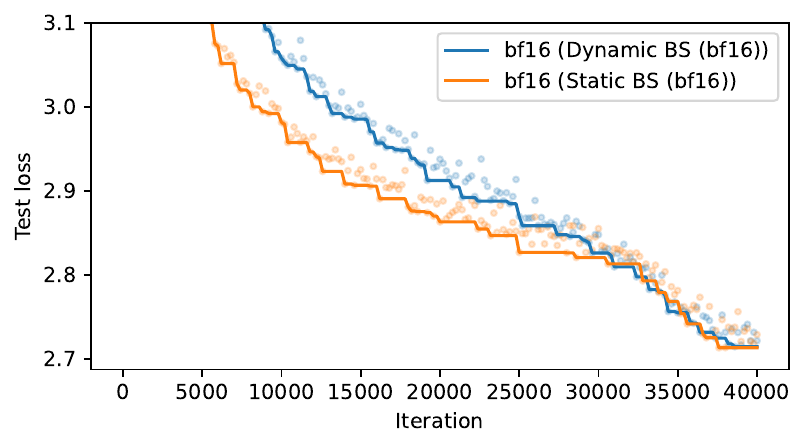}
  \\
    \includegraphics[width=0.32\linewidth]{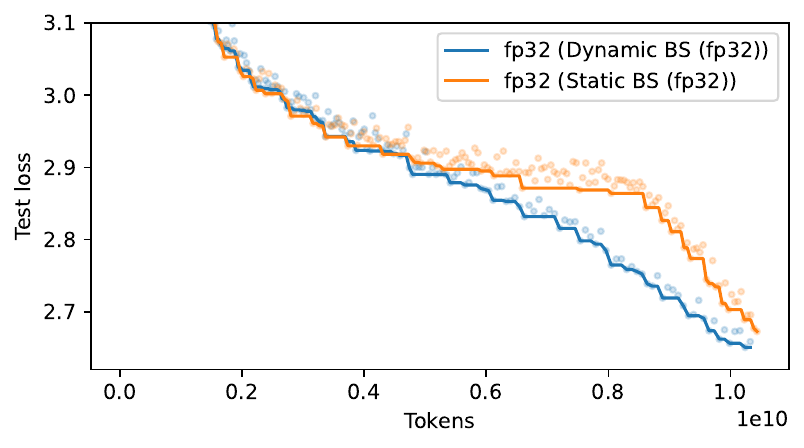}
    \includegraphics[width=0.32\linewidth]{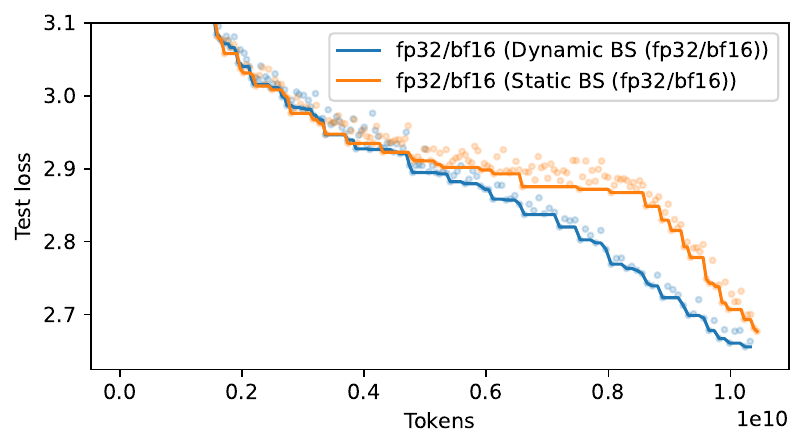}
    \includegraphics[width=0.32\linewidth]{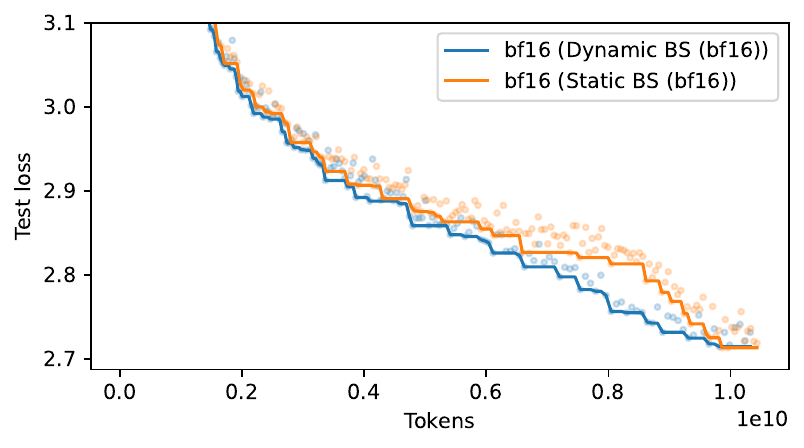}
    \end{minipage}
  \hspace{-0.4cm}
  \begin{minipage}[c]{0.28\textwidth}
    \includegraphics[width=\linewidth]{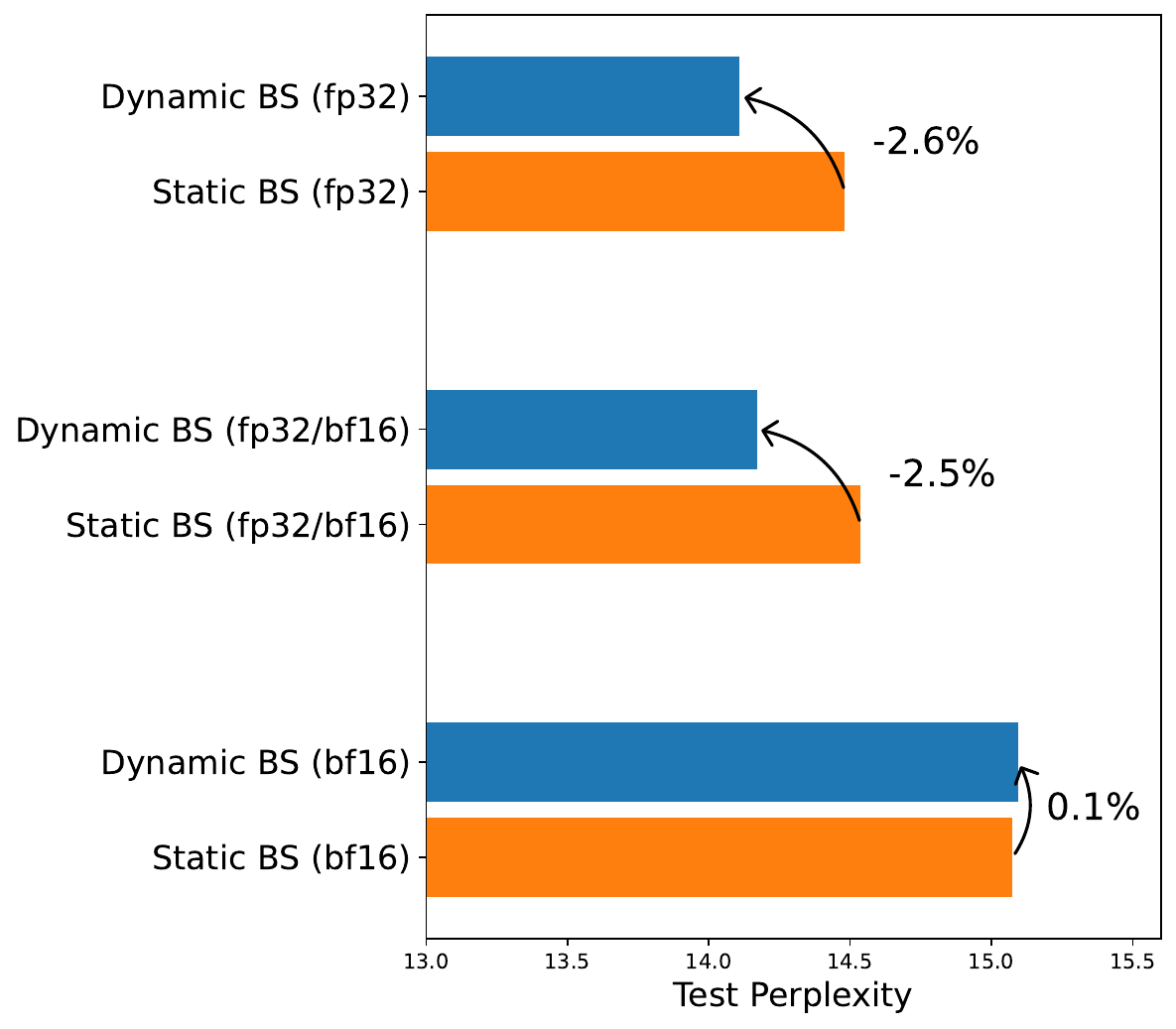}
    \end{minipage}
\vspace{-0.2cm}
    \caption{Impact of batch size schedules for Llama3 1B model on Fineweb-edu data. Left to right: full-precision (fp32), mixed precision (fp32/bf16), and half-precision (bf16).
 }
    \label{fig:vary precision}
\end{figure}

\subsection{Base batch size}
We show in \Cref{fig:vary bs} and \Cref{fig:vary bs2} that optimal batch size \textit{\textbf{always outperforms}} static batch size for any base batch size, under fixed compute and fixed iterations. Note that the advantage vanishes as $B$ approaches infinity, because the advantage is inversely proportional to base batch size by \Cref{cor:variance reduction}. Here $B$=1X is 128k tokens.
\begin{figure}[!htb]
    \centering
  \hspace{-0.3cm}
  \begin{minipage}[c]{0.74\textwidth}
    \includegraphics[width=0.24\linewidth]{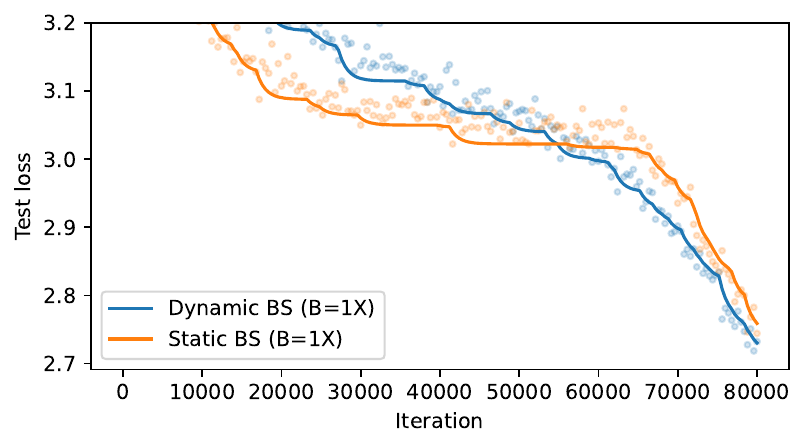}
    \includegraphics[width=0.24\linewidth]{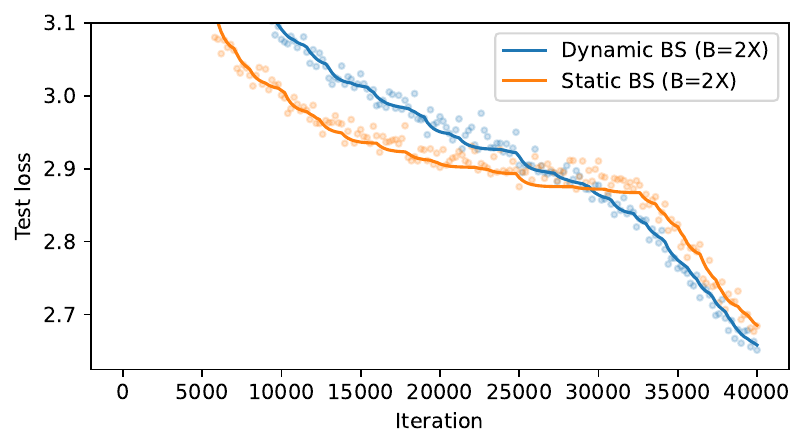}
    \includegraphics[width=0.24\linewidth]{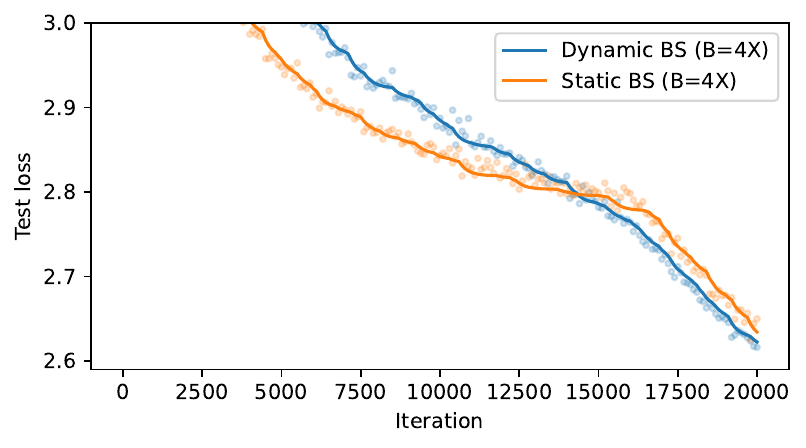}
    \includegraphics[width=0.24\linewidth]{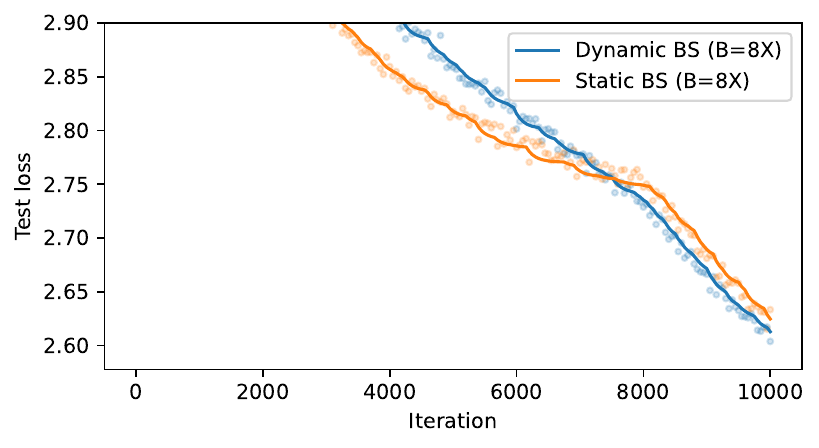}
  \\
     \includegraphics[width=0.24\linewidth]{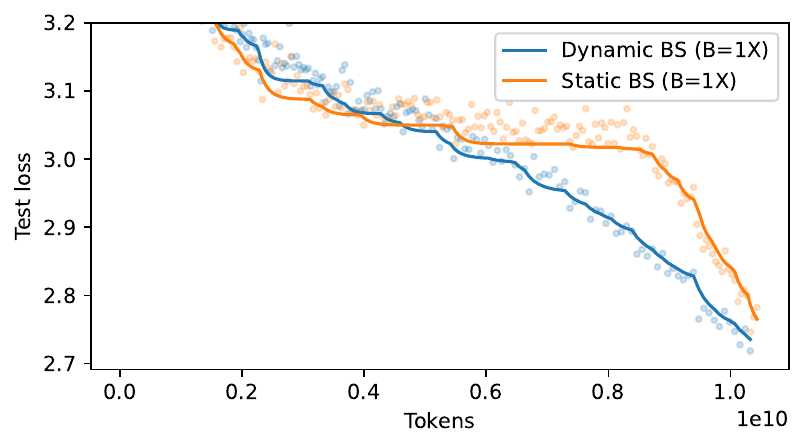}
     \includegraphics[width=0.24\linewidth]{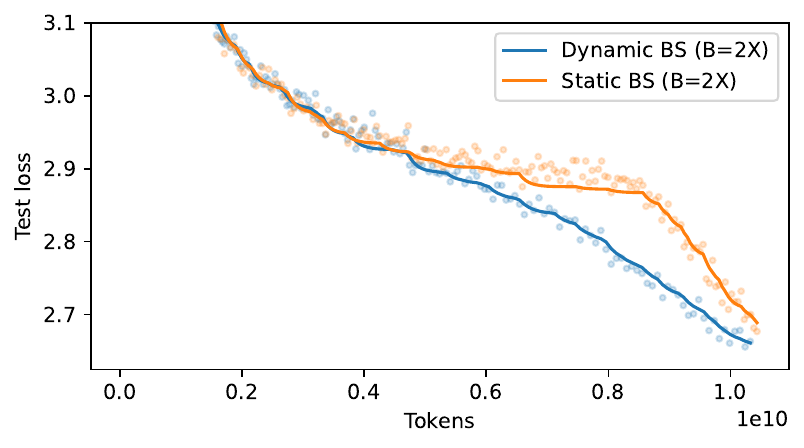}
     \includegraphics[width=0.24\linewidth]{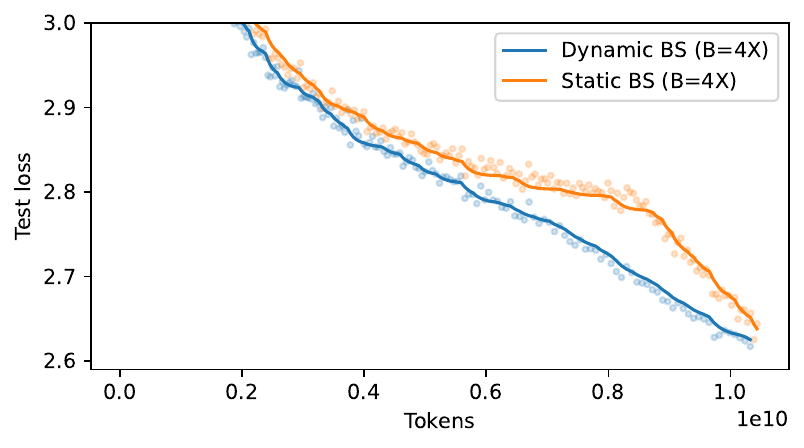}
     \includegraphics[width=0.24\linewidth]{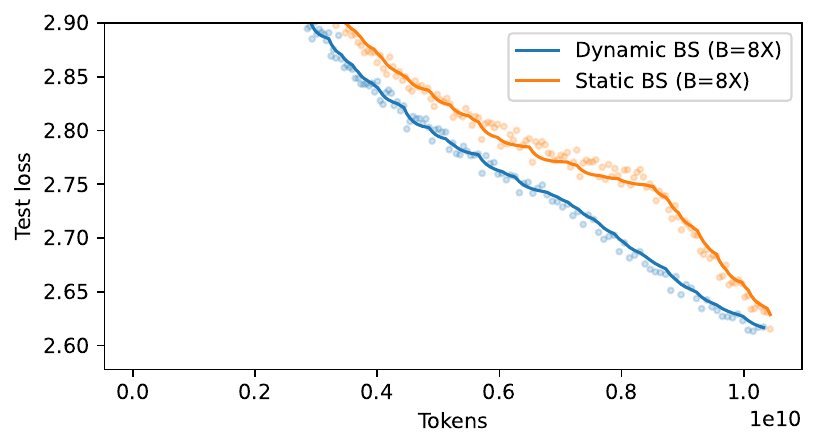}
   \end{minipage}
  \hspace{-0.2cm}
  \begin{minipage}[c]{0.26\textwidth}
    \includegraphics[width=\linewidth]{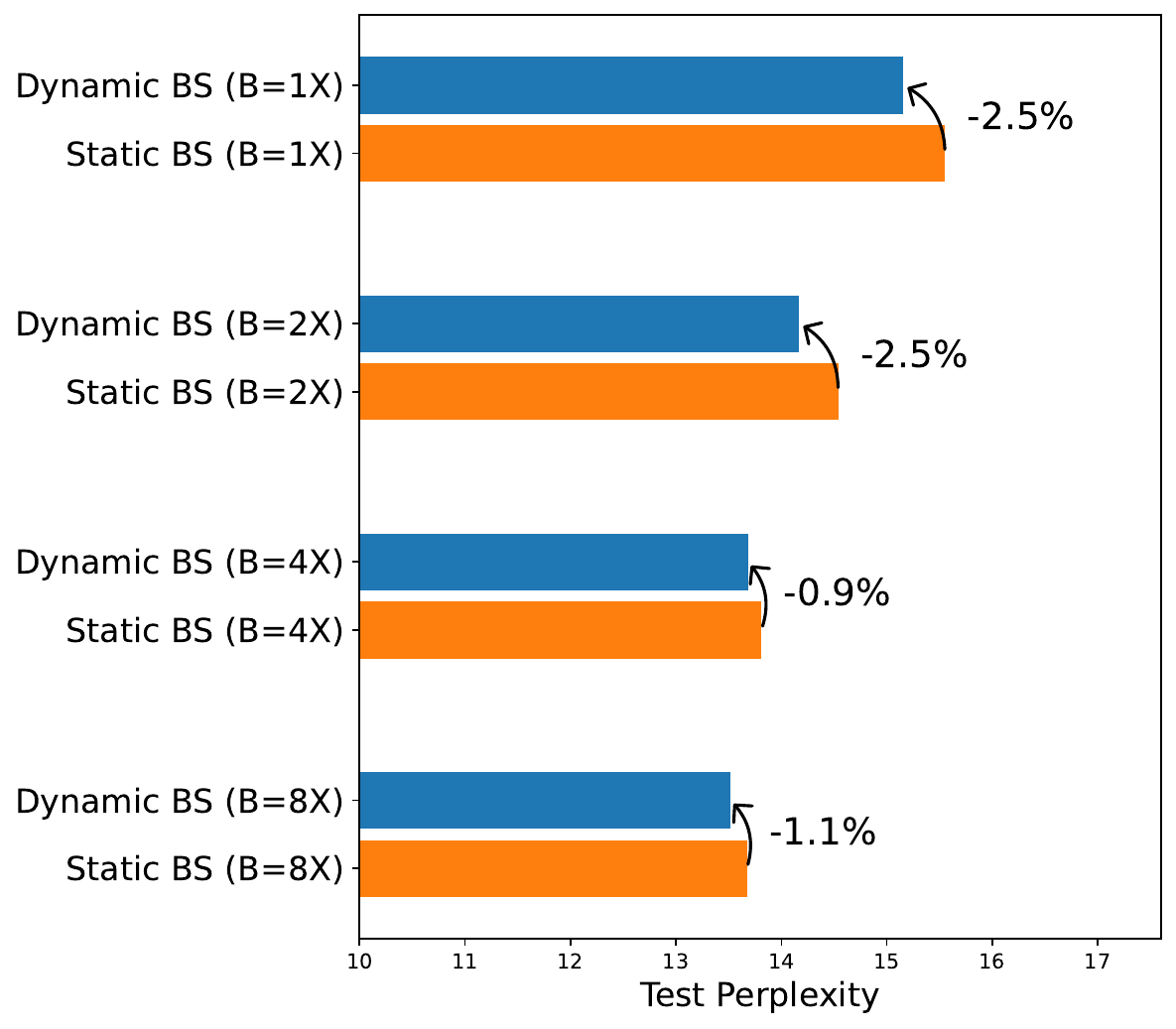}
    \end{minipage}
\vspace{-0.2cm}
    \caption{Impact of batch size schedules for Llama3 1B model on Fineweb-edu data, under fixed compute (larger batch uses fewer iterations).
    }
    \label{fig:vary bs}
\end{figure}

\begin{figure}[!htb]
    \centering
  \hspace{-0.3cm}
  \begin{minipage}[c]{0.74\textwidth}
    \includegraphics[width=0.24\linewidth]{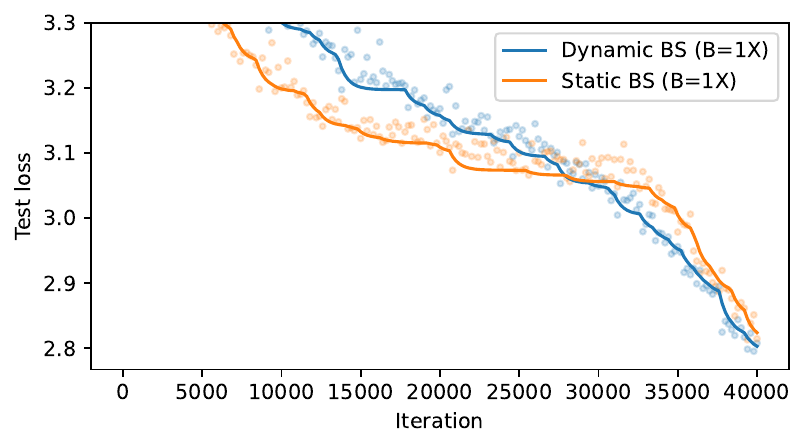}
    \includegraphics[width=0.24\linewidth]{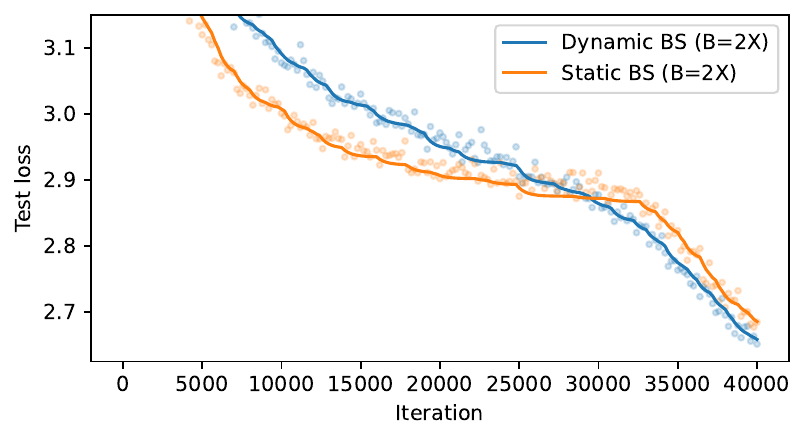}
    \includegraphics[width=0.24\linewidth]{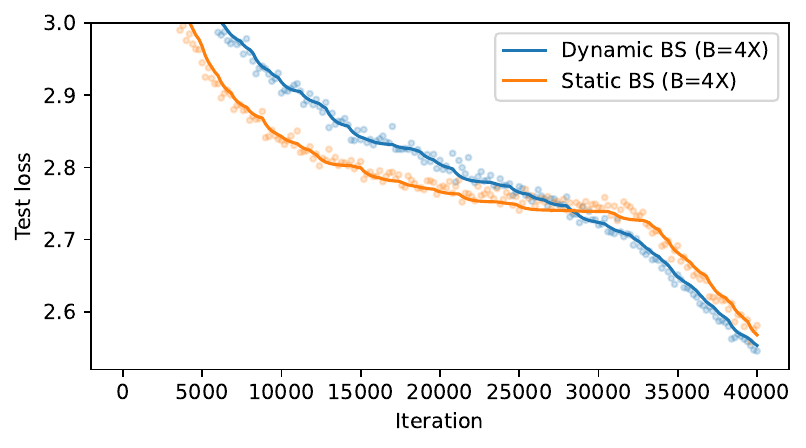}
    \includegraphics[width=0.24\linewidth]{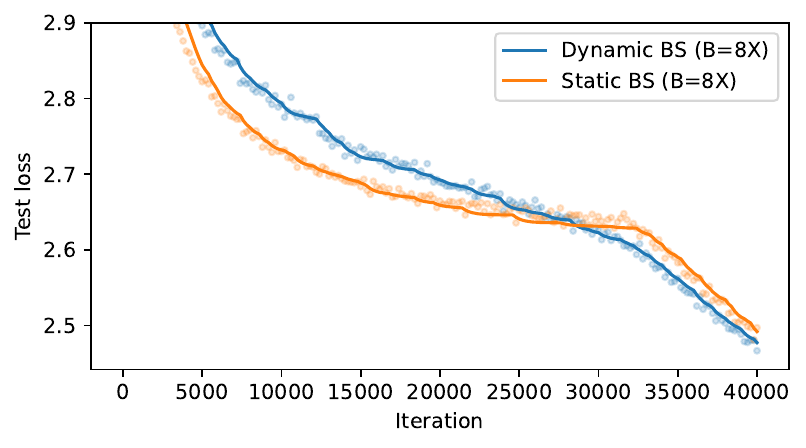}
  \\
     \includegraphics[width=0.24\linewidth]{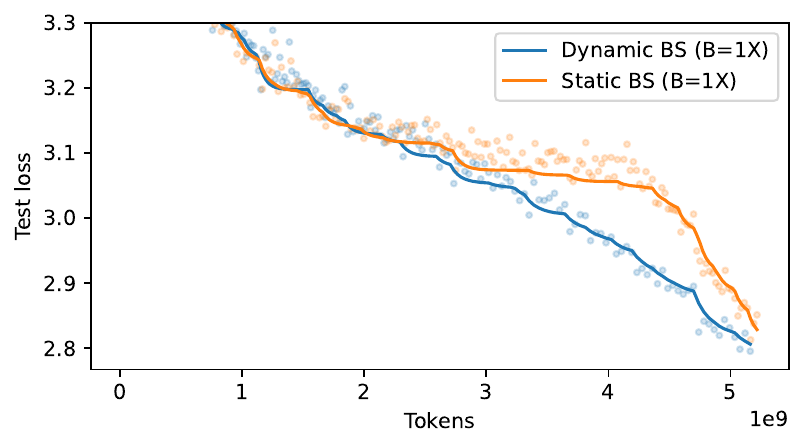}
     \includegraphics[width=0.24\linewidth]{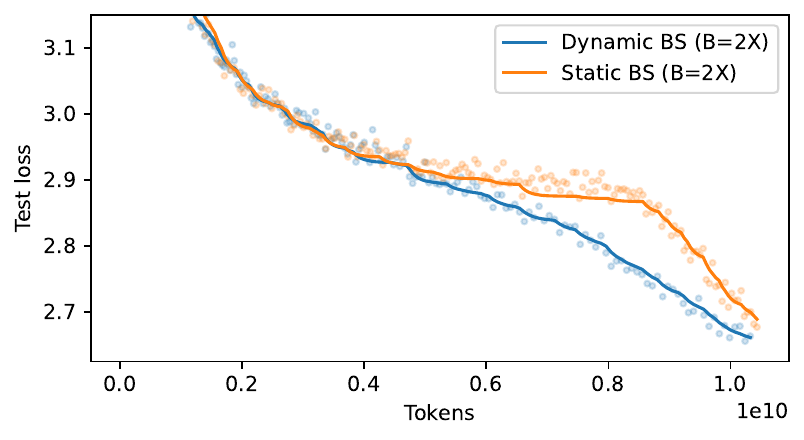}
     \includegraphics[width=0.24\linewidth]{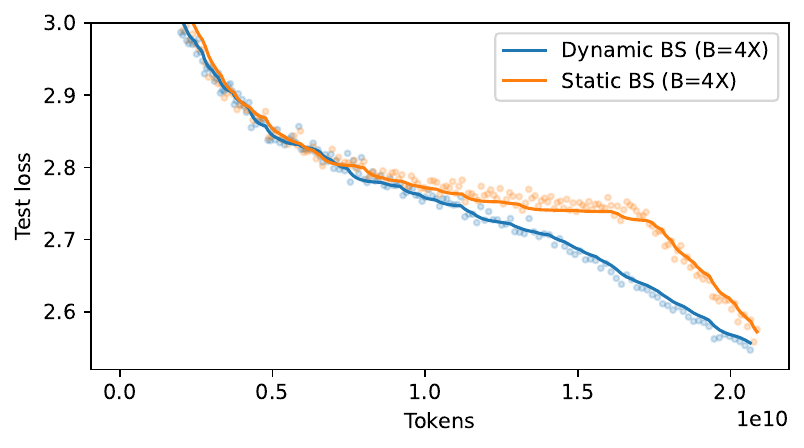}
     \includegraphics[width=0.24\linewidth]{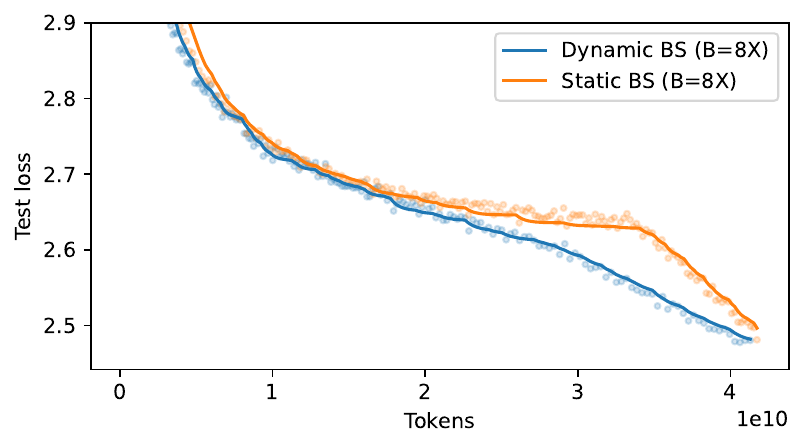}
   \end{minipage}
  \hspace{-0.2cm}
  \begin{minipage}[c]{0.26\textwidth}
    \includegraphics[width=\linewidth]{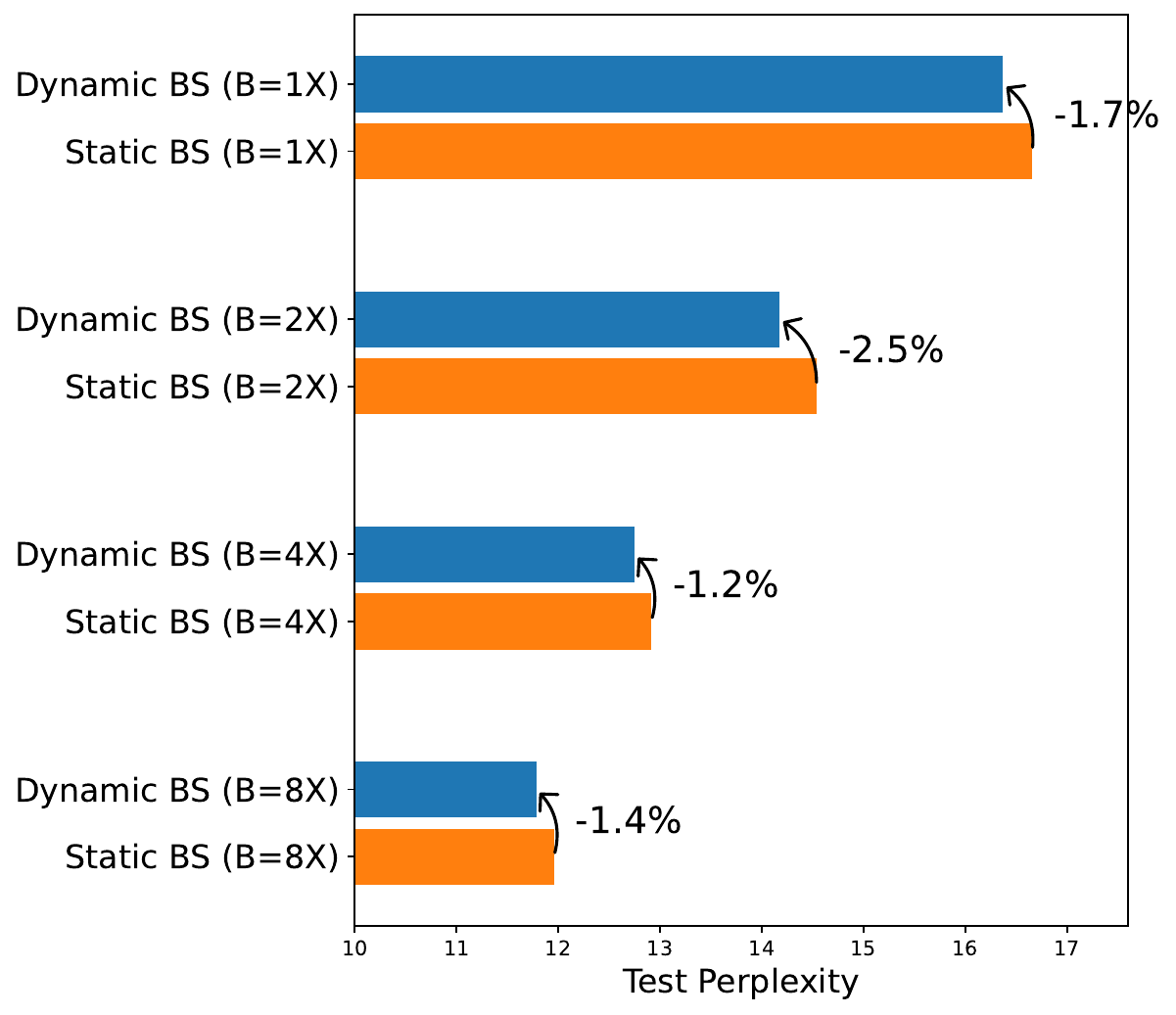}
    \end{minipage}
\vspace{-0.2cm}
    \caption{Impact of batch size schedules for Llama3 1B model on Fineweb-edu data, under fixed iterations of 40000. 
    }
    \label{fig:vary bs2}
\end{figure}

\subsection{The $\eta/B$ ratio does not determine training dynamics}
\label{sec:eta_B_ratio}

A natural hypothesis, suggested by the linear scaling rule~\citep{goyal2017accurate,smith2018dontdecay} and the analysis in~\citet{mccandlish2018empirical}, is that training dynamics are governed by the ratio $\eta/B$ rather than by $\eta$ and $B$ individually. If true, one could freely trade learning rate schedules for batch size schedules as long as $\eta(t)/B(t)$ remains unchanged. We design a controlled experiment to test this hypothesis directly.

\paragraph{Setup.}
We train Qwen3-0.6B from scratch on FineWeb-EDU using AdamW with batch size 256 and sequence length 1024. Each run trains for 10{,}000 steps, consuming $2.6$B tokens. We sweep over three learning rate schedules (WSD, cosine, linear) and three base learning rates ($3\times10^{-4}$, $10^{-3}$, $3\times10^{-3}$), with $\eta_{\min} = 0.1\,\eta_{\max}$.

\paragraph{Design.}
For each (schedule, $\eta_{\max}$) pair, we run two configurations that produce the \emph{same} $\eta/B$ trajectory at every step:
\begin{itemize}
    \item \textbf{Type~1 (LR schedule):} $\eta(t)$ follows the schedule, $B=G$ is constant. The ratio is $\eta(t)/G$.
    \item \textbf{Type~2 (BS schedule):} $\eta = \eta_{\text{const}}$ is fixed, $B(t) \propto 1/\eta_{\text{ref}}(t)$ so that $\eta_{\text{const}}/B(t) = \eta_{\text{ref}}(t)/G$. Here $\eta_{\text{const}} = T / \sum_{t=0}^{T-1} 1/\eta_{\text{ref}}(t)$ is chosen to match the total token budget exactly.
\end{itemize}
Both configurations consume identical total tokens by construction. If $\eta/B$ is the governing quantity, Type~1 and Type~2 should yield the same loss.

\paragraph{Results.}
Figure~\ref{fig:eta_B_ratio_step} and Figure~\ref{fig:eta_B_ratio_tokens} show the validation loss curves versus iteration and total tokens, respectively. Across all 9 configurations, the LR-schedule run (Type~1, blue) consistently outperforms the BS-schedule run (Type~2, orange), often by a significant margin. Table~\ref{tab:eta_B_ratio} summarizes the final validation loss: the gap ranges from $+0.017$ to $+0.74$ in favor of Type~1. The effect is largest at small learning rates and for cosine/linear schedules, where the BS-schedule run must use a very small constant $\eta_{\text{const}}$ to satisfy the token budget constraint.

These results demonstrate that $\eta/B$ alone does not determine training dynamics. The learning rate schedule provides benefits beyond its effect on the $\eta/B$ ratio---likely through its interaction with the optimizer state and loss landscape curvature.

\begin{figure}[!htb]
    \centering
    \includegraphics[width=\linewidth]{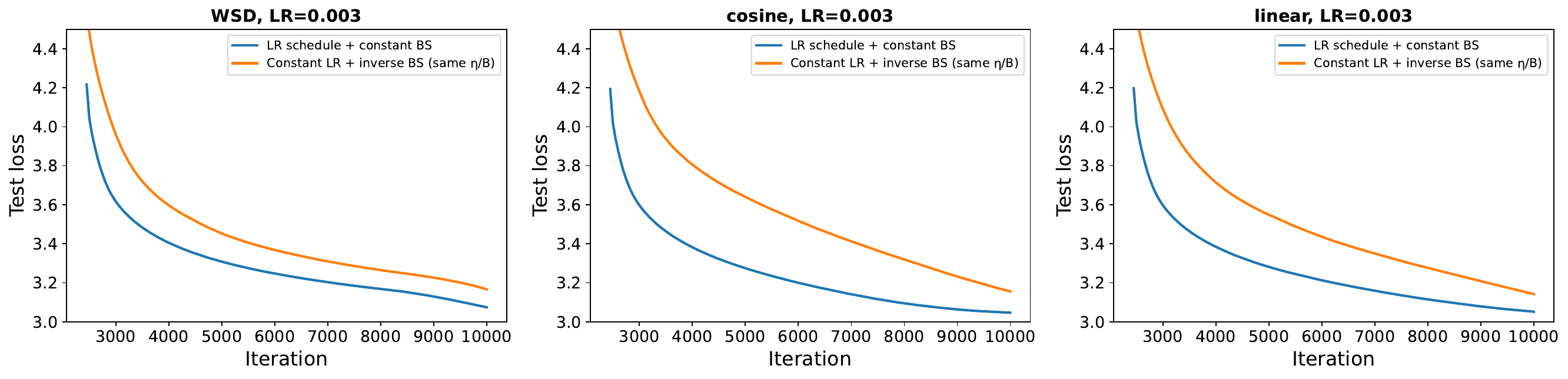}
    \\
    \includegraphics[width=\linewidth]{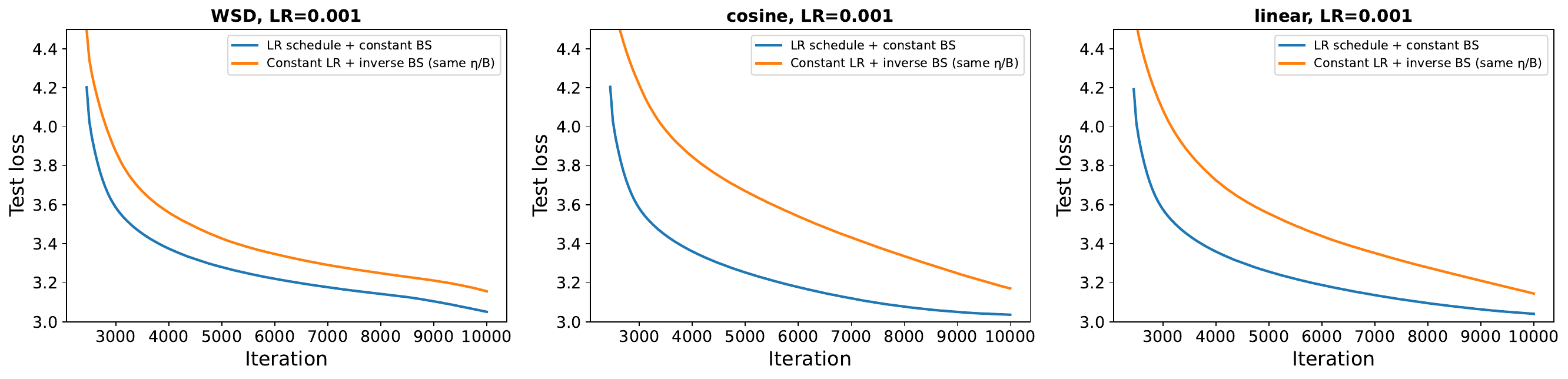}
    \\
    \includegraphics[width=\linewidth]{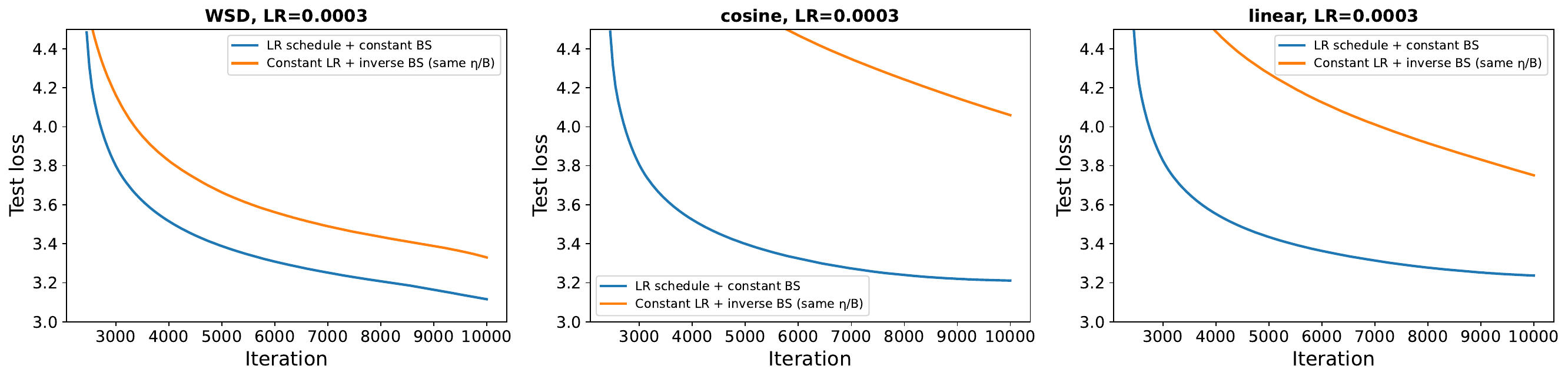}
    \vspace{-0.2cm}
    \caption{Validation loss vs.\ iteration for Type~1 (LR schedule + constant BS, blue) and Type~2 (constant LR + inverse BS, orange), which share the same $\eta/B$ trajectory. Columns: WSD, cosine, and linear schedules. Rows top to bottom: $\eta_{\max} = 3\times10^{-3}$, $10^{-3}$, $3\times10^{-4}$. The LR-schedule run consistently achieves lower loss despite identical $\eta/B$ at every step. Curves smoothed with a 50-iteration running average.}
    \label{fig:eta_B_ratio_step}
\end{figure}

\begin{figure}[!htb]
    \centering
    \includegraphics[width=\linewidth]{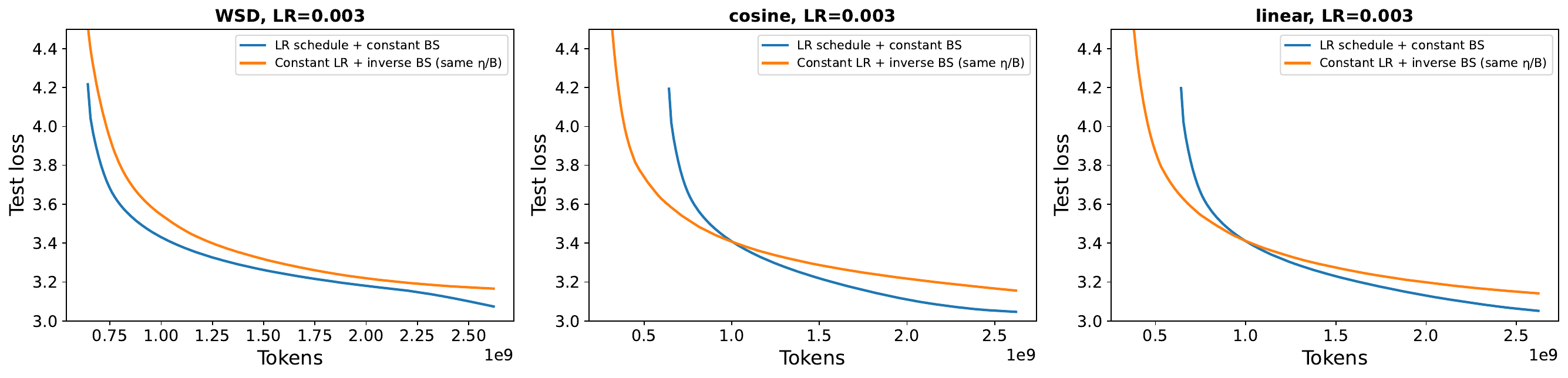}
    \\
    \includegraphics[width=\linewidth]{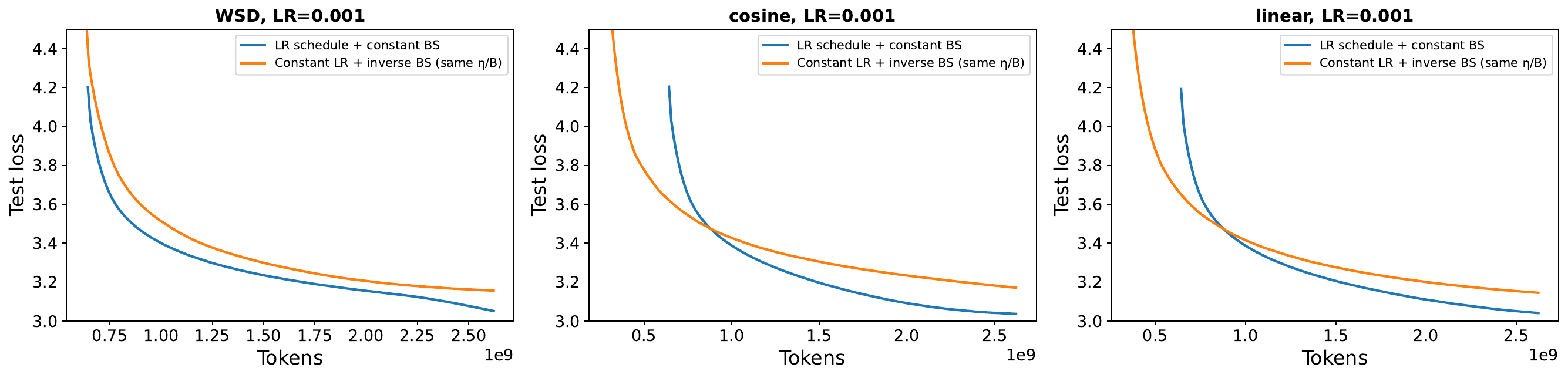}
    \\
    \includegraphics[width=\linewidth]{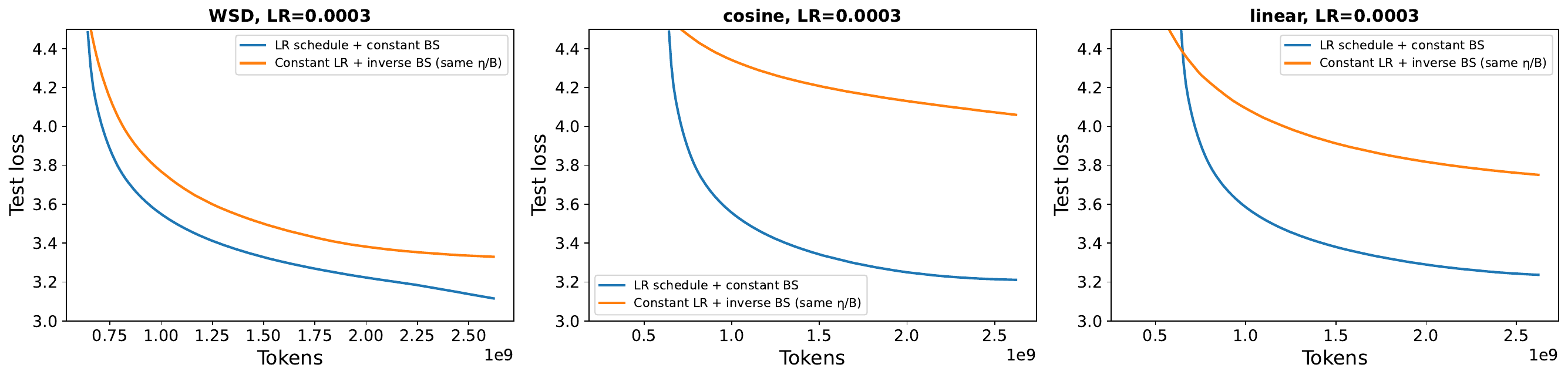}
    \vspace{-0.2cm}
    \caption{Same comparison as Figure~\ref{fig:eta_B_ratio_step}, plotted against total tokens consumed. Since both types consume exactly the same number of tokens, the x-axes are aligned. The LR-schedule advantage persists when controlling for compute budget.}
    \label{fig:eta_B_ratio_tokens}
\end{figure}

\begin{table}[!htb]
    \centering
    \caption{Final validation loss for Type~1 (LR schedule) vs.\ Type~2 (BS schedule) across all configurations. $\Delta$ is Type~2 minus Type~1; positive values indicate the LR schedule wins.}
    \label{tab:eta_B_ratio}
    \vspace{0.2cm}
    \begin{tabular}{llccc}
        \toprule
        $\eta_{\max}$ & Schedule & Type~1 (LR sched.) & Type~2 (BS sched.) & $\Delta$ \\
        \midrule
        $3\times10^{-3}$ & WSD    & 3.015 & 3.033 & +0.019 \\
        $3\times10^{-3}$ & Cosine & 3.050 & 3.082 & +0.031 \\
        $3\times10^{-3}$ & Linear & 3.041 & 3.059 & +0.017 \\
        \midrule
        $10^{-3}$        & WSD    & 2.995 & 3.035 & +0.039 \\
        $10^{-3}$        & Cosine & 3.042 & 3.093 & +0.051 \\
        $10^{-3}$        & Linear & 3.035 & 3.063 & +0.028 \\
        \midrule
        $3\times10^{-4}$ & WSD    & 3.077 & 3.212 & +0.136 \\
        $3\times10^{-4}$ & Cosine & 3.221 & 3.964 & +0.743 \\
        $3\times10^{-4}$ & Linear & 3.239 & 3.655 & +0.416 \\
        \bottomrule
    \end{tabular}
\end{table}

\clearpage
\section{Further experiment details}
\label{app:posttrain}

\Cref{tab:further_hparams} summarizes the hyperparameters used in the post-training experiments of \Cref{sec:posttrain}.

\begin{table}[!htb]
\centering
\caption{Hyperparameters for the VLM and math fine-tuning experiments in Section \ref{sec:posttrain}. For the LR sweep, we train with constant batch size and select the best LR by validation loss and downstream accuracy, then fix it for the dynamic-vs.-static comparison.}
\label{tab:further_hparams}
\small
\begin{tabular}{lcc}
\toprule
\textbf{Hyperparameter} & \textbf{VLM} & \textbf{Math} \\
\midrule
Model                    & SmoLM2+Siglip2                    & Qwen3-0.6B \\
Dataset                  & Cauldron                    & OrcaMath \\
Optimizer                & Muon-NSGD                    & AdamW \\
LR schedule              & WSD                    & WSD \\
Peak LR                  & 2e-3                    & 1e-5 \\
Min LR                   & 0                    & $0$ \\
Weight decay             & 0.01                    & 0.01 \\
Context length           & 4096                    & 512 \\
Base batch size         & 64                    & 32 \\
Total steps              & 40000                    & 500 \\
\bottomrule
\end{tabular}
\end{table}

\Cref{tab:vlm_results} reports the absolute scores for VLM benchmark tasks under dynamic and static batch size schedules.

\begin{table}[!htb]
\centering
\caption{Benchmark scores for dynamic and static batch sizes on VLM evaluation (\S\ref{sec:vlm}). Bold text indicates the better value per benchmark.}
\label{tab:vlm_results}
\small
\begin{tabular}{lcc}
\toprule
\textbf{Benchmark} & \textbf{Dynamic} & \textbf{Static} \\
\midrule
MMMU (acc)              & \textbf{0.2411} & 0.2311          \\
 ScienceQA               & \textbf{0.3954} & 0.3924          \\
 TextVQA                 & \textbf{0.3556} & 0.3397          \\
 DocVQA                  & \textbf{0.1847} & 0.1704          \\
 InfoVQA                 & \textbf{0.1099} & 0.1067          \\
 OCRBench                & \textbf{0.2810} & 0.2720          \\
 MME-Perception          & \textbf{704.03} & 653.82          \\
MME-Cognition           & 169.29          & \textbf{171.79} \\\midrule
MMStar (average)&0.2361&\textbf{0.2377} \\
\quad Coarse Perception       & 0.2782          & \textbf{0.2819} \\
\quad Fine-grained Perception & 0.2287          & \textbf{0.2317} \\
\quad Instance Reasoning      & \textbf{0.2457} & 0.2392          \\
\quad Logical Reasoning       & 0.2335          & \textbf{0.2341} \\
\quad Math                    & 0.2082          & \textbf{0.2369} \\
\quad Science                 & \textbf{0.2222} & 0.2024          \\

\bottomrule
\end{tabular}
\end{table}

\newpage
\section{How existing paper use batch size?}\label{appendix:existing_works}
    In this section we summarize the learning rate and batch size schedules used in training recent open source models. We summarize the schedules each model uses in Table \ref{tab:lr_bs_schedules}, also visualize them in Figure \ref{fig:related_work_lr_bs_schedules_p1} and \ref{fig:related_work_lr_bs_schedules_p2}. It could be seen that, most of the recent models are using a non-constant batch-size schedule, however they mostly employ a step function batch size schedule heuristically.

\begin{table*}[ht]
\centering
\small
\setlength{\tabcolsep}{3pt}
\renewcommand{\arraystretch}{1.15}
\resizebox{0.85\textwidth}{!}{\begin{tabularx}{\textwidth}{
  >{\raggedright\arraybackslash}p{2.25cm}
  >{\raggedright\arraybackslash}X
  >{\raggedright\arraybackslash}X
  >{\raggedright\arraybackslash}p{2.75cm}
}
\toprule
\textbf{Model}
& \textbf{Learning-rate schedule}
& \textbf{Batch-size schedule}
& \textbf{Where specified} \\
\midrule

Llama 3 405B
&
Peak LR \(8\times 10^{-5}\). Linear warmup for \(8{,}000\) steps; cosine decay to \(8\times 10^{-7}\) over \(1.2\)M steps; final long-context annealing linearly decays LR to \(0\) over the last \(40\)M tokens.
&
\(4\)M tokens/update initially; \(8\)M after \(252\)M tokens; \(16\)M after \(2.87\)T tokens.
&
\cite[Sec.~3.4, pp.~14--15]{llama3report2024}. \\

\midrule

Qwen3 / Qwen3.5 / Qwen3.6
&
Numeric LR schedule not disclosed in the public report/pages I found. Qwen3 states that predicted optimal LR/batch strategies are obtained by scaling laws and that LR decay is accelerated in Stage 2.
&
Numeric batch-size schedule not disclosed. Publicly stated stages: \(>30\)T tokens at 4K context, then \(\sim5\)T higher-quality tokens at 4K, then hundreds of billions of tokens at 32K.
&
\cite[Sec.~3.2, p.~4]{qwen3report2025}. \\

\midrule

DeepSeek-V4 Flash / Pro
&
Flash: \(2{,}000\)-step linear warmup; LR \(2.7\times10^{-4}\) for most training; near the end, cosine decay to \(2.7\times10^{-5}\). Pro: same strategy, peak \(2.0\times10^{-4}\), end \(2.0\times10^{-5}\). Exact decay-start token count not disclosed.
&
Flash: batch size increases from a small size to \(75.5\)M tokens/update, then stays there. Pro: maximum batch \(94.4\)M tokens/update. Exact ramp endpoint not disclosed.
&
\cite[Sec.~4.2.2, pp.~24--25]{deepseekv4report2026}. \\

\midrule

Kimi-K2
&
MuonClip with WSD. \(500\)-step warmup to \(2\times10^{-4}\), i.e. about \(500\times67\)M \(=33.5\)B tokens; constant \(2\times10^{-4}\) until \(10\)T tokens; cosine decay to \(2\times10^{-5}\) over the next \(5.5\)T tokens; terminal/long-context phase decays from \(2\times10^{-5}\) to \(7\times10^{-6}\), with decay shape not explicitly specified.
&
Constant \(67\)M tokens/update.
&
\cite[Sec.~2.5, p.~9]{kimik2report2025}. \\

\midrule

GPT-3 175B
&
Peak LR \(0.6\times10^{-4}=6\times10^{-5}\). Linear warmup for \(375\)M tokens; cosine decay to \(10\%\) of peak over \(260\)B tokens; constant at \(10\%\) of peak through \(300\)B tokens.
&
Final batch \(3.2\)M tokens/update. Batch linearly increases from \(32\)K tokens/update to final size over the first \(4\)--\(12\)B tokens, depending on model size; the 175B-specific ramp length is not separately stated.
&
\cite[Table~2.1, p.~8; App.~B, p.~42]{brown2020gpt3}. \\

\midrule

MT-NLG 530B
&
Peak LR \(5\times10^{-5}\). Linear warmup for \(1\)B tokens; cosine decay targeting \(10\%\) of peak over \(340\)B tokens. Training stops at \(270\)B tokens, so the LR does not reach the \(10\%\) target.
&
Batch starts at \(32\) sequences and increases by \(32\) sequences until \(1{,}920\) sequences over first \(12\)B tokens. With sequence length \(2{,}048\), this is \(65{,}536\to3.93\)M tokens/update.
&
\cite[Sec.~3.2, p.~10]{smith2022mtnlg}. \\

\midrule

PaLM 540B
&
Adafactor. LR \(10^{-2}\) for first \(10{,}000\) steps, then inverse-square-root decay proportional to \(1/\sqrt{k}\), where \(k\) is the step number.
&
Sequence length \(2{,}048\). Batch size \(512\) until step \(50\)K, \(1{,}024\) until step \(115\)K, and \(2{,}048\) until step \(255\)K; equivalently about \(1.05\)M, \(2.10\)M, and \(4.19\)M tokens/update.
&
\cite[Sec.~5, p.~10]{chowdhery2022palm}. \\

\midrule

Chinchilla 70B
&
Peak LR \(1\times10^{-4}\). Uses Gopher training setup except listed changes: warmup from \(10^{-7}\) to peak over \(1{,}500\) steps, followed by \(10\times\) cosine decay.
&
Batch \(1.5\)M \(\to\) \(3\)M tokens/update, doubled midway. Since Chinchilla trains on \(1.4\)T tokens, this is about \(700\)B tokens at each batch size.
&
\cite[Sec.~4.1/Table~4, p.~9]{hoffmann2022chinchilla};
\cite[Sec.~3.2, p.~6]{rae2021gopher}. \\

\midrule

Gopher 280B
&
Peak LR \(4\times10^{-5}\). Warmup from \(10^{-7}\) to peak over first \(1{,}500\) steps, followed by \(10\times\) cosine decay.
&
Batch \(3\)M \(\to\) \(6\)M tokens/update. The Chinchilla paper states the batch is doubled midway; for \(300\)B training tokens this gives roughly \(150\)B tokens at each batch size.
&
\cite[Sec.~3.1--3.2/Table~1, p.~6]{rae2021gopher};
\cite[Table~4, p.~9]{hoffmann2022chinchilla}. \\

\midrule

GLM-130B
&
Peak LR \(8\times10^{-5}\). Warmup from \(10^{-7}\) to peak over first \(0.5\%\) of samples, then \(10\times\) cosine decay to \(8\times10^{-6}\). For \(400\)B tokens, \(0.5\%\approx2\)B tokens.
&
Sequence length \(2{,}048\). Batch warmup from \(192\) to \(4{,}224\) sequences over first \(2.5\%\) of samples, then constant; equivalently \(0.393\)M \(\to\) \(8.65\)M tokens/update, with the ramp lasting about \(10\)B tokens.
&
\cite[p.~5; Table~11, p.~48]{zeng2023glm130b}. \\

\bottomrule
\end{tabularx}
}
\caption{Pre-training learning-rate and batch-size schedules. Batch size is expressed as tokens per optimizer update unless otherwise noted.}
\label{tab:lr_bs_schedules}
\end{table*}

\begin{figure*}[ht]
    \centering
    \setlength{\tabcolsep}{1pt}
    \renewcommand{\arraystretch}{0.9}

    \begin{tabular}{cccc}
        \textbf{LR} &
        \textbf{LR normalized} &
        \textbf{Batch size} &
        \textbf{Batch size normalized} \\
        \toprule

        \multicolumn{4}{l}{\textbf{Llama 3 405B}} \\[-0.35em]
        \includegraphics[width=0.235\textwidth]{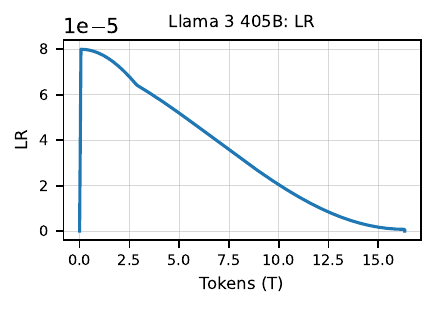} &
        \includegraphics[width=0.235\textwidth]{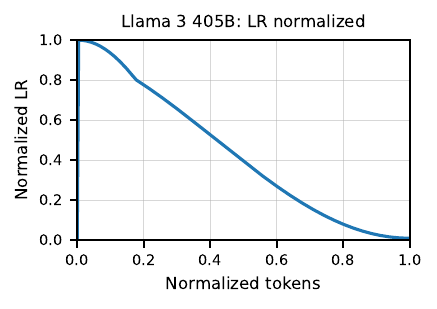} &
        \includegraphics[width=0.235\textwidth]{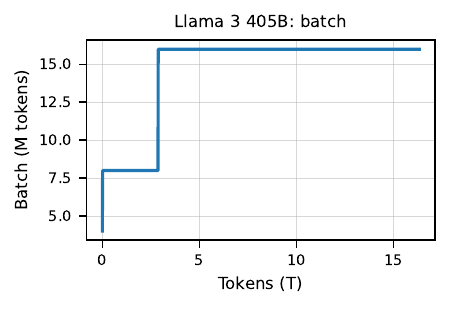} &
        \includegraphics[width=0.235\textwidth]{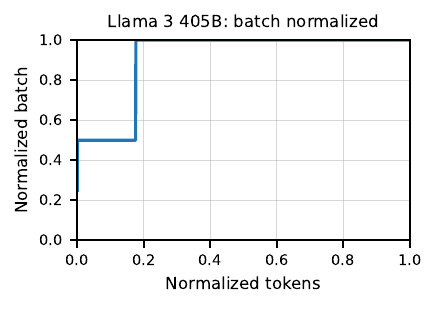} \\[0.35em]

        \multicolumn{4}{l}{\textbf{Kimi-K2}} \\[-0.35em]
        \includegraphics[width=0.235\textwidth]{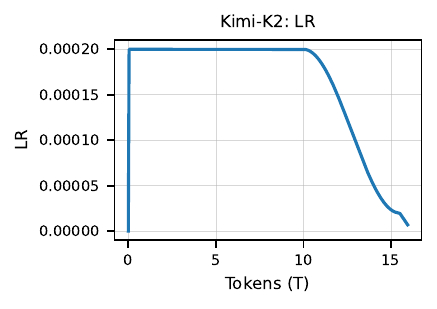} &
        \includegraphics[width=0.235\textwidth]{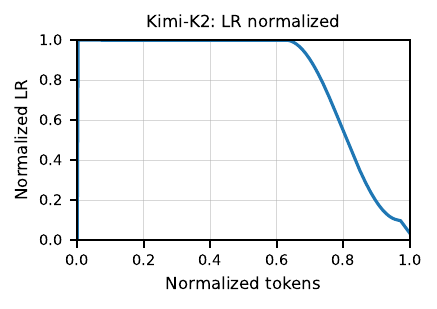} &
        \includegraphics[width=0.235\textwidth]{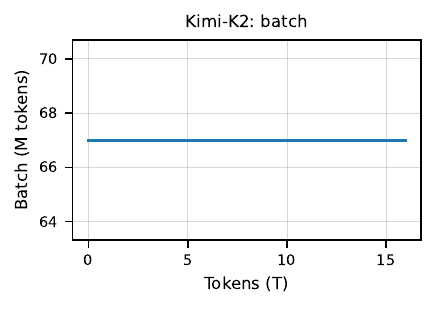} &
        \includegraphics[width=0.235\textwidth]{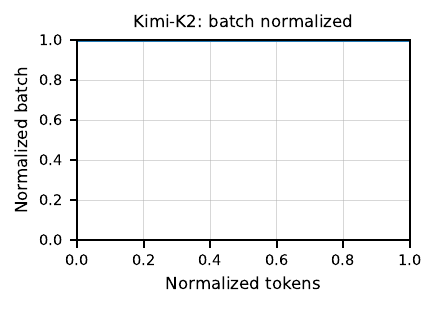} \\[0.35em]

        \multicolumn{4}{l}{\textbf{DeepSeek-V4 Pro (approximate)}} \\[-0.35em]
        \includegraphics[width=0.235\textwidth]{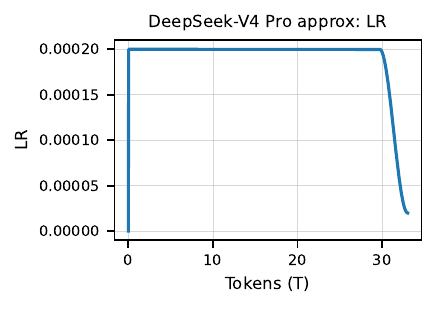} &
        \includegraphics[width=0.235\textwidth]{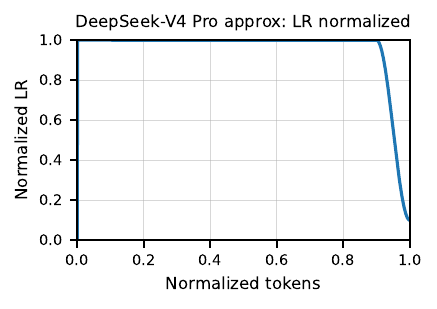} &
        \includegraphics[width=0.235\textwidth]{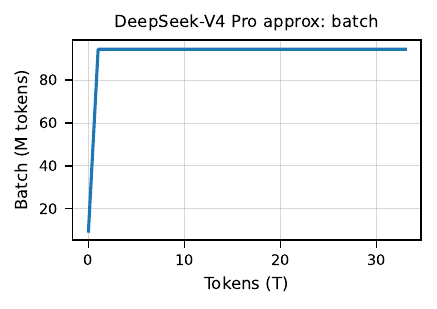} &
        \includegraphics[width=0.235\textwidth]{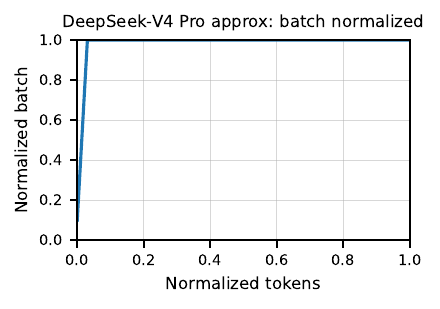} \\[0.35em]

        \multicolumn{4}{l}{\textbf{GPT-3 175B}} \\[-0.35em]
        \includegraphics[width=0.235\textwidth]{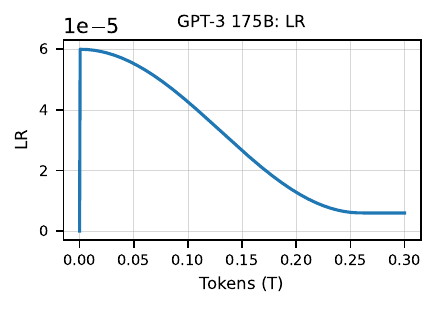} &
        \includegraphics[width=0.235\textwidth]{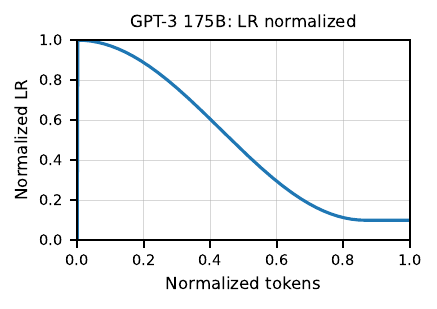} &
        \includegraphics[width=0.235\textwidth]{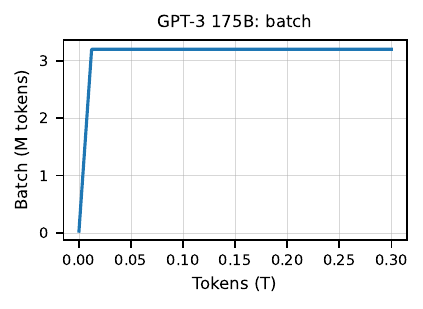} &
        \includegraphics[width=0.235\textwidth]{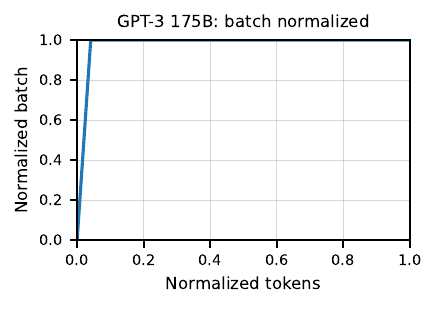} \\[0.35em]

        \multicolumn{4}{l}{\textbf{MT-NLG 530B}} \\[-0.35em]
        \includegraphics[width=0.235\textwidth]{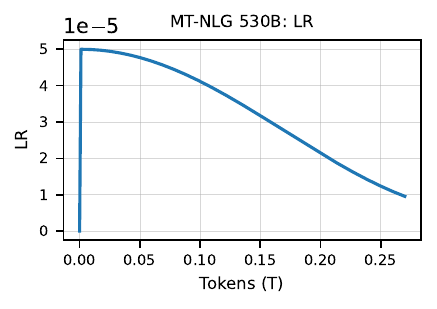} &
        \includegraphics[width=0.235\textwidth]{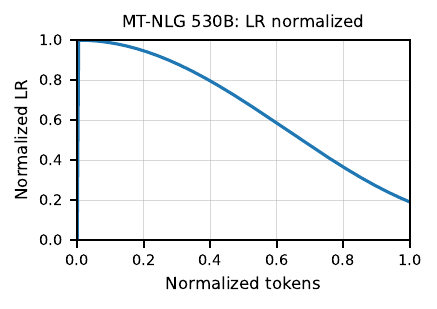} &
        \includegraphics[width=0.235\textwidth]{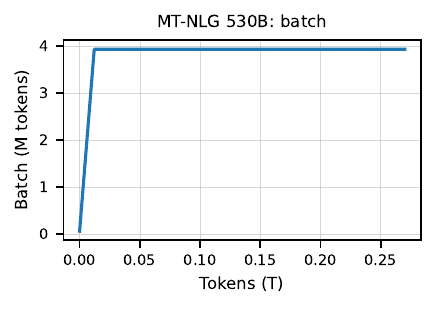} &
        \includegraphics[width=0.235\textwidth]{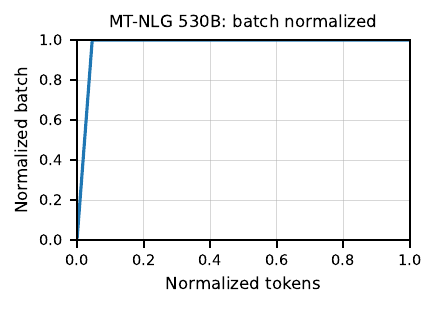} \\[0.35em]
\end{tabular}

    \caption{
        Learning-rate and batch-size schedules used in large-scale language model pre-training, part 1. The data are collected from the corresponding existing publications specified in Table \ref{tab:lr_bs_schedules}. 
        Each row corresponds to one model.
        The four columns show, from left to right: learning rate in original coordinates,
        learning rate in normalized coordinates, batch size in original coordinates,
        and batch size in normalized coordinates.
        In normalized plots, both axes are scaled to lie in $[0,1]$.
    }
    \label{fig:related_work_lr_bs_schedules_p1}
\end{figure*}

\begin{figure*}[ht]
    \centering
    \setlength{\tabcolsep}{1pt}
    \renewcommand{\arraystretch}{0.9}

    \begin{tabular}{cccc}
        \textbf{LR} &
        \textbf{LR normalized} &
        \textbf{Batch size} &
        \textbf{Batch size normalized} \\
        \toprule

        \multicolumn{4}{l}{\textbf{PaLM 540B}} \\[-0.35em]
        \includegraphics[width=0.235\textwidth]{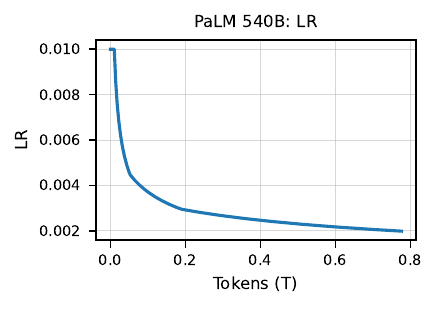} &
        \includegraphics[width=0.235\textwidth]{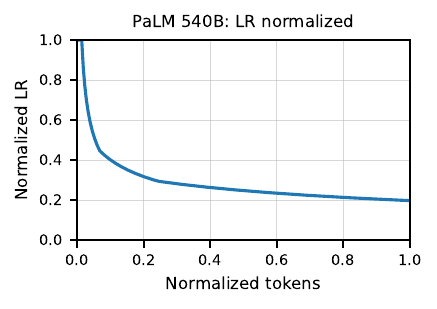} &
        \includegraphics[width=0.235\textwidth]{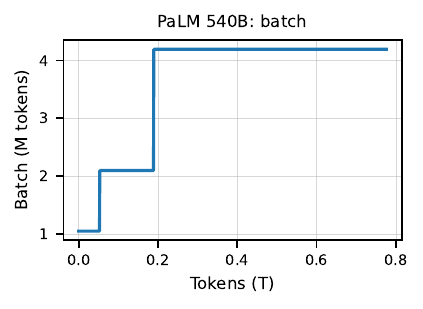} &
        \includegraphics[width=0.235\textwidth]{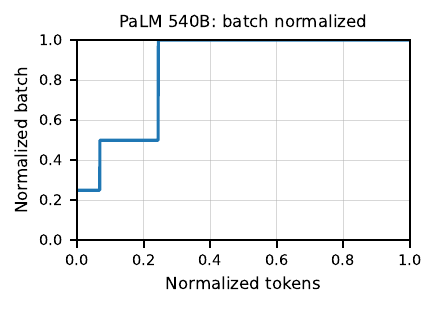} \\[0.35em]

        \multicolumn{4}{l}{\textbf{Chinchilla 70B}} \\[-0.35em]
        \includegraphics[width=0.235\textwidth]{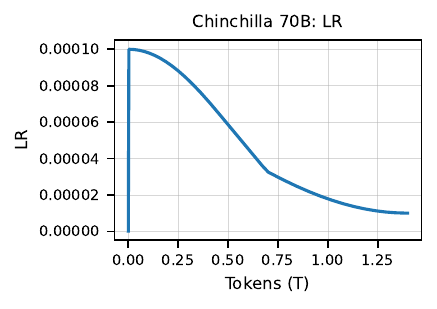} &
        \includegraphics[width=0.235\textwidth]{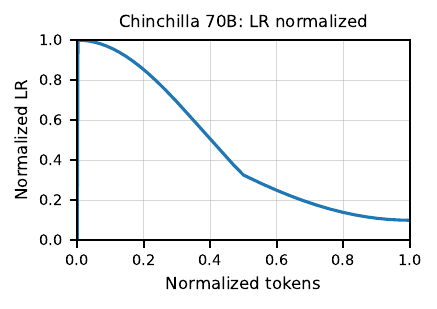} &
        \includegraphics[width=0.235\textwidth]{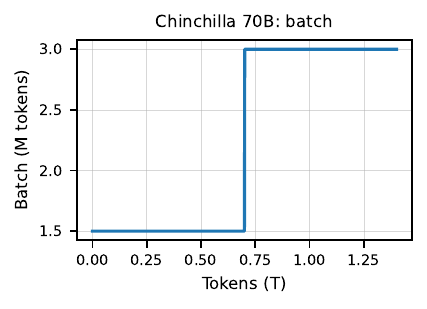} &
        \includegraphics[width=0.235\textwidth]{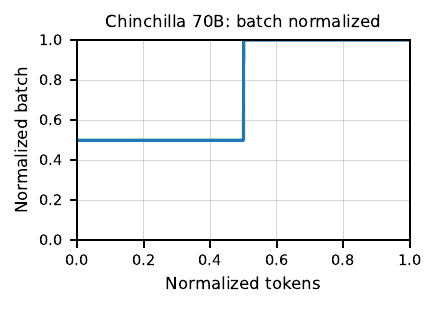} \\[0.35em]

        \multicolumn{4}{l}{\textbf{Gopher 280B}} \\[-0.35em]
        \includegraphics[width=0.235\textwidth]{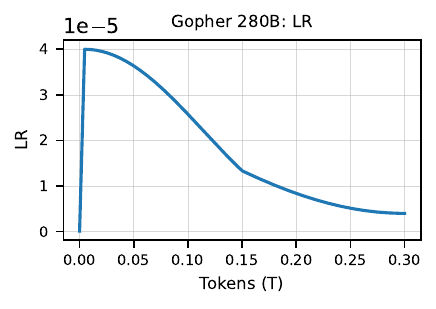} &
        \includegraphics[width=0.235\textwidth]{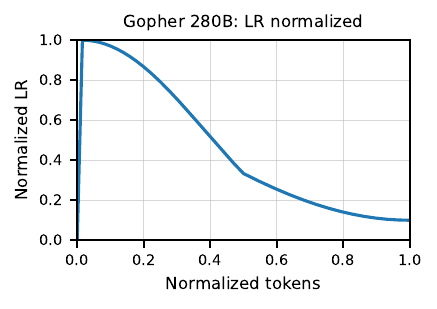} &
        \includegraphics[width=0.235\textwidth]{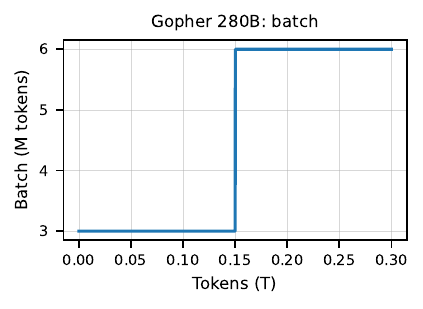} &
        \includegraphics[width=0.235\textwidth]{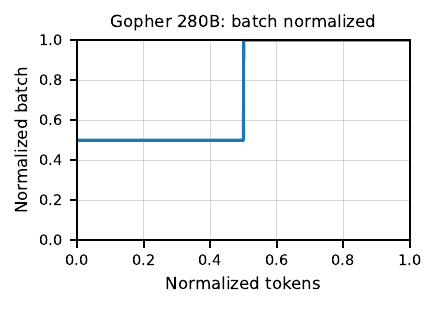} \\[0.35em]

        \multicolumn{4}{l}{\textbf{GLM-130B}} \\[-0.35em]
        \includegraphics[width=0.235\textwidth]{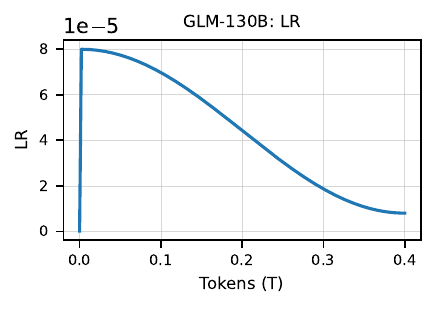} &
        \includegraphics[width=0.235\textwidth]{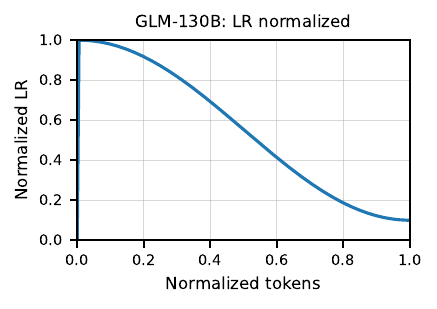} &
        \includegraphics[width=0.235\textwidth]{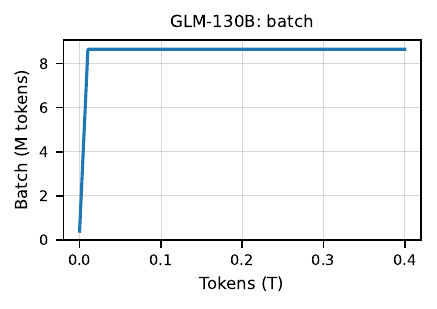} &
        \includegraphics[width=0.235\textwidth]{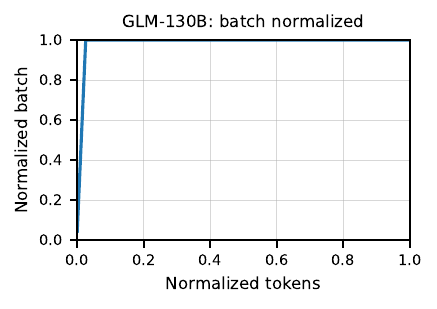} \\[0.35em]
    \end{tabular}

    \caption{
        Learning-rate and batch-size schedules used in large-scale language model pre-training, part 2. The data are collected from the corresponding existing publications specified in Table \ref{tab:lr_bs_schedules}.
        Each row corresponds to one model.
        The four columns show, from left to right: learning rate in original coordinates,
        learning rate in normalized coordinates, batch size in original coordinates,
        and batch size in normalized coordinates.
        In normalized plots, both axes are scaled to lie in $[0,1]$.
    }
    \label{fig:related_work_lr_bs_schedules_p2}
\end{figure*}

\end{document}